\documentclass[mnsc,non-blindrev]{informs3} 

\OneAndAHalfSpacedXI 

\usepackage{endnotes}
\usepackage[ruled]{algorithm}
\usepackage{algorithmicx}    
\usepackage{algpseudocode}
\usepackage{graphicx}
\usepackage{epstopdf}
\usepackage{subfigure,caption}
\let\footnote=\endnote
\let\enotesize=\normalsize
\def\notesname{Endnotes}%
\def\makeenmark{$^{\theenmark}$}
\def\enoteformat{\rightskip0pt\leftskip0pt\parindent=1.75em
  \leavevmode\llap{\theenmark.\enskip}}

\usepackage{natbib}
 \bibpunct[, ]{(}{)}{,}{a}{}{,}%
 \def\bibfont{\small}%
\usepackage{bm}
\usepackage{comment}
\newcommand{\R}{\mathcal{R}}
\newcommand{\ie}{\textit{i.e.}}
\newcommand{\eg}{\textit{e.g.}}
\newcommand{\Ex}{\mathbf{E}}
\renewcommand{\Re}{\mathbb{R}}

\newcommand{\swofu}{\texttt{SW-UCB}\text{ algorithm}}
\newcommand{\bob}{\texttt{BOB}\text{ algorithm}}
\definecolor{cadmiumgreen}{rgb}{0.0, 0.52, 0.24}

\TheoremsNumberedThrough     
\ECRepeatTheorems

\EquationsNumberedThrough    

\allowdisplaybreaks
\begin{document}


\RUNAUTHOR{Cheung, Simchi-Levi, and Zhu}

\RUNTITLE{Hedging the Drift: Learning to Optimize under Non-Stationarity}

\TITLE{Hedging the Drift: Learning to Optimize under Non-Stationarity}

\ARTICLEAUTHORS{
\AUTHOR{Wang Chi Cheung}
\AFF{Department of Industrial Systems Engineering and Management, National University of Singapore \EMAIL{isecwc@nus.edu.sg}}
\AUTHOR{David Simchi-Levi}
\AFF{Institute for Data, Systems, and Society, Massachusetts Institute of Technology, Cambridge, MA 02139, \EMAIL{dslevi@mit.edu}}
\AUTHOR{Ruihao Zhu}
\AFF{Institute for Data, Systems, and Society, Massachusetts Institute of Technology, Cambridge, MA 02139, \EMAIL{rzhu@mit.edu}} 
} 

\ABSTRACT{%

We introduce data-driven decision-making algorithms that achieve state-of-the-art \emph{dynamic regret} bounds for a collection of non-stationary stochastic bandit settings. These settings capture applications such as advertisement allocation, dynamic pricing, and traffic network routing in changing environments. We show how the difficulty posed by the (unknown \emph{a priori} and possibly adversarial) non-stationarity can be overcome by an unconventional marriage between stochastic and adversarial bandit learning algorithms. Beginning with the linear bandit setting, we design and analyze a sliding window-upper confidence bound algorithm that achieves the optimal dynamic regret bound when the underlying \emph{variation budget} is known. This budget quantifies the total amount of temporal variation of the latent environments. Boosted by the novel Bandit-over-Bandit framework that adapts to the latent changes, our algorithm can further enjoy nearly optimal dynamic regret bounds in a (surprisingly) parameter-free manner. We extend our results to other related bandit problems, namely the multi-armed bandit, generalized linear bandit, and combinatorial semi-bandit settings, which model a variety of operations research applications. In addition to the classical exploration-exploitation trade-off, our algorithms leverage the power of the ``forgetting principle" in the learning processes, which is vital in changing environments. Extensive numerical experiments with synthetic datasets and a dataset of an online auto-loan company demonstrate that our proposed algorithms achieve superior performance compared to existing algorithms.
}%



\maketitle

%

\section{Introduction}
\label{sec:intro}
Consider the following general decision-making framework: a decision-maker (DM) interacts with a \emph{multi-armed bandit} (MAB) system by picking actions one at a time sequentially. Upon selecting an action, she instantly receives a reward drawn randomly from a probability distribution tied to this action. The goal of the DM is to maximize her cumulative rewards. However, she faces the following challenges:
\begin{itemize}
	\item \emph{Uncertainty:} the reward distribution of each action is initially not known to the DM. She has to estimate the underlying reward distributions via interacting with the environment.
	\item \emph{Non-Stationarity:} the reward distributions can evolve over time. 
	\item \emph{Partial/Bandit Feedback:} the DM can only observe the random reward of the selected action each time, while the rewards of the unchosen actions are not observed.
\end{itemize}
Many applications naturally fall into this non-stationary MAB framework. For instance, with a linear reward model, which will also be the main focus of this paper, we can cast the problems of dynamic pricing \citep{KZ14,KZ16}, advertisement allocation \citep{LCLS10,CLRS11} in dynamic and evolving environments into the above decision-making framework. This framework also finds applications in traffic network routing \citep{GKJ12,KWAS15}.
\begin{example}[\textbf{Dynamic Pricing}] In the classical setup of dynamic pricing \citep{KZ14,KZ16}, a seller decides dynamically the prices of a product for a sequence of incoming customers with the hope to maximize the cumulative revenue. Beginning with an unknown demand function that represents the customers' sensitivity towards price changes, the DM only observes the purchase decision (\eg, buy/not buy or purchase quantities) of each customer under the corresponding posted price. Moreover, the demand function can evolve over time due to unexpected events. For example, after the announcement of the COVID-19 pandemic on 11 March 2020 \citep{WHO20}, the demand for daily essentials and shelf-stable foods increased suddenly \citep{BecdachBHHM20}.
\end{example}
\begin{example}[\textbf{Advertisement Allocation}] An online platform allocates advertisements (ads) to a sequence of users. For each arriving user, the platform has to deliver an ad to her, and only observes the response to her displayed ad. The platform has full access to the features of the ads and the users. Following \citep{LCLS10,CLRS11}, we could assume that a user's click behavior towards an ad, or simply the click through rate (CTR) of this ad by a particular user, follows a probability distribution governed by a common, but initially unknown, response function of the features. The platform's goal is to maximize the total number of clicks. However, the unknown response function can change over time. For instance, if it is around the time when Apple releases a new iPhone model, one can expect that the popularity of an Apple's ad grows.
\end{example}
\begin{example}[\textbf{Traffic Network Routing}]A navigation service provider has to iteratively offer route planning services to drivers from an origin to a destination through a traffic network with initially unknown random delay on each road. For each driver, the provider could only see the delays of the roads traversed by this driver, but not the other roads'. Moreover, the delay distributions could change over time as the roads are also shared by other traffics (\ie, those not using this navigation service). The provider wants to minimize the cumulative delays throughout the course of vehicle routing.
\end{example}

Evidently, the DM faces a trilemma among exploration, exploitation as well as adaptation to changes. On one hand, the DM wishes to exploit, and to select the action with the best historical performances to earn as much reward as possible. On the other hand, she wants to explore other actions to get a more accurate estimation of the reward distributions. The changing environment makes the exploration-exploitation trade-off even more delicate. Indeed, past observations could become obsolete due to the changes in the environment, and the DM needs to explore for changes and refrain from exploiting possibly outdated observations.

We focus on resolving this trilemma in various MAB problems. Traditionally, most MAB problems are studied in the stochastic \citep{ABF02} and adversarial \citep{ABFS02} environments. In the former, the uncertain model is static, and each feedback is corrupted by a mean zero random noise. The DM aims at estimating the latent static environment using historical data and converging to the optimum, which is achieved by a static strategy that selects a single action throughout. In the latter, the model is not only uncertain, but also dynamically changed by an adversary. While the DM strives to hedge against the changes, it is generally impossible to achieve the optimum. Hence, existing research also focuses on competing favorably in comparison to a static strategy.

Unfortunately, strategies for the stochastic environments can quickly deteriorate under non-stationarity as historical data might ``expire", while the permission of a confronting adversary in the adversarial settings could be too pessimistic. Starting from \citep{BGZ14,BGZ15}, a stream of research works (see Section \ref{sec:related_works}) focuses on MAB problems in a \emph{drifting} environment, which is a hybrid of a stochastic and an adversarial environment. Although the environment can be dynamically and adversarially changed, the total changes (quantified by a suitable metric) in a $T$-round problem is upper bounded by $B_T~(= \Theta(T^{\rho})\text{ for some }\rho\in (0, 1))$, the \emph{variation budget} \citep{BGZ14,BGZ15}, and the feedback is corrupted by an additive mean zero random noise. The aim is to minimize the \emph{dynamic regret} \citep{BGZ14}, which is the optimality gap compared to the sequence of (possibly dynamically changing) optimal decisions, by simultaneously estimating the current environment and hedging against future changes every round. The framework of \citep{BGZ14,BGZ15} enable us to compete against the so-called \emph{dynamic comparator}. Most of the existing works for non-stationary bandits have focused on the the relatively ideal case in which $B_T$ is known. In practice, however, $B_T$ is often not available ahead as it is a quantity that requires knowledge of future information. Though some efforts have been made towards this direction \citep{KA16,LWAL18}, the design of algorithms with low dynamic regret when $B_T$ is unknown remains largely a challenging problem.

In this paper, we design and analyze a novel algorithmic framework for bandit problems in drifting environments. We begin by demonstrating our results via the lens of the linear bandit model, and then we demonstrate the generality of our framework on related MAB models. Our main contributions can be summarized as follows.
\begin{itemize}
	\item When the variation budget $B_T$ is known, we provide a lower bound on the dynamic regret incurred by any non-anticipatory policy. In complement, we develop a tuned Sliding Window Upper-Confidence-Bound (\texttt{SW-UCB}) algorithm with a matching dynamic regret upper bound, up to multiplicative logarithmic factors.
	\item When $B_T$ is unknown, we propose a novel Bandit-over-Bandit (\texttt{BOB}) framework that tunes the window size of the \swofu~adaptively. When the amount of non-stationarity is above a certain threshold (that depends on $B_T, T$),  the \bob~achieves the optimal dynamic regret bound. Otherwise, it still obtains a dynamic regret bound sublinear in $T$. While the optimal dynamic regret bound is not achieved in the latter case, the resulting dynamic regret bound is better than the state-of-the-art in prior literature.
	\item Our algorithm design and analysis shed light on the fine balance among exploration, exploitation and adaptation to changes in dynamic learning environments. We rigorously incorporate the ``forgetting principle'' \citep{GM11} into the Optimism-in-Face-of-Uncertainty principle \citep{ABF02,AYPS11}, by demonstrating that the DM can enjoy an optimal dynamic regret bound if she keeps disposing of sufficiently old observations. We also provide a rate of disposal that leads to the optimality.
	\item Finally, we point out that a preliminary version of this paper appears in the $22^{\text{nd}}$ International Conference on Artificial Intelligence and Statistics (AISTATS 2019) \citep{CSLZ19}, and the current paper provides significant additional contributions in three directions. 
	First, when $B_T$ is unknown, the current version provides a substantially refined design and analysis of the \bob~for the linear bandit model, resulting in an improved dynamic regret bound (\ie, Theorem \ref{theorem:bob} of Section \ref{sec:bob}) 
	 compared to Theorem 4 of \citep{CSLZ19}. Second, unlike \citep{CSLZ19}, which only focuses on the linear bandit model, in the current paper we extend our approach, in Section \ref{sec:applications}, to several related bandit settings, including multi-armed bandits, generalized linear bandits, and combinatorial semi-bandits. These extensions capture many important operations research applications, such as the three examples highlighted in the introduction. Third, we conduct numerical experiments using a new synthetic dataset to evaluate our algorithms in piecewise-linear environments for both 2-armed bandit and linear bandit settings. We also study the performances of our algorithms in a case of dynamic pricing under the SARS epidemic with a real world auto-loan dataset.	Both of these experiments extend significantly beyond the simple drifting 2-armed bandit experiments in the AISTATS version.
\end{itemize} 

The rest of the paper is organized as follows. In Section \ref{sec:related_works}, we review existing MAB works in stationary and non-stationary environments. In Section \ref{sec:formulation}, we formulate the non-stationary linear bandit model. In Section \ref{sec:lower_bound}, we establish a minimax lower bound on the dynamic regret. In Section \ref{sec:swlse}, we describe the sliding window estimator for parameter estimation under non-stationarity. In Section \ref{sec:swofu}, we develop the sliding window-upper confidence bound algorithm with optimal dynamic regret (when the amount of non-stationarity is known ahead). In Section \ref{sec:bob}, we introduce the novel Bandit-over-Bandit framework with nearly optimal dynamic regret. In Section \ref{sec:applications}, we demonstrate the generality of the established results by applying them to related bandit settings, namely the multi-armed bandit, generalized linear bandit, and combinatorial semi-bandit settings. In Section \ref{sec:numerical}, we conduct extensive numerical experiments with both synthetic and CPRM-12-001: on-line auto lending datasets to show the superior empirical performances of our algorithms. In Section \ref{sec:conclusion}, we conclude our paper.
\section{Related Works} \label{sec:related_works}
\subsection{Stationary and Adversarial Bandits}
MAB problems with stochastic and adversarial environments are extensively studied, as surveyed in \citep{BC12, LS18}. 
To model inter-dependence 
among different arms, models for linear bandits in stochastic environments have been studied. In \citep{A02,DHK08,RT10,CLRS11,AYPS11}, UCB type algorithms for stochastic linear bandits were studied, and the authors of \citep{AYPS11} provided the tightest regret analysis for algorithms of this kind. The authors of \citep{RVR14,AG13,AL17} proposed Thompson sampling algorithms for this setting to bypass the high computational complexity of the UCB type algorithms.
\subsection{Bandits in Drifting Environments}
Departing from purely stochastic or adversarial settings, Besbes et al. \citep{BGZ14,BGZ15} laid down the foundation of bandit in drifting environments, and considered the $K$-armed bandit setting. They achieved the tight dynamic regret bound $\tilde{O}((K B_T)^{1/3} T^{2/3})$ by restarting the EXP3 algorithm \citep{ABFS02} periodically when $B_T$ is known. \cite{WHL16} provided refined regret bounds based on empirical variance estimation, assuming the knowledge of $B_T$. \cite{WeiS18} analyzed the sliding window upper confidence bound algorithm for the $K$-armed MAB with known $B_T$ setting. Subsequently, \cite{KA16} considered the setting without knowing $B_T$ and $K=2$, and achieved a dynamic regret bound of $\widetilde{O}(B_T^{9/50} T^{41/50} + T^{77/100})$ with a change point detection type technique. In a recent work, \cite{LWAL18} generalized this change point detection type technique to the $K$-armed contextual bandits in drifting environments, and in particular demonstrated an improved bound $\widetilde{O}(KB_T^{1/5}T^{4/5})$ for the $K$-armed bandit problem in drifting environments when $B_T$ is not known. \cite{KZ16} considered a dynamic pricing problem in a drifting environment with 2-dimensional linear demands. Assuming a known variation budget $B_T,$ they proved an $\Omega (B_T^{1/3}T^{2/3} )$ dynamic regret lower bound and proposed a matching algorithm by properly discounting historical observations (this includes sliding-window estimation as a special case). When $B_T$ is not known, their algorithm achieves $\tilde{O}(B_T T^{2/3})$ dynamic regret bound. Finally, various online problems with full feedback in drifting environments were studied in \citep{CYLMLJZ12,BGZ15,JRSS15}.
\begin{table}[!ht]
	\begin{center}
	\begin{tabular}{|c|c|c|} 
		\hline
		&Known $B_T$&Unknown $B_T$\\
		\hline
		\citep{BGZ15}&$\widetilde{O}\left(B_T^{{1}/{3}}T^{{2}/{3}}\right)$&$\widetilde{O}\left(B_TT^{{2}/{3}}\right)$\\
		\hline
		\citep{KA16}&$\widetilde{O}\left(B_T^{9/50} T^{41/50} + T^{77/100}\right)$&$\widetilde{O}\left(B_T^{9/50} T^{41/50} + T^{77/100}\right)$\\
		\hline
	 	\citep{LWAL18}&$\widetilde{O}\left(B_T^{{1}/{3}}T^{{2}/{3}}\right)$&$\widetilde{O}\left(B_T^{1/5}T^{4/5}\right)$\\
		\hline 
		The current work&$\widetilde{O}\left(B_T^{{1}/{3}}T^{{2}/{3}}\right)$&$\widetilde{O}\left(B_T^{{1}/{3}}T^{{2}/{3}}+T^{{3}/{4}}\right)$\\
		\hline
	\end{tabular}
\caption{Comparisons between our results and prior works. Here, the dynamic regret bounds only show dependence on $B_T$ and $T.$ $\widetilde{O}(\cdot)$ denotes the function growth, and omits the logarithmic factors.}
\end{center}
\end{table}
\subsection{Bandits in Piecewise Stationary/Switching Environments}
Apart from drifting environments, numerous research works consider the \emph{piecewise stationary/switching environment}, where the time horizon is partitioned into at most $S$ intervals. The expected reward for each arm remains constant in each interval, but it can vary across different intervals. The partition is not known to the DM. Algorithms were designed for various bandit settings, with knowledge of $S$ \citep{ABFS02, GM11,LiuLS18,LWAL18,CaoWKX19}, or without knowing $S$ \citep{KA16, LWAL18}. Notably, the Sliding Window-UCB and the ``forgetting principle" was first proposed by Garivier and Moulines \citep{GM11}. The algorithm was only analyzed under $K$-armed switching environments. But we also have to emphasize that the $S$ is a looser measure of non-stationarity in the sense that every tiny change in the environment could be counted towards the total number of switches. In other words, even if there are a total of $T$ switches, the total variation budget $B_T$ could still be far less than $T.$ Hence, the drifting environment serves as a better proxy for non-stationarity.
\subsection{Further Contrasts to Existing Works}
The main idea underpinning our Bandit-over-Bandit framework is to use a learning algorithm to tune the underlying base learning algorithm's parameters. While this shares similar spirit to several existing works, such as the heuristic envelop policy \citep{BGZ18} and algorithms for bandit corralling (see \cite{ALNS17,LWAL18} and references therein), our design is different in the sense that rather than simultaneously maintaining multiple copies of the base learning algorithm~(as in \cite{ALNS17,LWAL18,BGZ18}), we treat the problem of selecting window size for the \swofu~as another independent adversarial bandit learning instance. To achieve this, we divide the time horizon into epochs, and force the \swofu~to restart at the beginning of each epoch. This critical difference allows us to establish an improved and nearly optimal parameter-free dynamic regret bound of the \bob~when compared to prior research.
\subsection{Follow-Up Works and Other Related Works}
The results presented in \cite{LWAL18} were further improved to the optimal $\widetilde{O}(K^{1/3}B_T^{1/3}T^{2/3})$ dynamic regret bound in \cite{CLLW19}, but it is unclear how to generalize the techniques in \cite{CLLW19} beyond the $K$-armed bandit setting. In \cite{BessonK19,AuerGO19}, the authors presented optimal learning algorithms for the switching setting without knowing the number of switches. In \cite{ZhouCGX20}, the authors considered an environment where the non-stationarity is governed by a finite-state Markov chain. In \cite{ChenWW20}, a periodically changing environment was also studied. The design of parameter-free online learning algorithms were also considered in other online learning settings, such as bandit convex optimization \citep{ZhaoWZZ19} and reinforcement learning \citep{CheungSLZ19,CheungSLZ20ICML}. Another related but different line of research is bandit learning with corrupted data, interested readers can refer to \cite{LykourisML18,GolrezaeiMSS20} for more details.

\section{Problem Formulation for Drifting Linear Bandits}
\label{sec:formulation}
We start by introducing the notations to be used and the model formulation. From the current section to the end of Section \ref{sec:bob}, we focus on the drifting linear bandit problem, which serves to illustrate our algorithmic framework. After that, we provide generalizations to other bandit problems in drifting environments in Section \ref{sec:applications}.
\subsection{Notation}
Throughout the paper, all vectors are column vectors, unless specified otherwise. We define $[n]$ to be the set $\{1,2,\ldots,n\}$ for any positive integer $n.$ 
We denote $\langle x, y\rangle = x^\top y$ as the inner product between $x, y\in \mathbb{R}^d$. For $p\in [1, \infty]$, we use $\|\bm x\|_p$ to denote the $p$-norm of a vector $\bm x\in\Re^d.$ For a positive definite matrix $A\in \Re^{d\times d}$, we use $\|\bm x\|_A$ to denote $\sqrt{\bm{x}^{\top}A\bm x}$ of a vector $\bm x\in\Re^d.$ We denote $x\vee y$ and $x\wedge y$ as the maximum and minimum between $x,y\in\Re,$ respectively. We adopt the asymptotic notations $O(\cdot),\Omega(\cdot),$ and $\Theta(\cdot)$ \citep{CLRS09}. When logarithmic factors are omitted, we use $\widetilde{O}(\cdot),\widetilde{\Omega}(\cdot),$ $\widetilde{\Theta}(\cdot),$ respectively. With some abuse, these notations are used when we try to avoid the clutter of writing out constants explicitly.
\subsection{Learning Protocol}
In each round $t\in[T]$, a decision set $D_t\subseteq\Re^d$ is presented to the DM. Then, the DM chooses an action $X_t\in D_t.$ Afterwards, the reward $Y_t=\langle X_t,\theta_t\rangle+\eta_t$
is revealed to the DM as a whole. We allow $D_t$ to be chosen by an \emph{oblivious adversary}, who chooses the decision sets $\{D_t\}^T_{t=1}$
before the protocol starts \citep{CBL06}. The parameter vector $\theta_t\in\Re^d$ is an unknown $d$-dimensional vector, and $\eta_t$ is a random noise drawn i.i.d. from an unknown sub-Gaussian distribution \citep{RH18} with variance proxy $R$. By definition, this means $\Ex\left[\eta_t\right]=0$, and $\forall\lambda\in\Re$ we have
$\Ex\left[\exp\left(\lambda\eta_t\right)\right]\leq\exp (\lambda^2R^2 / 2 ).$ Following the convention of the existing linear bandit literature \citep{AYPS11,AG13}, we assume there are positive constants $L$ and $S,$ such that $\|X\|_2\leq L$ for all $X\in D_t$ and all $t\in[T]$, and $\|\theta_t\|_2 \leq S$ holds for all $t\in[T]$. In addition, the instance is normalized so that $| \langle X,\theta_t\rangle | \leq 1$ for all $X\in D_t$ and $t\in[T].$ The constants $L, S$ are known to the DM.

We consider the \emph{drifting environment} \citep{BGZ14}, where $\theta_t$ can change over different $t,$ with the constraint that the sum of the Euclidean distances between consecutive $\theta_t$'s is bounded from above by the variation budget $B_T= \Theta(T^{\rho})\text{ for some }\rho\in (0, 1)$, \ie, \begin{align}\label{eq:variation_budget}\sum_{t=1}^{T-1}\left\|\theta_{t+1}-\theta_t\right\|_2 \leq B_T.\end{align}
We allow $\theta_t$'s to be chosen by an oblivious adversary. It is worth pointing out that the concepts of a drift environment and variation budget were originally introduced in \citep{BGZ15} and \citep{BGZ14,BGZ18} for the full information setting and the partial/bandit feedback setting, respectively.

We define $\mathcal{H}_t=\{D_s,X_s,Y_s\}_{s=1}^{t-1}\cup\{D_t\}$ as the available history information at round $t\in[T]$. The DM's goal is to design a non-anticipatory policy $\pi,$ which only uses the information $\mathcal{H}_t$ in each round $t,$ to maximize the cumulative reward. Equivalently, the goal is to minimize the \emph{dynamic regret}, which is the worst case cumulative regret against the optimal policy $\pi^*$, that has full knowledge of $\theta_t$'s. Denoting $x_t^*=\argmax_{x\in D_t}\langle x,\theta_t\rangle,$ the dynamic regret of a non-anticipatory policy $\pi$ is mathematically expressed as
$\R_T(\pi)=\Ex\left[\text{Regret}_T(\pi)\right]=\Ex\left[\sum_{t=1}^T\langle x_t^*-X_t,\theta_t\rangle\right],$
where the expectation is taken with respect to the randomness of $X_t$ and $\mathcal{H}_t$ as well as the (possible) randomness of the policy.
\begin{remark}[\textbf{Comparison to Piecewise Stationary Environment}]
	A related non-stationary environment is the piecewise stationary environment \citep{GM11}, which allows $\theta_t$'s to change at most $S$ times throughout the time horizon. However, as discussed in Section \ref{sec:related_works}, this can be a looser measure of non-stationarity as a very tiny change in the environment is still counted towards the total number of switches. That is to say, even if there are a total of $T$ switches, the total variation could grow in a sublinear rate in $T.$
\end{remark}
\section{Lower Bound}
\label{sec:lower_bound}
We first provide a lower bound on the the dynamic regret for the linear model.
\begin{theorem}
	\label{theorem:lower_bound}
	In the drifting linear bandit setting, for any $T\geq d$ and $B_T\in[dT^{-1/2},8d^{-2}T],$ there exists decision sets $\{D_t\}_{t=1}^T$ and reward vectors $\{\theta_t\}_{t=1}^T,$ such that for all $t\in[T]$ and all $x\in D_t,$ we have $\|x\|\leq1$, $\|\theta_t\|\leq 1,$ and $\|\langle x,\theta_t\rangle\|\leq1,$ and the dynamic regret for any non-anticipatory policy $\pi$ satisfies $\R_T(\pi)=\Omega\left(d^{2/3}B_T^{1/3}T^{2/3}\right).$
\end{theorem}
\begin{proof}{Poof Sketch.}
	The complete proof is presented in Section \ref{sec:theorem:lower_bound} of the appendix. The construction of the lower bound instance is similar to the approach by \citep{BGZ14}. The nature divides the whole time horizon into $\lceil T/H\rceil$ blocks of equal length $H=\lceil{(dT)^{{2}/{3}}B_T^{-{2}/{3}}}\rceil~(\leq T)$ rounds, and the last block can possibly have less than $H$ rounds. In each block, the nature initiates a new stationary linear bandit instance with parameter vectors from the set $\{\pm\sqrt{d/4H}\}^d$. We set up the instance so that the parameter vector of a block cannot be learned using the observations from the previous blocks. Consequently, every online policy must incur a regret of $\Omega(d\sqrt{H})$ in each block, by applying the regret lower bound for stationary linear bandits (for example, see \cite{LS18}) on each block. Since there are at least $\lfloor T/H\rfloor$ blocks, the total dynamic regret is $\Omega(dT/\sqrt{H})=\Omega(d^{2/3}B_T^{1/3}T^{2/3}).\Halmos$ 
\end{proof}
\section{Sliding Window Regularized Least Squares Estimator}
\label{sec:swlse}
As a preliminary, we introduce the sliding window regularized least squares estimator (SW-RLSE), which is the key tool in estimating the unknown parameters $\{\theta_t\}^T_{t=1}$ online. The SW-RLSE generalizes the sliding window sample estimator proposed by \citep{GM11} for the $K$-armed bandits in piecewise stationary environments. In addition, our SW-RLSE can be constructed for any sequence of arm pulls, which is different from \citep{KZ16}, who require each arm (in their setting a posted price) to be pulled equally often. Despite the underlying non-stationarity in our model, we show that the estimation error of our SW-RLSE scales gracefully with the variation of $\theta_t$'s across time.

To motivate SW-RLSE, consider a round $t$, where the DM aims to estimate $\theta_t$ based on the historical observations $\{(X_s, Y_s)\}^{t-1}_{s=1}$. The design of SW-RLSE is based on the \emph{forgetting principle} \citep{GM11}, which argues the following: the DM could estimate $\theta_t$ using only $\{(X_s, Y_s)\}^{t-1}_{s = 1\vee (t-w)}$, the observation history during the time window $(1\vee (t-w))$ to $(t-1)$, instead of all prior observations. Here, $w$ is the window size. The rationale is that, under non-stationarity, the observations far in the past are obsolete, and they are not as informative for regressing $\theta_t$.  The principle crucially hinges on $w$, which is a positive integer called the window size. Intuitively, when the variation across $\theta_1, \ldots, \theta_T$ increases, the window size $w$ should be smaller, since the past observations become obsolete at a faster rate. We treat $w$ as a fixed parameter in this section, and then shine lights on choosing $w$ in subsequent sections.

The SW-RLSE $\hat{\theta}_t$ is the optimal solution to the following ridge regression problem with regularization parameter $\lambda >0$: 
$$\min_{\theta:\theta\in\Re^d} \lambda \left\|\theta\right\|^2_2 + \sum^{t-1}_{s = 1\vee (t-w)} (X_s^\top \theta - Y_s)^2.$$ Define matrix $V_{t-1} := \lambda I + \sum^{t-1}_{s = 1\vee (t-w)} X_s X_s^\top$. The SW-RLSE $\hat{\theta}_t$ can be explicitly expressed as
\begin{align}
	\label{eq:sw20}
	&\hat{\theta}_t = V_{t-1}^{-1}\left( \sum^{t-1}_{s = 1\vee(t-w)}X_s Y_s \right) =V_{t-1}^{-1}\sum_{s=1\vee(t-w)}^{t-1}X_sX_s^{\top}\theta_s+V_{t-1}^{-1}\sum_{s=1\vee (t-w)}^{t-1}\eta_sX_s.
\end{align}
Next, we demonstrate the accuracy of the SW-RLSE. Denoting
\begin{equation}\label{eq:sw_beta}
\beta := R\sqrt{d\ln\left(\frac{1+wL^2/\lambda}{\delta}\right)}+\sqrt{\lambda}S,
\end{equation} 
we provide an error bound on estimating the latent reward, \ie, the confidence radius, of any action $x\in D_t$ in a round $t$, under the following regularity assumption made in \cite{FauryRAC21} over the decision sets $D_t$'s.
\begin{assumption}\label{ass:reg}
There exists an orthonormal basis $\Psi=(\psi_1,\ldots,\psi_d)$ such that for any $t\in[T]$ and any $X\in D_t,$ there exists a number $z\in\mathbb{R}$ and an $i\in[d]$ such that $X=z\cdot \psi_i.$ 
\end{assumption}
\begin{remark}
One can easily verify that this assumption holds in the multi-armed bandits case. Of course, this assumption allows for more general models than the multi-armed bandits setting as it still allows each of the time-varying $D_t$'s to have arbitrarily large number of actions. 
\end{remark}

In what follows, we analyze the linear bandit setting under Assumption \ref{ass:reg}. We also discuss how to remove this assumption in Remark \ref{remark:ass} of the forthcoming Section \ref{sec:bob}.
\begin{theorem}
	\label{theorem:sw_deviation}
 For any $t\in[T]$ and any $\delta\in[0,1]$, we have with probability at least $1-\delta,$ 
	$\left|x^\top ( \hat{\theta}_t - \theta_t)\right|\leq L \sum^{t-1}_{s = 1\vee (t-w)}\left\|\theta_s-\theta_{s+1}\right\|_2+\beta\left\|x\right\|_{V^{-1}_{t-1}}$
	holds for all $x\in D_t.$
\end{theorem} 
\begin{proof}{Proof Sketch.}
	The complete proof is in Section \ref{sec:theorem:sw_deviation} of the appendix. Note that $
		\hat{\theta}_t-\theta_t=V_{t-1}^{-1}\sum_{s=1\vee(t-w)}^{t-1}X_sX_s^{\top}\left(\theta_s-\theta_t\right)+V_{t-1}^{-1} \left(\sum_{s=1\vee(t-w)}^{t-1}\eta_sX_s-\lambda\theta_t\right),$
	we first upper bound the first term as
	$\left\|V_{t-1}^{-1}\sum_{s=1\vee(t-w)}^{t-1}X_sX_s^{\top}\left(\theta_s-\theta_t\right)\right\|_2\leq\sum^{t-1}_{s = 1\vee (t-w)}\left\|\theta_s-\theta_{s+1}\right\|_2,$
	and then adopts Theorem 2 from \citep{AYPS11} for the second term, \ie, with probability at least $1-\delta,$
	$\left\|\sum_{s=1\vee(t-w)}^{t-1}\eta_sX_s-\lambda\theta_t\right\|_{V_{t-1}^{-1}}\leq \beta.$ Therefore, fixed any $\delta\in[0,1],$ we have that for any $t\in[T]$ and any $x\in D_t,$
	\begin{align}
\nonumber\left|x^\top ( \hat{\theta}_t - \theta_t)\right|=&\left| x^\top \left( V_{t-1}^{-1}\sum_{s=1\vee(t-w)}^{t-1}X_sX_s^{\top}\left(\theta_s-\theta_t\right)\right)+ x^\top V_{t-1}^{-1}\left(\sum_{s=1\vee(t-w)}^{t-1}\eta_sX_s-\lambda\theta_t\right) \right|  \nonumber\\
\label{eq:sw24}\leq& \left\|x\right\|_2 \cdot \left\|V_{t-1}^{-1}\sum_{s=1\vee(t-w)}^{t-1}X_sX_s^{\top}\left(\theta_s-\theta_t\right)\right\|_2 + \left\|x\right\|_{V^{-1}_{t-1}}\left\|\sum_{s=1\vee(t-w)}^{t-1}\eta_sX_s-\lambda\theta_t\right\|_{V_{t-1}^{-1}}\\
\nonumber\leq&L \sum^{t-1}_{s = 1\vee (t-w)}\left\|\theta_s-\theta_{s+1}\right\|_2 + \beta \left\|x\right\|_{V^{-1}_{t-1}},		
\end{align}
where we have applied the triangle inequality and the Cauchy-Schwarz inequality successively in inequality \eqref{eq:sw24}.\halmos
\end{proof}
\section{Sliding Window-Upper Confidence Bound (\texttt{SW-UCB}) Algorithm: An Optimal Strategy with Known Variation Budgets}
\label{sec:swofu}
In this section, we describe the Sliding Window Upper Confidence Bound (\texttt{SW-UCB}) algorithm for the linear model. When the variation budget $B_T$ is known, we show that \swofu ~with a tuned window size achieves a dynamic regret bound which is optimal up to a multiplicative logarithmic factor. When the variation budget $B_T$ is unknown, we show that \swofu ~can still be implemented with a suitably chosen window size so that the regret dependency on $T$ is optimal, akin to that of \citep{KZ16}.
\subsection{Design Intuition and Design Details}
\label{sec:swofu_intuition}
\label{sec:swofu_details}
In the stochastic environment where the reward function is stationary, the well known UCB algorithm follows the principle of optimism in face of uncertainty \citep{ABF02,AYPS11}. Under this principle, the DM selects an action that maximizes the UCB, which is the value of ``mean plus confidence radius" \citep{ABF02} in each round. Following this principle, in each round $t,$ the \swofu~first computes the estimate $\hat{\theta}_t$ for $\theta_t$ according to eq. \eqref{eq:sw20} (one can set $\lambda=1$), and then constructs an UCB on the latent mean reward $\langle x, \theta_t\rangle$ for each action $x\in D_t.$ By Theorem \ref{theorem:sw_deviation}, the UCB of $x\in D_t$ in each round $t\in[T]$ is $\langle x,\hat{\theta}_t\rangle+L \sum^{t-1}_{s = 1\vee (t-w)}\left\|\theta_s-\theta_{s+1}\right\|+\beta\left\|x\right\|_{V^{-1}_{t-1}}.$ The \swofu~then choose the action $X_t$ with the highest UCB, \ie,
\begin{align}
	\label{eq:sw_policy}
	X_t=&\argmax_{x\in D_t}\left\{\langle x,\hat{\theta}_t\rangle+L \sum^{t-1}_{s = 1\vee (t-w)}\left\|\theta_s-\theta_{s+1}\right\|+\beta\left\|x\right\|_{V^{-1}_{t-1}}\right\}=\argmax_{x\in D_t}\left\{\langle x,\hat{\theta}_t\rangle+\beta\left\|x\right\|_{V^{-1}_{t-1}} \right\}.
\end{align}
Finally, the corresponding reward $Y_t$ is observed. The pseudo-code of the \swofu~is shown in Algorithm \ref{alg:swofu}.
\begin{algorithm}[!ht]
	\caption{\swofu ~for drifting linear bandits}
	\label{alg:swofu}
	\begin{algorithmic}[1]
		\State \textbf{Input:} Sliding window size $w$, dimension $d,$ variance proxy of the noise terms $R,$ upper bound of all the actions' Euclidean norms $L,$ upper bound of all the $\theta_t$'s Euclidean norms $S,$ and regularization constant $\lambda.$
		\State \textbf{Initialization:} $V_0\leftarrow\lambda I.$
		\For{$t=1,\ldots,T$}
		\State{Update $\hat{\theta}_t\leftarrow V_{t-1}^{-1}\left(\sum_{s=1\vee(t-w)}^{t-1}X_sY_s\right).$}
		\State{$X_t\leftarrow\argmax_{x\in D_t}\left\{x^\top\hat{\theta}_t+ \beta \left\|x\right\|_{V^{-1}_{t-1}} \right\},$ where $\beta$ is defined in (\ref{eq:sw_beta}).}
		\State{Observe $Y_t=\langle X_t,\theta_t\rangle+\eta_t.$} 
		\State{Update $V_t\leftarrow\lambda I+\sum_{s=1\vee(t-w+1)}^tX_sX_s^{\top}.$}
		\EndFor
	\end{algorithmic}
\end{algorithm}
\subsection{Dynamic Regret Analysis}
\label{sec:swofu_regret}
We are now ready to formally state a dynamic regret upper bound of the \swofu~for drifting linear bandits.
\begin{theorem}
	\label{theorem:sw_main}
For the drifting linear bandit setting, the dynamic regret of the \swofu~is upper bounded as $\R_T\left(\swofu\right)=\widetilde{O}\left(wB_T+dT/\sqrt{w}\right).$
	When $B_T$ is known, by taking $w=\Theta\left((dT)^{2/3}B_T^{-2/3}\right),$ the dynamic regret of the \swofu~is $\R_T\left(\swofu\right)=\widetilde{O}\left(d^{{2}/{3}}B_T^{{1}/{3}}T^{{2}/{3}}\right).$
	When $B_T$ is unknown, by taking $w=\Theta\left((dT)^{2/3}\right),$ the dynamic regret of the \swofu~is $\R_T\left(\swofu\right)=\widetilde{O}\left(d^{{2}/{3}}B_TT^{{2}/{3}}\right).$
\end{theorem}
\begin{proof}{Poof Sketch.}
	The complete proof is in Section \ref{sec:theorem:sw_main} of the appendix. Upon selecting $X_t,$ we have
	\begin{align}
	\label{eq:sw14}\langle x^*_t,\hat{\theta}_t\rangle+L \sum^{t-1}_{s = 1\vee (t-w)}\left\|\theta_s-\theta_{s+1}\right\|_2+\beta\left\|x^*_t\right\|_{V^{-1}_{t-1}}\leq&\langle X_t,\hat{\theta}_t\rangle+L \sum^{t-1}_{s = 1\vee (t-w)}\left\|\theta_s-\theta_{s+1}\right\|_2+ \beta\left\|X_t\right\|_{V^{-1}_{t-1}}
	\end{align}
	by virtue of the UCB action selection rule. From Theorem \ref{theorem:sw_deviation}, we further have with probability at least $1-\delta,$
	\begin{align}
	\label{eq:sw18}
	\langle x^*_t,\theta_t \rangle \leq \langle x^*_t,\hat{\theta}_t\rangle+ L \sum^{t-1}_{s = 1\vee (t-w)}\left\|\theta_s-\theta_{s+1}\right\|_2+\beta\left\|x^*_t\right\|_{V^{-1}_{t-1}}
	\end{align}
	and
	\begin{align}
	\label{eq:sw19}&\langle X_t,\hat{\theta}_t\rangle+L \sum^{t-1}_{s = 1\vee (t-w)}\left\|\theta_s-\theta_{s+1}\right\|_2+\beta\left\|X_t\right\|_{V^{-1}_{t-1}}\leq\langle X_t,\theta_t\rangle+ 2L \sum^{t-1}_{s = 1\vee (t-w)}\left\|\theta_s-\theta_{s+1}\right\|_2+ 2\beta\left\|X_t\right\|_{V^{-1}_{t-1}} .
	\end{align}
	Combining inequalities (\ref{eq:sw14}), (\ref{eq:sw18}), and (\ref{eq:sw19}), we establish the following high probability upper bound for the expected per round regret, \ie, with probability $1-\delta,$
	\begin{equation}\label{eq:sw-ucb}
	\langle x^*_t - X_t, \theta_t\rangle \leq 2 L \sum^{t-1}_{s = 1\vee (t-w)}\left\|\theta_s-\theta_{s+1}\right\|_2+2\beta\left\|X_t\right\|_{V^{-1}_{t-1}} .
	\end{equation}
	The regret upper bound of the \swofu~is thus
	\begin{align}
	\label{eq:decomp}
	2&\sum_{t\in[T]}L \sum^{t-1}_{s = 1\vee (t-w)}\left\|\theta_s-\theta_{s+1}\right\|_2+\beta\left\|X_t\right\|_{V^{-1}_{t-1}} 
	=\widetilde{O}\left(wB_T+\frac{dT}{\sqrt{w}}\right).
	\end{align} 
	If $B_T$ is known, the DM can set $w=\lfloor d^{2/3}T^{2/3}B_T^{-2/3}\rfloor$ and achieve a regret upper bound $\widetilde{O} (d^{2/3}B_T^{1/3}T^{2/3} ).$ If $B_T$ is not known, which is often the case in practice, the DM can set $w=\lfloor (dT)^{2/3}\rfloor$ to obtain a regret upper bound $\widetilde{O}(d^{2/3}(B_T+1)T^{2/3}).$ 	\halmos
\end{proof}
\begin{remark}
When the variation budget $B_T$ is known, Theorem  \ref{theorem:sw_main} recommends choosing the size $w$ of the sliding window to be decreasing with $B_T$. The recommendation is in agreement with the intuition that, when the learning environment becomes more volatile, the DM should focus on more recent observations. Indeed, if the underlying learning environment is changing at a higher rate, then the DM's past observations become obsolete faster. Theorem \ref{theorem:sw_main} pins down the intuition of forgetting past observation in face of drifting environments, by providing the mathematical definition of the sliding window size $w$ that yields the optimal dynamic regret bound. 
\end{remark}
\section{Bandit-over-Bandit (\texttt{BOB}) Algorithm: Adapting to the Unknown Variation Budget}
\label{sec:bob}
When $B_T$ is not known, the DM can achieve the dynamic regret bound $\widetilde{O}\left(d^{2/3}(B_T+1)T^{2/3}\right)$ for the drifting linear bandit problem, by setting $w = \Theta((dT)^{2/3})$ (see Section \ref{sec:swofu}). While the bound is optimal in terms of $T$ by Theorem \ref{theorem:lower_bound}, the bound becomes trivial when $B_T=\Omega(T^{1/3})$, since then the resulting dynamic regret bound is linear in $T$. 

To mitigate this issue, we make use of the \swofu~as a sub-routine, and ``hedge" \citep{ABFS02,AB09} against the (possibly adversarial) changes of $\theta_t$'s to identify a reasonable fixed window size. Inspired by the heuristic envelop policy \citep{BGZ18} and the bandit corralling technique \citep{ALNS17,LWAL18}, we develop a novel Bandit-over-Bandit (\texttt{BOB}) algorithm that achieves a nearly optimal dynamic regret bound without knowing $B_T$. Specifically, we show that the \bob~has a dynamic regret sub-linear in $T$ even when $B_T =o(T)$ is not known, unlike the \swofu. Similar to the style of previous sections, the discussion in this section focuses on linear model. Nevertheless, we emphasize that the proposed framework applies to a variety of bandit models (see the forthcoming Section \ref{sec:applications}).
\subsection{Design Intuition and Design Details}
\label{sec:bob_details}
As illustrated in Fig. \ref{fig:bob_int}, the \bob~divides the whole time horizon into $\lceil T/H\rceil$ blocks of equal length $H$ rounds (the last block can possibly have less than $H$ rounds). In addition, the algorithm specifies a set of candidate window sizes $J$. For each block $i\in\left[\lceil T/H\rceil\right]$, the \bob~first selects a window size $w_i \in J$. Then, the \texttt{BOB} algorithm restarts the \swofu~from scratch (see Remark \ref{remark:bob_design} for a discussion on the design of restarting) with the selected window size $w_i$ for $H$ rounds. On top of this, the \bob~also maintains a separate bandit algorithm to determine each window size $w_i$ based on the observed history in the previous $i-1$ blocks, and thus the name Bandit-over-Bandit. The choice of $w_i$ is based on the EXP3 algorithm \citep{ABFS02}, which allows us to compete with the best window size in $J$ (in the sense of minimizing dynamic regret), even when the $\theta_t$'s variation does not follow any pattern. The EXP3 algorithm is designed for adversarial multi-armed bandits, where the underlying reward function is designed by an oblivious adversary \citep{ABFS02,AB09}. Finally, to properly apply the EXP3 algorithm, we note that the  total reward during each block is normalized so that the normalized reward lies in $[0, 1]$ with high probability. 
\begin{figure}[!ht]
	\centering
	\includegraphics[width=13cm,height=4.1cm]{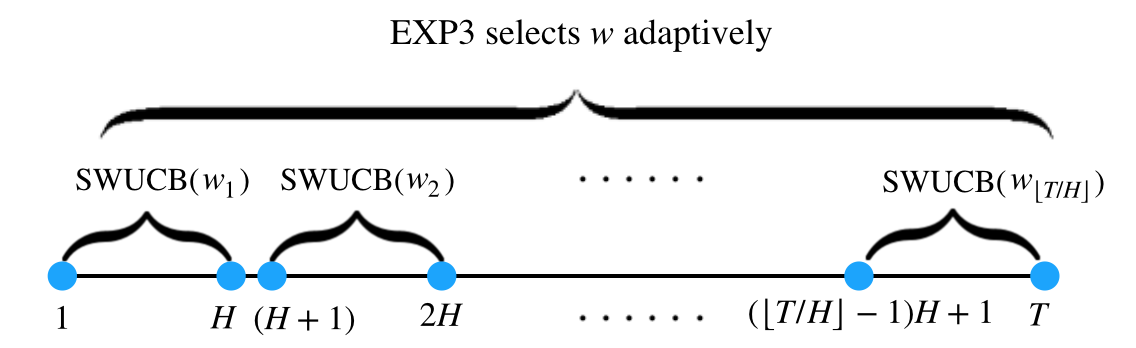}
	\caption{Structure of the \bob}
	\label{fig:bob_int}
\end{figure}

\begin{algorithm}[!ht]
	\caption{\bob~for drifting linear bandits}
	\label{alg:bob}
	\begin{algorithmic}[1]
		\State \textbf{Input:} Time horizon $T$, the \swofu, parameters $H, \Delta, J,Q$ (as defined in \ref{eq:bob_parameters}).
		\State \textbf{Initialize} parameters 
		$\gamma, \{s_{j, 1}\}^{\Delta}_{j=0}$ by eq. (\ref{eq:bob_parameters1}).		
		\For{$i=1,2,\ldots,\lceil T/H\rceil$}
		\State {Define distribution $(p_{j, i})^{\Delta}_{j=0}$ by eq. (\ref{eq:p_j_i}), and set $j_t\leftarrow j$ with probability $p_{j,i}$.}
		\State{Set the window size $w_i\leftarrow\left\lfloor H^{j_t/\Delta}\right\rfloor$.}
		\State{Restart the \swofu~for $H$ rounds with window size $w_i$.}\label{alg:template}
		\State{Update $s_{j_i,i+1}$ according to eq. (\ref{eq:s_j_i}), and $s_{u,i+1}\leftarrow s_{u,i}~\forall u\neq j_i$}
		\EndFor
	\end{algorithmic}
\end{algorithm}

To this end, we describe the details of the \bob, displayed in Algorithm \ref{alg:bob},  for the linear bandit model. Define the parameters (we justify these choices in Section \ref{sec:bob_intuition})
\begin{align}
	\label{eq:bob_parameters}
	&H=\left\lfloor dT^{\frac{1}{2}}\right\rfloor,\Delta=\lceil\ln H\rceil,J=\left\{H^0,\left\lfloor H^{\frac{1}{\Delta}}\right\rfloor,\ldots, H\right\},Q=2H+4R\sqrt{H\ln (T/\sqrt{H})}.
\end{align} The \bob~first divides the time horizon $T$ into $\lceil T/H\rceil$ blocks of length $H$ rounds (except for the last block, which can be less than $H$ rounds), and then initiates the parameters 
\begin{align}
	\label{eq:bob_parameters1}
	&\gamma=\min\left\{1,\sqrt{\frac{(\Delta+1)\ln(\Delta+1)}{(e-1)\lceil T/H\rceil}}\right\},
	s_{j,1}=1\quad\forall j=0,1,\ldots,\Delta.
\end{align}
for the EXP3 algorithm \citep{ABFS02}.  At the beginning of each block $i\in\left[\lceil T/H\rceil\right],$ the \bob~first sets
\begin{align}\label{eq:p_j_i}
	p_{j,i}=(1-\gamma)\frac{s_{j,i}}{\sum_{u=0}^{\Delta}s_{u,i}}+\frac{\gamma}{\Delta+1}\quad\forall j=0,1,\ldots,\Delta,
\end{align}
and then sets $j_i=j$ with probability $p_{j,i}$ for each $j=0,1,\ldots,\Delta.$ The selected window size is then $w_i=\left\lfloor H^{j_i/\Delta}\right\rfloor.$ Afterwards, the \bob~selects actions $X_t$ by running the \swofu~with window size $w_i$ for each round $t$ in block $i,$ and the total collected reward is 
$$\sum_{t=(i-1)H+1}^{i\cdot H\wedge T}Y_t=\sum_{t=(i-1)H+1}^{i\cdot H\wedge T}\langle X_t,\theta_t\rangle+\eta_t.$$ 
Finally, the total rewards is normalized by first dividing $Q,$ and then added by $1/2$ so that it lies within $[0,1]$ with high probability. The parameter $s_{j_i,i+1}$ is set to
\begin{align}\label{eq:s_j_i}
	s_{j_i,i}\cdot\exp\left(\frac{\gamma}{(\Delta+1)p_{j_i,i}}\left(\frac{1}{2}+\frac{\sum_{t=(i-1)H+1}^{i\cdot H\wedge T}Y_t}{Q}\right)\right);
\end{align}
while $s_{u,i+1}$ is the same as $s_{u,i}$ for all $u\neq j_i.$ 

\subsection{Dynamic Regret Analysis}
We are now ready to present the dynamic regret bound for the \bob.
\begin{proposition}\label{prop:bob}
	For the drifting linear bandit setting, the dynamic regret of the \bob~is 
	\begin{align}
	\label{eq:bob_regret}&\R_T(\bob)=\widetilde{O}\left(w^{\dag}B_T+\frac{dT}{\sqrt{w^{\dag}}}+Q\sqrt{\frac{|J|T}{H}}\right).
	\end{align}
\end{proposition}
\begin{proof}{Proof Sketch.} The complete proof is presented in Section \ref{sec:prop:bob} of the appendix.
	The dynamic regret bound (\ref{eq:bob_regret}) can be decomposed as
	\begin{equation}\label{eq:bob_lin_decompose}
	\underbrace{\widetilde{O}\left(w^{\dag}B_T+\frac{dT}{\sqrt{w^{\dag}}}\right)}_{\R_T\left(\swofu\right)\text{ with $w^{\dag}$}}+ \underbrace{\widetilde{O}\left(Q\sqrt{\frac{|J|T}{H}}\right)}_{\text{Loss in learning $w^{\dag}$}}.
	\end{equation}
	The first term in (\ref{eq:bob_lin_decompose}) is due to the dynamic regret of the underlying \swofu~under the optimally tuned window size $w^\dag$. More precisely, we can view each block as a new non-stationary linear bandit instance, and the dynamic regret is due to the application of \swofu~with window size $w^{\dag}$ on each block. The second term in (\ref{eq:bob_lin_decompose}) is due to the loss by the EXP3 algorithm, which essentially treat each of the window size in $J$ as an expert, and compete with the best expert. Here, we point out due to the design of restarting, any instance of the \swofu~cannot last for more than $H$ rounds. As a consequence, even if the EXP3 algorithm selects a window size $w_i>H$ for some block $i,$ the effective window size is $H.$ In other words, $w^*$ is not necessarily attainable, \ie, by definition, $w^*=\left\lfloor (dT)^{2/3}B_T^{-2/3}\right\rfloor$ might be larger than $H$ when $B_T$ is small. We thus have to denote the optimally (over $J$) tuned window size as $w^{\dag}.$  	\halmos
\end{proof}

\begin{theorem}
	\label{theorem:bob}With the parameters specified in Section \ref{sec:bob_details}, the dynamic regret of the \bob~for drifting linear bandit is
	$\R_T\left(\bob\right)=\widetilde{O}\left(d^{{2}/{3}}B_T^{{1}/{3}}T^{{2}/{3}}+d^{{1}/{2}}T^{{3}/{4}}\right).$
\end{theorem}
The proof of Theorem \ref{theorem:bob} can be found in Section \ref{sec:theorem:bob} of the appendix. In the next section, we discuss the choice of parameters in \eqref{eq:bob_parameters} and discuss its relationship

\subsection{Choices of Parameters and Justifications}
\label{sec:bob_intuition}
We first justify the choice of $Q$ in \eqref{eq:bob_parameters}. Note that $Q$ is used to perform normalization, we thus prove high probability upper and lower bounds for the total rewards of each block (here, we prove a slightly more general result by allowing $\max_{t\in[T],x\in D_t}|\langle x,\theta_t\rangle|$ to be in $[-\nu,\nu]$ for some $\nu>0$). 
\begin{lemma}
	\label{lemma:bob}
	Suppose $\max_{t\in[T],x\in D_t}|\langle x,\theta_t\rangle|\in[-\nu,\nu]$ for some $\nu>0$ and denote $M_i$ as the absolute value of cumulative rewards for block $i$, then with probability at least $1-2/T,$ $M_i$ does not exceed $H\nu+2R\sqrt{H\ln(T/\sqrt{H})}$ for all $i,$ \ie, $\Pr\left(\forall i\in\lceil T/H\rceil \quad M_i\leq H\nu+2R\sqrt{H\ln\frac{T}{\sqrt{H}}}\right)\geq 1-\frac{2}{T}.$
\end{lemma}
The complete proof of Lemma \ref{lemma:bob} is in Section \ref{sec:lemma:bob} of the appendix. With Lemma \ref{lemma:bob} and the choice of $Q=2H+4R\sqrt{H\ln(T/\sqrt{H})}$ (note that $\nu=1$ by our model assumption in Section \ref{sec:formulation}), it is evident that ${\sum_{t=(i-1)H+1}^{i\cdot H\wedge T}Y_t}/{Q}$ in eq. \eqref{eq:s_j_i} lies in $[-1/2,1/2]$ with probability at least $1-2/T.$ Adding this by $1/2,$ we normalize the total rewards of each block to $[0,1]$ with probability at least $1-2/T$ for all the blocks.

To determine $H,\Delta,$ and $J$, we consider the dynamic regret bound of the \bob~as stated in Proposition \ref{prop:bob}. Eq. (\ref{eq:bob_regret}) in Proposition \ref{prop:bob} exhibits a similar structure to the regret of the \swofu~as stated in Theorem \ref{theorem:sw_main}, and this immediately indicates a clear trade-off in the design of the block length $H:$ 
\begin{itemize}
	\item On one hand, $H$ should be small to control the regret incurred by the EXP3 algorithm in identifying $w^{\dag},$ \ie, the third term in eq. (\ref{eq:bob_regret}).
	\item On the others, $H$ should also be large enough to allow $w^{\dag}$ to get close to $w^*=\lfloor (dT)^{2/3}B_T^{-2/3}\rfloor$ so that the sum of the first two terms in eq. (\ref{eq:bob_regret}) is minimized. 
\end{itemize} A more careful inspection also reveals the tension in the design of $J.$ Obviously, we hope that $|J|$ is small to minimize the third term in eq. (\ref{eq:bob_regret}), but we also wish $J$ to be dense enough so that it forms a cover to the set $[H].$ Otherwise, even if $H$ is large enough that $w^{\dag}$ can approach $w^*,$ approximating $w^*$ with any element in $J$ can cause a major loss.

These observations suggest the following choice of $J.$ 
\begin{align}
	J=\left\{H^0,\left\lfloor H^{\frac{1}{\Delta}}\right\rfloor,\ldots, H\right\}
\end{align}
for some positive integer $\Delta,$ and since the choice of $H$ should not depend on $B_T,$ we can set $H=\left\lfloor d^{\epsilon}T^{\alpha}\right\rfloor$ 
with some $\alpha\in[0,1]$ and $\epsilon>0$ to be determined. We then distinguish two cases depending on whether $w^*$ is smaller than $H$ or not (or alternatively, whether $B_T$ is larger than $d^{(2-3\epsilon)/2}T^{(2-3\alpha)/2}$ or not).
\paragraph{Case 1: $w^*\leq H$ or $B_T\geq d^{(2-3\epsilon)/2}T^{(2-3\alpha)/2}.$} Under this situation, $w^{\dag}$ can automatically adapt to the nearly optimal window size $\text{clip}_J\left(w^*\right)$ , where $\text{clip}_J(x)$ finds the largest element in $J$ that does not exceed $x.$  Notice that $|J|=\Delta+1,$ the dynamic regret of the \bob~then becomes
\begin{align}
	\nonumber\R_T(\bob)=&\widetilde{O}\left(w^{\dag}B_T+\frac{dT}{\sqrt{w^{\dag}}}+\sqrt{H|J|T}\right)\\
	\nonumber=&\widetilde{O}\left(w^*H^{\frac{1}{\Delta}}B_T +\frac{dT}{\sqrt{w^*H^{-1/\Delta}}}+\sqrt{d^{\epsilon}T^{{\alpha+1}}\Delta}\right)\\
	\label{eq:bob_regret1}=&\widetilde{O}\left(d^{\frac{2}{3}}\left(B_T+1\right)^{\frac{1}{3}}T^{\frac{2}{3}}H^{\frac{1}{\Delta}} +d^{\frac{\epsilon}{2}}T^{\frac{\alpha+1}{2}}\Delta^{\frac{1}{2}}\right).
\end{align}
\paragraph{Case 2: $w^*> H$ or $B_T< d^{(2-3\epsilon)/2}T^{(2-3\alpha)/2}.$} Under this situation, $w^{\dag}$ equals to $H,$ which is the window size closest to $w^*,$ the regret of the \bob~then becomes
\begin{align}
	\nonumber\R_T(\bob)=&\widetilde{O}\left(w^{\dag}B_T+\frac{dT}{\sqrt{w^{\dag}}}+\sqrt{H|J|T}\right)\\
	\nonumber=&\widetilde{O}\left(HB_T+\frac{dT}{\sqrt{H}}+\sqrt{H|J|T}\right)\\
	\nonumber=&\widetilde{O}\left(d^{\epsilon}\left(B_T+1\right)T^{\alpha}+d^{1-\frac{\epsilon}{2}}T^{\frac{2-\alpha}{2}}++d^{\frac{\epsilon}{2}}T^{\frac{\alpha+1}{2}}\Delta^{\frac{1}{2}}\right)\\
	\label{eq:bob_regret2}=&\widetilde{O}\left(d^{1-\frac{\epsilon}{2}}T^{\frac{2-\alpha}{2}}+d^{\frac{\epsilon}{2}}T^{\frac{\alpha+1}{2}}\Delta^{\frac{1}{2}}\right),
\end{align}
where we have make use of the fact that $B_T< d^{(2-3\epsilon)/2}T^{(2-3\alpha)/2}$ in the last step.

Now both eq. (\ref{eq:bob_regret1}) and eq. (\ref{eq:bob_regret2}) suggests that we should set $\Delta=\lceil\ln H\rceil,$ and eq. (\ref{eq:bob_regret2})  further reveals that we should take $\alpha=1/2$ and $\epsilon=1.$ These then lead to the choice of parameters presented in eq. \eqref{eq:bob_parameters}, \ie,
$H=\left\lfloor dT^{\frac{1}{2}}\right\rfloor,\Delta=\lceil\ln H\rceil,J=\left\{H^0,\left\lfloor H^{\frac{1}{\Delta}}\right\rfloor,\ldots, H\right\}.$ Here we have to emphasize that $w^{\dag},\alpha,$ and $\epsilon$ are used only in the analysis, while the only parameters that we need to decide are $H,\Delta,J,$ and $Q,$ which clearly do not depend on $B_T.$
\subsection{Further Remarks Regarding the \bob}
\begin{remark}[Removing Assumption \ref{ass:reg}]\label{remark:ass}
To remove Assumption \ref{ass:reg}, one can apply a restarting strategy \citep{BGZ18} together with an algorithm for adversarial linear bandit, \eg, Algorithm 15 of \cite{LS18}. When $B_T$ is known and $D_t$'s are fixed, by an argument similar to Theorem 2 of \cite{BGZ18}, one can show that this restarting strategy can achieve the minimax-optimal dynamic regret bound $\widetilde{O}(d^{2/3}B_T^{1/3}T^{2/3});$ when $B_T$ is unknown, we can apply the \bob~to adaptively tune the restarting rate to achieve the dynamic regret bound $\widetilde{O}(d^{2/3}B_T^{1/3}T^{2/3}+d^{1/2}T^{3/4}).$
\end{remark}
\begin{remark}[Algorithm's Optimality]
	Compared with the lower bound of Theorem \ref{theorem:lower_bound}, the dynamic regret bound presented in Theorem \ref{theorem:bob} is optimal when $B_T\geq d^{-1/2}T^{1/4};$ while it also leaves a small $O(T^{1/12})$ gap in the worst case \ie, when $B_T=\Theta(1).$ This is because for the \bob, the smaller the amount of non-stationarity (as quantified in the left hand side of (\ref{eq:variation_budget})), the harder it is for the EXP3 algorithm to detect the amount of non-stationarity, resulting in a worse dynamic regret bound. Indeed, the worst possible case for our analysis is when $B_T=dT^{-1/2}$ according to Theorem \ref{theorem:lower_bound}.
\end{remark}
\begin{remark}[Failure of Naive Learning of $B_T$]
	Theorem \ref{theorem:sw_main} shows that running the \swofu~for $T$ with window size $
	w^*=\left\lfloor (dT)^{2/3}B_T^{-2/3}\right\rfloor$
	leads to an optimal dynamic regret. However, the choice of the window size $w^*$ requires the crucial knowledge of $B_T$, which is not available to the DM. A natural attempt would be to ``learn" the unknown $B_T$ in order to properly tune the window size $w$. In a more restrictive setting in which the differences between consecutive $\theta_t$'s follow some underlying stochastic process, one possible approach is to apply a suitable machine learning technique to learn the underlying stochastic process and tune the parameter $w$ accordingly. However, under the general setting of drifting environments (\ref{eq:variation_budget}), the differences between consecutive $\theta_t$'s need not follow any pattern, which challenges the use of statistical machine learning algorithms for identifying the patterns on the underlying changes. 
\end{remark}
\begin{remark}[Restarting Structure of the \bob]\label{remark:bob_design}
	The block structure and restarting the \swofu~with a single window size for each block are essential for the correctness of the \bob. Otherwise, suppose the DM utilizes the EXP3 algorithm to select the window size $w_t$ for each round $t,$ and implements the \swofu~with the selected window size without ever restarting it. Instead of eq. (\ref{eq:bob_decompose}), the regret of the \bob~is then decomposed as 
	\begin{align}
		\label{eq:disc}
		\nonumber&\sum_{t=1}^T\left(\text{Reward of }\texttt{SW-UCB}\left(\left\{w^{\dag}\right\}_{\tau=1}^t\right)\text{ in round $t$}-\text{Reward of }\texttt{SW-UCB}\left(\left\{w_{\tau}\right\}_{\tau=1}^t\right)\text{ in round $t$}\right)\\
		&+\sum_{t=1}^T\left(\text{Optimal reward in round }t-\text{Reward of }\texttt{SW-UCB}\left(\left\{w^{\dag}\right\}_{\tau=1}^t\right)\text{ in round $t$}\right)
	\end{align}  
	Here, with some abuse of notations, $\texttt{SW-UCB}(\{w^{\dag}\}_{\tau=1}^t)$ (respectively $(\texttt{SW-UCB}( \{w_{\tau}\}_{\tau=1}^t )$) refers to in round $t,$ the DM runs the \swofu~with window size $w^{\dag}$ (respectively $w_t$) and historical data, \eg, (action, reward) pairs, generated by running the \swofu~with window size $w^{\dag}$ (respectively $w_{\tau}$) for rounds $\tau=1,\ldots,t-1.$ Same as before, the second term of eq. (\ref{eq:disc}) can be upper bounded as a result of Theorem \ref{theorem:sw_main}. It is also tempting to apply results from the EXP3 algorithm to upper bound the first term. Unfortunately, this is incorrect as it is required by the adversarial bandits protocol \citep{ABFS02} that the DM and its competitor should receive the same reward if they select the same action, \ie, the reward of $\texttt{SW-UCB}\left(\left\{w^{\dag}\right\}_{\tau=1}^{t-1},w_t=w\right)$ in round $t$ and the reward of $\texttt{SW-UCB}\left(\left\{w_{\tau}\right\}_{\tau=1}^{t-1},w_t=w^{\dag}\right)$ in round $t$ should be the same for every $w.$ Nevertheless, this is violated as running the \swofu~with different window sizes for previous rounds can generate different (action,reward) pairs, and this results in possibly different estimated $\hat{\theta}_t$'s for the two \swofu s even if both of them use the same window size in round $t.$ Hence, the selected actions and the corresponding reward by these two instances might also be different. By the careful design of blocks as well as the restarting scheme, the \bob~decouples the \swofu~for a block from previous blocks, and thus fixes the above mentioned problem, \ie, the regret of the \bob~is decomposed as eq. (\ref{eq:bob_decompose}).
\end{remark}
\begin{remark}[Applications]
	The Bandit-over-Bandit framework can go beyond the problem of non-stationary bandit optimization. In a high level, it provides us a viable approach to automatically optimize the performances of data-driven sequential decision-making algorithms. Although not always optimal, it can be applied to bandit model selection \citep{FosterKL19} as well as online meta-learning \citep{BastaniSLZ19}, in which the DM is trying to optimize the performances of her algorithms by selecting a correct model class or a set of proper parameters. Both of these are of great importance in the operations of data-driven decision-making algorithms.
\end{remark}
\section{Extensions to Other Bandit Models}
\label{sec:applications}
In this section, we demonstrate the generality of our established results. As illustrative examples, we apply our technique to several bandit settings, including multi-armed bandits \citep{ABF02},  the generalized linear bandits  \citep{FCAS10,LLZ17}, and the combinatorial semi-bandits \citep{GKJ12,KWAS15}. A preview of the results is shown in Table \ref{table:application}. Note that for generalized linear bandits, we need to impose Assumption \ref{ass:reg}. On the other hand, for multi-armed bandits, this assumption is always valid while for combinatorial semi-bandits, this assumption is not required.
\begin{table}[!ht]
	\begin{center}
		\begin{tabular}{|c|c|c|} 
			\hline
			&Known $B_T$&Unknown $B_T$\\
			\hline
			$d$-armed bandit&$\widetilde{O}\left(d^{{1}/{3}}B_T^{{1}/{3}}T^{{2}/{3}}\right)$&$\widetilde{O}\left(d^{{1}/{3}}B_T^{{1}/{3}}T^{{2}/{3}}+d^{{1}/{4}}T^{{3}/{4}}\right)$\\
			\hline
			Generalized linear bandit&$\widetilde{O}\left(d^{{2}/{3}}B_T^{{1}/{3}}T^{{2}/{3}}\right)$&$\widetilde{O}\left(d^{{2}/{3}}B_T^{{1}/{3}}T^{{2}/{3}}+d^{{1}/{2}}T^{{3}/{4}}\right)$\\
			\hline
			Combinatorial semi-bandit&$\widetilde{O}\left(d^{{1}/{3}}m^{2/3}B_T^{{1}/{3}}T^{{2}/{3}}\right)$&$\widetilde{O}\left(d^{1/3}m^{2/3}B_T^{1/3}T^{2/3}+d^{1/4}m^{3/4}T^{3/4}\right)$\\
			\hline
		\end{tabular}
		\caption{Dynamic regret bounds of the \swofu~and the \bob~for different settings. Here $m$ is an upper bound for the 1-norm of all the actions in the combinatorial semi-bandit problem.}
		\label{table:application}
	\end{center}
\end{table}
\subsection{An Algorithmic Template}
The \swofu~and the \bob~developed in the previous sections can be viewed as an algorithmic template that allows us to extend the results from linear bandits to other bandit settings. Given a bandit setting $\texttt{A}$, we leverage the forgetting principle (similar to Section \ref{sec:swlse}), and first modify the reward estimator used in the stationary setting to a sliding-window estimator. We then incorporate it into the UCB algorithm to arrive at the corresponding \swofu~for the drifting environments. When the variation budget is known, we could optimally tune the window size to enjoy an optimal dynamic regret bound. To achieve low dynamic regret when the variation budget is unknown, we can proceed by plugging the \swofu~for $\texttt{A}$ into the \bob, \ie, line \ref{alg:template} of Algorithm \ref{alg:bob}, and custom-tailor the parameters (as those listed in eq. \eqref{eq:bob_parameters}) to accommodate the need of $\texttt{A}.$

We note that the power of this algorithmic template is indeed entailed by a salient property,  \ie, the dynamic regret of the \swofu~can be decomposed as ``dynamic regret of drift" + ``dynamic regret of uncertainty" (or eq. \eqref{eq:decomp}), that actually holds for a variety of bandit learning models in addition to linear models. In what follows, we shall derive the \swofu~as well as the parameters required by the \bob, \ie, similar to those defined in eq. (\ref{eq:bob_parameters}), for each of the above mentioned settings. 
\subsection{$d$-Armed Bandits}
The $d$-armed bandit problem in drifting environments was first studied by \citep{BGZ15}, who proposed Rexp3, an innovative and interesting variant of the EXP3 algorithm \citep{ABFS03}. When the underlying variation budget is known, their algorithm achieves the optimal dynamic regret bound. In this subsection, we provide an alternative derivation of the dynamic regret bound by our framework. 

In the $d$-armed bandits setting, every action set $D_t$ is comprised of $d$ actions $e_1,\ldots,e_d.$ The $i^{\text{th}}$ action $e_i$ has coordinate $i$ equals to 1 and all other coordinates equal to $0.$ Therefore, the reward of choosing action  $X_t=e_{I_t}$ in round $t$ is $Y_t=\langle X_t,\theta_t\rangle+\eta_t=\theta_{t}(I_t)+\eta_t,$ where $\theta_{t}(I_t)$ is the $I_t^{\text{th}}$ coordinate of $\theta_t.$ We again assume $|\langle x,\theta_t\rangle|\in[-1,1]$ for all $x\in D_t$ and all $t\in[T].$ Different than the linear bandit setting, we follow \citep{BGZ15,BGZ18} to define the variation budget with the infinity norm, \ie, $\sum_{t=1}^{T-1}\left\|\theta_{t+1}-\theta_t\right\|_{\infty}\leq B_T.$
For a window size $w,$ we also define $N_{t-1}(i)$ as the number of times that action $i$ is chosen within rounds $(t-w),\ldots,(t-1),$ \ie, for all $i\in[d],$ 
$N_{t-1}(i)=\sum_{s=1\wedge(t-w)}^{t-1}\bm{1}[X_t=e_i].$
Here $\bm{1}[\cdot]$ is the indicator function. Similar to the procedure in Section \ref{sec:swlse}, we set the regularization parameter $\lambda=0,$ and compute the sliding window least squares estimate $\hat{\theta}_t$ for $\theta_t$ in each round, \ie,
\begin{align}
\label{eq:sw21}
&\hat{\theta}_t = V_{t-1}^{*}\left( \sum^{t-1}_{s = 1\vee(t-w)}X_s Y_s \right),
\end{align}
where $V_{t-1}^*$ is Moore-Penrose pseudo-inverse of $V_{t-1}.$ We can also derive the error bound for the latent expected reward of every action $x\in D_t$ in any round $t.$
\begin{theorem}
	\label{theorem:mab_deviation}
	For any $t\in[T]$ and any $i\in[d],$ we have with probability at least $1-{1}/{T},$ 
$\left|e_i^\top ( \hat{\theta}_t - \theta_t)\right|\leq\sum^{t-1}_{s = 1\vee (t-w)}\left\|\theta_s-\theta_{s+1}\right\|_{\infty}+R\sqrt{{2\ln\left({2dT^2}\right)}}\left\|e_i\right\|_{V^{*}_{t-1}}.$	holds for all $x\in D_t.$
\end{theorem}
The complete proof is provided in Section \ref{sec:theorem:mab_deviation} of the appendix. We can now follow the same principle in Section \ref{sec:swofu} by choosing in each round the action $X_t$ with the highest UCB, \ie,
\begin{align}
\label{eq:mab_sw_policy}
X_t=&\argmax_{x\in D_t}\left\{\langle x,\hat{\theta}_t\rangle+R\sqrt{{2\ln\left({2dT^2}\right)}}\left\|x\right\|_{V^{*}_{t-1}} \right\},
\end{align}
and arrive at the following regret upper bound for the \swofu.
\begin{theorem}
	\label{theorem:mab_sw_main}
	For the $d$-armed bandit setting, the dynamic regret of the \swofu~is upper bounded as $
	\R_T\left(\swofu\right)=\widetilde{O}\left(wB_T+\sqrt{d}T/\sqrt{w}\right).$ When $B_T~(>0)$ is known, by taking $w=\Theta\left(d^{1/3}T^{2/3}B_T^{-2/3}\right),$ the dynamic regret of the \swofu~is $\R_T\left(\swofu\right)=\widetilde{O}\left(d^{1/3}B_T^{1/3}T^{2/3}\right).$
	When $B_T$ is unknown, by taking $w=\Theta\left(d^{1/3}T^{2/3}\right),$ the dynamic regret of the \swofu~is $\R_T\left(\swofu\right)=\widetilde{O}\left(d^{1/3}B_TT^{2/3}\right).$
\end{theorem}
\begin{proof}{Proof Sketch.}
	The proof of this theorem is very similar to that of Theorem \ref{theorem:sw_main}, and is thus omitted. The key difference is that $\beta$ (defined in eq. (\ref{eq:sw_beta}) for the linear bandit setting) is now set to $R\sqrt{2\ln\left(2dT^2\right)},$ and this saves the extra $\sqrt{d}$ factor presented in eq. (\ref{eq:explicit_swucb_bd}). Hence the dynamic regret bound can be obtained accordingly.\halmos
\end{proof}
Comparing the results obtained in Theorem \ref{theorem:mab_sw_main} to the lower bound presented in \citep{BGZ15}, we can easily see that the dynamic regret bound is optimal when $B_T$ is known. When $B_T$ is unknown, we can implement the \bob~with the following parameters:
\begin{align}
\label{eq:mab_bob_parameters}
&H=\left\lfloor \left(dT\right)^{\frac{1}{2}}\right\rfloor,\Delta=\lceil\ln H\rceil,J=\left\{H^0,\left\lfloor H^{\frac{1}{\Delta}}\right\rfloor,\ldots, H\right\},Q=2H+4R\sqrt{H\ln (T/\sqrt{H})}.
\end{align}
The regret of the \bob~for the MAB setting is characterized as follows.
\begin{theorem}
	\label{theorem:mab_bob}
	The dynamic regret of the \bob~for the $d$-armed bandit setting is $\R_T\left(\bob\right)=\widetilde{O}\left(d^{1/3}B_T^{1/3}T^{2/3}+d^{1/4}T^{3/4}\right).$
\end{theorem}
The proof of the theorem is very similar to Theorem \ref{theorem:bob}'s, and it is thus omitted.
\subsection{Generalized Linear Bandits}
For the generalized linear bandits model, we adopt the setup in \citep{FCAS10,LLZ17}: it is essentially the same as the linear bandit setting except that the decision set is time invariant, \ie, $D_t=D$ for all $t\in[T],$ and the reward of choosing action $X_t\in D$ is $Y_t=\mu\left(\langle X_t,\theta_t\rangle\right)+\eta_t.$

Let $\dot{\mu}(\cdot)$ and $\ddot{\mu}(\cdot)$ denote the first derivative and second derivative of $\mu(\cdot)$, respectively, we follow \citep{FCAS10} to make the following assumption. 
\begin{assumption}\label{ass:glm}
	i) There exists a set of $d$ actions $a_1,\ldots,a_d\in D$ such that the minimal eigenvalue of $\sum_{i=1}^da_ia_i^{\top}$ is $\lambda_0~(>0).$ ii) The link function $\mu(\cdot):\Re\to\Re$ is strictly increasing, continuously differentiable, Lipschitz with constant $k_{\mu},$ and we define $c_{\mu}=\inf_{x\in D,\theta\in\Re^d:\|\theta\|\leq S}\dot\mu\left(\langle x,\theta\rangle\right).$ iii) There exists $Y_{\max}>0$ such that for any $t\in[T],$ $Y_t\in\left[0,Y_{\max}\right].$
\end{assumption}
Similar to the procedure in Section \ref{sec:swlse}, we compute the maximum quasi-likelihood estimate $\hat{\theta}_t$ for $\theta_t$ in each round $t\in[T]$ by solving the equation
\begin{align}
\label{eq:glm}
\sum^{t-1}_{s = 1\vee(t-w)}\left(Y_s-\mu\left(\left\langle X_s,\hat{\theta}_t\right\rangle\right)\right)X_s=0.
\end{align}
Defining $\beta=2k_{\mu}Y_{\max}\sqrt{2d\ln(w)\ln(2dT^2)\left(3+2\ln\left(1+2{L^2}/{\lambda_0}\right)\right)}/c_{\mu},$ we can also derive the deviation inequality type bound for the latent expected reward of every action $x\in D_t$ in any round $t.$ Here, as pointed out in \cite{FauryRAC21}, we need to assume that $\|\hat{\theta}_t\|\leq S$ holds for every $t\in[T]$. Otherwise, we need to perform a projection step similar to \cite{FCAS10,FauryRAC21}.
\begin{theorem}
	\label{theorem:glm_sw_deviation}
For any $t\in[T],$ we have with probability at least $1-1/T,$ 
$\left|\mu\left(x^\top \hat{\theta}_t\right) - \mu\left(x^{\top}\theta_t\right)\right|\leq \frac{k^2_{\mu}L}{c_{\mu}} \sum^{t-1}_{s = 1\vee (t-w)}\left\|\theta_s-\theta_{s+1}\right\|_2+\beta\left\|x\right\|_{V^{-1}_{t-1}}$
	holds for all $x\in D_t.$ 
\end{theorem} 
\begin{proof}{Proof Sketch.}
	The proof is a consequence of Proposition 1 of \citep{FCAS10} and Theorem \ref{theorem:sw_deviation}. Please refer to Section \ref{sec:theorem:glm_sw_deviation} of the appendix for the complete proof.
	\halmos
\end{proof}
We can now follow the same principle in Section \ref{sec:swofu} to design the \swofu. Note that in order for $V_{t-1}$ to be invertible for all $t,$ our algorithm should select the actions $a_1,\ldots,a_d$ every $w$ rounds for some window size $w.$ For each of the remaining round $t,$ it chooses the action $X_t$ with the highest UCB, \ie,
\begin{align}
\label{eq:glm_sw_policy}
X_t=&\argmax_{x\in D_t}\left\{\langle x,\hat{\theta}_t\rangle+\beta\left\|x\right\|_{V^{*}_{t-1}} \right\},
\end{align}
and arrive at the following regret upper bound.
\begin{theorem}
	\label{theorem:glm_sw_main}
For the drifting generalized linear bandit setting, the dynamic regret of the \swofu~is upper bounded as $\R_T\left(\swofu\right)=\widetilde{O}\left(wB_T+dT/\sqrt{w}\right).$ When $B_T~(>0)$ is known, by taking $w=\Theta\left((dT)^{2/3}B_T^{-2/3}\right),$ the dynamic regret of the \swofu~is $\R_T\left(\swofu\right)=\widetilde{O}\left(d^{2/3}B_T^{1/3}T^{2/3}\right).$ When $B_T$ is unknown, by taking $w=\Theta\left((dT)^{2/3}\right),$ the dynamic regret of the \swofu~is $
	\R_T\left(\swofu\right)=\widetilde{O}\left(d^{2/3}B_TT^{2/3}\right).$
\end{theorem}
\begin{proof}{Proof Sketch.}
	The proof of this theorem is similar to that of Theorem \ref{theorem:sw_main}, and is thus omitted. The only difference is that we need to include the regret contributed by selecting actions $a_1,\ldots,a_d$ every $w$ rounds. But these sums to $\widetilde{O}\left(dT/w\right),$ which is dominated by the term $\widetilde{O}\left(dT/\sqrt{w}\right).$ Hence the dynamic regret bounds can be obtained similarly as the linear bandit setting.\halmos
\end{proof}
We can now implement the \bob~with the same set of parameters as eq. (\ref{eq:bob_parameters}), except that $Q$ is set to $H\cdot Y_{\max},$ \ie,
\begin{align}
\label{eq:glm_bob_parameters}
&H=\left\lfloor \left(dT\right)^{\frac{1}{2}}\right\rfloor,\Delta=\lceil\ln H\rceil,J=\left\{H^0,\left\lfloor H^{\frac{1}{\Delta}}\right\rfloor,\ldots, H\right\},Q=2H\cdot Y_{\max}.
\end{align}
 This is because the total rewards of each block is deterministically bounded by $[-H\cdot Y_{\max},H\cdot Y_{\max}].$ The dynamic regret bound when $B_T$ is unknown thus follows.
\begin{theorem}
	\label{theorem:glm_bob}
	The dynamic regret bound of the \bob~for the drifting generalized linear bandit setting is $\R_T\left(\bob\right)=\widetilde{O}\left(d^{2/3}B_T^{1/3}T^{2/3}+d^{1/2}T^{3/4}\right).$
\end{theorem}
The proof of the theorem is similar to Theorem \ref{theorem:bob}'s, and it is thus omitted.
\subsection{Combinatorial Semi-Bandits}
Finally, we consider the drifting combinatorial semi-bandit problem. For ease of presentation, we use $X(i)$ to denote the $i^{\text{th}}$ coordinate of a vector $X.$ Following the setup in Kveton et al. \citep{KWAS15}, an instance of combinatorial semi-bandit is represented by the tuple $(E, \mathcal{E}, \{P_t\}^T_{t=1}),$ where the ground set $E$ consist of $d$ items, and ${\cal E}$ 
is a family of indicator vectors of subsets of $E$. 
Each $P_t$ is a latent distribution on the reward vector $W_t = (W_t(1),\ldots W_t(d))$ on each and every item $i\in E$ in round $t\in[T].$  The DM only knows that $W_t(i)$ belongs to $[0,1]$ for each $i\in[d]$ and $t\in[T]$, but she does not know $\theta_t(i) = \mathbb{E}[W_t(i)]$ for any $i\in[d]$ and $t\in[T].$ We can thus know from Lemma 1.8 of Rigollet and H\"utter \citep{RH18} that $W_t(i)-\theta_t(i)$ is $R=1/2$ sub-Gaussian for all $t\in[T]$ and $i\in[d]$. The sequence $\{P_t\}^T_{t=1}$ are generated by an oblivious adversary before the online process begins.

In each round $t,$ a reward vector $W_t$ is sampled according to the latent distribution $P_t$. Then, the DM pulls an action $X_t\in {\cal E}_t$, and earns a reward $Y_t=\langle X_t,W_t\rangle=\sum_{i\in E} X_t(i) W_t(i)$ that corresponds to the items indicated by $X_t$. Under the semi-bandit feedback model, the DM observes the realized rewards $\{W_t(i) : X_t(i) = 1\}$ for the indicated items, but she does not observe $W_t(i)$ for $X_t(i) = 0$. The DM desires to minimize the dynamic regret $\mathbb{E}\left[\sum^T_{t=1}\max_{x_t^*\in \mathcal{E}}\langle x_t^*-X_t,\theta_t\rangle  \right].$
Similar to the $d$-armed bandit setting, we define the variation budget $B_T$ with the infinity norm:
$\sum^{T-1}_{t = 1} \|\theta_{t+1} - \theta_t\|_\infty \leq B_T.$ For the subsequent discussion, we denote $m = \max_{X\in {\cal E}} \sum_{i\in E}X(i)$ as the maximum arm size of the underlying instance.

We first show a lower bound for this setting. 
\begin{theorem}
	\label{theorem:semi_lower_bound}
	Let $(d, m, T, B_T)$ be a tuple that satisfies inequalities $d\geq 2m \geq 2$, $T\geq 1$, $m/d \leq B_T\leq T m /d$. For any non-anticipatory policy, there exists a drifting combinatorial bandit instance $(E, \mathcal{E}, \{P_t\}^T_{t=1}),$ with $d$ items, maximum arm size $m$, and variation budget $B_T$ such that the dynamic regret in $T$ rounds is $\Omega(d^{1/3}m^{2/3}B_T^{1/3}T^{2/3}).$
\end{theorem}
The complete proof is presented in Section \ref{sec:theorem:semi_lower_bound} of the appendix. For a window size $w,$ we define $N_{t-1}(i)$ as the number of times that coordinate $i$ of the chosen action is set to $1$ within rounds $(t-w),\ldots,(t-1),$ \ie, for all $i\in[d],$ 
$N_{t-1}(i)=\sum^{t-1}_{s = 1\vee (t-w)} \mathbf{1}[X_s(i) = 1].$ Here $\bm{1}[\cdot]$ is the indicator function. In each round $t$, the DM also maintains the sliding-window estimates for each coordinate $i\in[d]$ of $\theta_t$:
$$\hat{\theta}_{t}(i) = \frac{\sum^{t-1}_{s = 1\vee (t-w)} W_s(i) \cdot \mathbf{1}[X_s(i) = 1]}{\max\{N_{i, t-1},1\}}. $$
Thanks to the semi-bandit feedback, the outcome $W_s(i)$ is observed when $X_s(i) = 1$, so $\hat{\theta}_{t, i}$ can be constructed from the observations in the previous $w$ rounds. We can thus reuse the Theorem \ref{theorem:mab_deviation} derived for the $d$-armed bandit case:
\begin{theorem}
	\label{theorem:semi_deviation}
	For all $t\in[T]$ and all $i\in[d],$ we have with probability at least $1-{1}/{T},$ 
	$\left|\hat{\theta}_t(i) - \theta_t(i)\right|\leq\sum^{t-1}_{s = 1\vee (t-w)}\left\|\theta_s-\theta_{s+1}\right\|_{\infty}+4R \sqrt{\frac{\ln(2 d T^2)}{N_{t-1}(i) + 1}},$
	holds for all $x\in D_t.$
\end{theorem}
The complete proof is presented in Section \ref{sec:theorem:semi_deviation}. Following the rationale of UCB algorithm for stochastic combinatorial semi-bandit \citep{KWAS15} as well as that of Section \ref{sec:swofu}, we consider the \swofu~which selects a combinatorial action $X_t$ with highest UCB in each round $t$, \ie, 
$$
\max_{X\in {\cal E}_t} \left\{\sum_{i\in E} X(i) \cdot \left[\hat{\theta}_{t, i}  + 4R\sqrt{\frac{\ln (2dT^2)}{N_{t-1}(i)+1}}\right]\right\}.
$$
Denoting $m:=\max_{t\in[T],X\in \mathcal{E}_t}\|X\|_1$, we can now arrive at the following regret upper bound.
\begin{theorem}
	\label{theorem:semi_sw_main}
	For any window size $w\geq d/m,$ the dynamic regret of the \swofu~for the drifting combinatorial semi-bandit setting is upper bounded as
	$\R_T\left(\swofu\right)=\widetilde{O}\left(wmB_T+{\sqrt{dm}T}/{\sqrt{w}}\right).$
	When $B_T<mT/d,$ is known, by taking $w=\Theta\left(d^{1/3}m^{-1/3}T^{2/3}B_T^{-2/3}\right),$ the dynamic regret of the \swofu~is $
	\R_T\left(\swofu\right)=\widetilde{O}\left(d^{1/3}m^{2/3}B_T^{1/3}T^{2/3}\right).$
	When $B_T$ is unknown, by taking $w=\Theta\left(d^{1/3}m^{-1/3}T^{2/3}\right),$ the dynamic regret of the \swofu~is $
	\R_T\left(\swofu\right)=\widetilde{O}\left(d^{1/3}m^{2/3}B_TT^{2/3}\right).$
\end{theorem}
The complete proof is presented in Section \ref{sec:theorem:semi_sw_main} of the appendix. When $B_T$ is unknown, we can implement the \bob~with the following parameters:
\begin{align}
	\label{eq:semi_bob_parameters}
	&H=\left\lfloor \left(dT\right)^{\frac{1}{2}}m^{-\frac{1}{2}}\right\rfloor,\Delta=\lceil\ln H\rceil,J=\left\{H^0,\left\lfloor H^{\frac{1}{\Delta}}\right\rfloor,\ldots, H\right\},Q=2H\cdot m
\end{align}
This is because the total rewards of each block is deterministically bounded by $[-H\cdot m,H\cdot m].$
The dynamic regret bound of the \bob~for the combinatorial semi-bandit setting is characterized as follows.
\begin{theorem}
	\label{theorem:semi_bob}
	The dynamic regret of the \bob~for the drifting combinatorial semi-bandit setting is  $\R_T\left(\bob\right)=\widetilde{O}\left(d^{1/3}m^{2/3}B_T^{1/3}T^{2/3}+d^{1/4}m^{3/4}T^{3/4}\right).$
\end{theorem}
The complete proof is presented in Section \ref{sec:theorem:semi_bob}.

\section{Numerical Experiments}\label{sec:numerical}
As a complement to our theoretical results, we conduct numerical experiments on synthetic datasets and the CPRM-12-001: On-Line Auto Lending dataset provided by the Center for Pricing and Revenue Management at Columbia University to compare the dynamic regret performances of the \swofu~and the \bob~with several existing non-stationary bandit algorithms.

\subsection{Experiments on Synthetic Dataset}
For synthetic dataset, in Section \ref{sec:numerical_BGZ_eg}, we first evaluate the growth of dynamic regret when $T$ increases. We follow the setup of \citep{BGZ18} for fair comparisons. Then, in Section \ref{sec:numerical_piecewise_lin}, we fix $T=10^5$, and evaluate the behavior of the algorithms across rounds.

\subsubsection{The Trend of Dynamic Regret with Varying $T$}\label{sec:numerical_BGZ_eg}
We consider a 2-armed bandit setting, and we vary $T$ from $3\times10^4$ to $2.4\times 10^5$ with a step size of $3\times10^4.$ We set $\theta_t$ to be the following sinusoidal process, \ie, $\forall t\in[T],$ $\theta_t=\begin{pmatrix}
0.5+0.3\sin\left({5B_T\pi t}/{T}\right),
0.5+0.3\sin\left(\pi+{5B_T\pi t}/{T}\right)
\end{pmatrix}^{\top}.$
The total variation of the $\theta_t$'s across the whole time horizon is upper bounded by $\sqrt{2}B_T.$ We also use i.i.d. normal distribution with $R=0.1$ for the noise terms. 
\paragraph{Known Constant Variation Budget.} We start from the known constant variation budget case, \ie, $B_T=1,$ to measure the regret growth of the two optimal algorithms, \ie,  the optimally tuned (\ie, knowing $B_T$) \swofu~and the modified EXP3.S algorithm \citep{BGZ15}, with respect to the total number of rounds. The log-log plot is shown in Fig. \ref{fig:const_var}. From the plot, we can see that the regret of \swofu~is only about $20\%$ of the regret of EXP3.S algorithm.
\paragraph{Unknown Time-Dependent Variation Budget.} We then turn to the more realistic time-dependent variation budget case, \ie, $B_T=T^{1/3}.$ As the modified EXP3.S algorithm does not apply to this setting, we compare the performances of the obliviously tuned (\ie, not knowing $B_T$) \swofu~and the \bob. The log-log plot is shown in Fig. \ref{fig:inc_var}. From the results, we verify that the slope of the regret growth of both algorithms roughly match the established results, and the regret of \bob's~is much smaller than that of the \swofu's.
\begin{figure}[!ht]
	\subfigure[Log-log plot for known $B_T=O(1).$]{	\label{fig:const_var}\includegraphics[width=8cm,height=5.8cm]{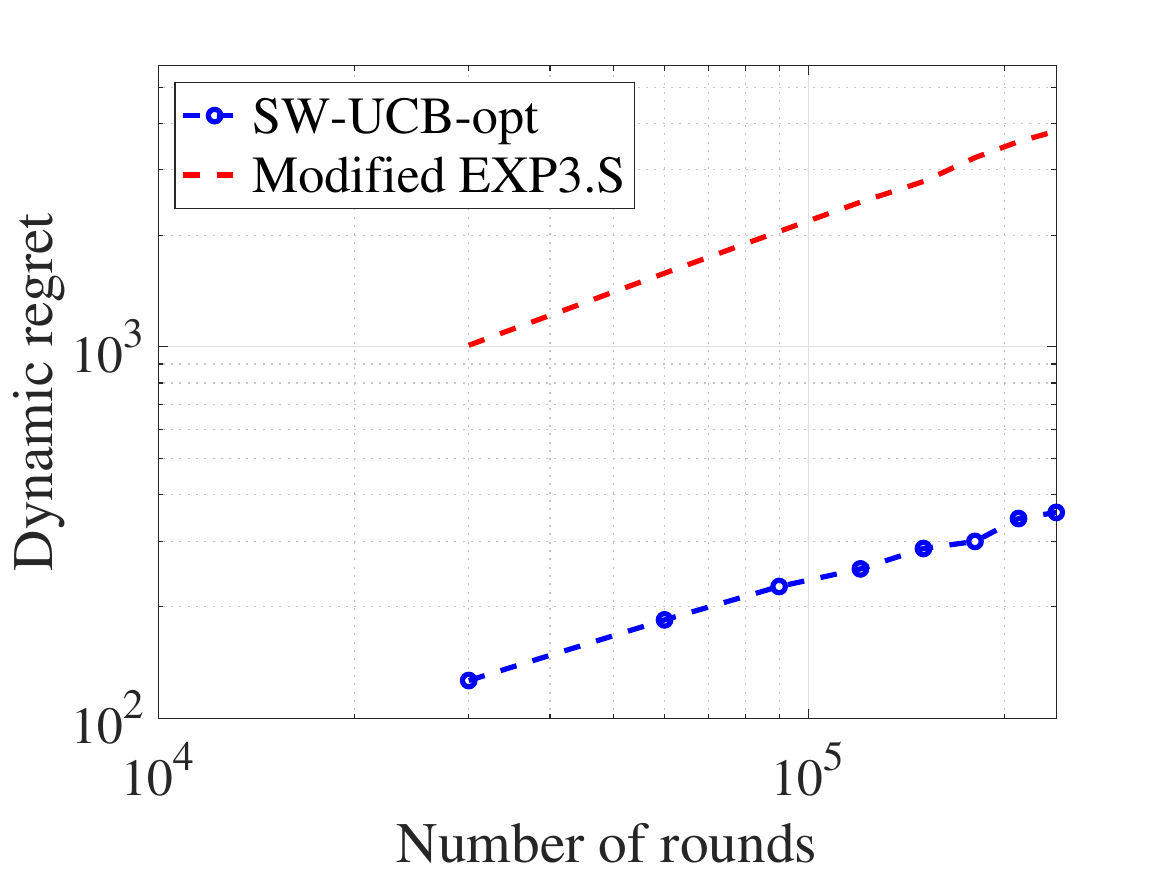}}
	\subfigure[Log-log plot for unknown $B_T=O(T^{1/3}).$]{	\label{fig:inc_var}\includegraphics[width=8cm,height=5.8cm]{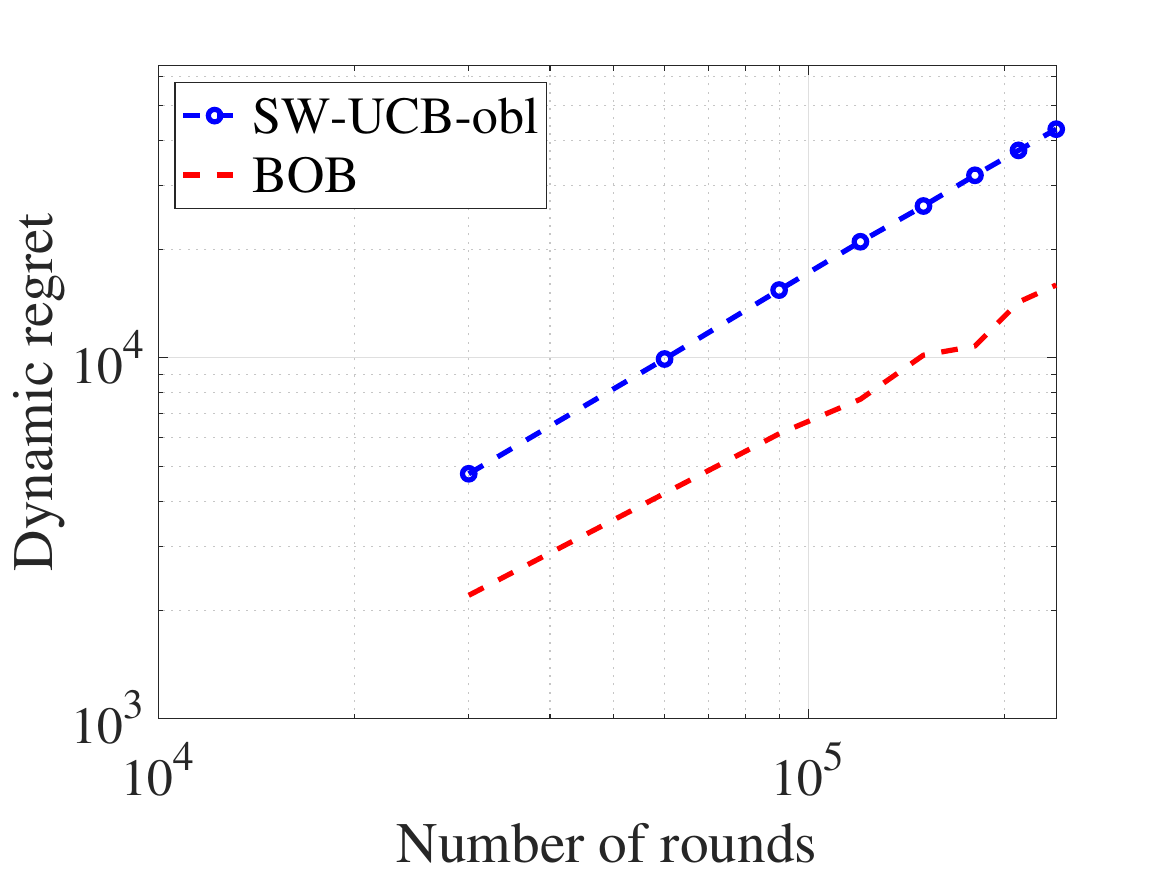}}
	\caption{Results for gradually change environment with 2 arms}
\end{figure}

\subsubsection{A Further Study on the Algorithms' Behavior}\label{sec:numerical_piecewise_lin}
We provide additional numerical evaluation, by considering \emph{piecewise linear instances}, where the reward vector $\theta_t\in \mathbb{R}^d$ is a randomly generated piecewise linear function of $t$. To generate such an instance, we first set $T = 10^5$, and then we randomly sample 30 time points in $\tau_1, \tau_2, \ldots, \tau_{30}\in \{2, \ldots, T-1\}$ without replacement. We further denote $\tau_ 0 =1, \tau_{31} = T$. After that, we randomly sample 32 random unit length vectors 
$v_0, \ldots, v_{31} \in \Re^d$. Finally, for each $t \in [T]$, we define $\theta_t$ as the linear interpolation between $v_s, v_{s+1}$, where $\tau_s\leq t\tau_{s+1}$. More precisely, we have $\theta_t = ((\tau_{s+1} - t) v_{s} + (t - \tau_s) v_{s+1}) / (\tau_{s+1} - \tau_s)$. Note that the random reward in each period can be negative. 

In what follows, we first evaluate the performance of the algorithms by \citep{BGZ18} as well as our algorithms in a 2-armed bandit piece-wise linear instance. Then, we evaluate the performance of our algorithms in a linear bandit piece-wise linear instance, where $d=5$, and each $D_t$ is a random subset of 40 unit length vectors in $\Re^d$. We do not evaluate the algorithms by \citep{BGZ18} in the second instance, since the algorithms by \citep{BGZ18} are only designed for the non-stationary $K$-armed bandit setting. For each instance, each algorithm is evaluated 50 times.
\paragraph{Two armed bandits.} We first evaluate the performance of the modified EXP.3S in \citep{BGZ18} as well as the performance of the \swofu, \bob in a randomly generated 2-armed bandit instance. Fig \ref{fig:2-arm_reward} illustrates the average cumulative reward earned by each algorithm in the 50 trials, and Fig \ref{fig:2-arm_regret} depicts the average dynamic regret incurred by each algorithm in the 50 trials. In Figs \ref{fig:2-arm_reward}, \ref{fig:2-arm_regret}, shorthand SW-UCB-opt is the \swofu, where $B_T$ is known and $w = w^\text{opt}$ is set to further optimized the log factors of the dynamic regret bound (see Appendix \ref{sec:numerical_details} for the expression of $w^\text{opt}$). Shorthand $\text{EXP3.S}$ stands for the modified EXP3.S algorithm by \citep{BGZ18}, where $B_T$ is known and the window size is set to optimized the dynamic regret bound. Shorthand $\text{BOB}$ stands for the \bob. Shorthand SW-UCB-obl is the \swofu, where $B_T$ is not known, and $w = w^\text{obl}$ is obliviously set (see Appendix \ref{sec:numerical_details} for the expression of $w^\text{obl}$). Finally, shorthand $\text{UCB}$ stands for the UCB algorithm by \citep{AYPS11}, which is applicable to the stationary $K$-armed bandit problem. Note that $B_T$ is known to SW-UCB-opt, $\text{EXP3.S}$, but not to $\text{BOB}$, SW-UCB-obl, $\text{UCB}$.

Overall, we observe that SW-UCB-opt is the better performing algorithm when $B_T$ is known, and BOB is the best performing when $B_T$ is not known. It is evident from Fig \ref{fig:2-arm_reward} that SW-UCB-opt, $\text{EXP3.S}$ and $\text{BOB}$ are able to adapt to the change in the reward vector $\theta_t$ across time $t$. We remark that $\text{BOB}$, which does not know $B_T$, achieves a comparable amount of cumulative reward to $\text{EXP3.S}$, which does know $B_T$, across time. It is also interesting to note that UCB, which is designed for the stationary setting, fails to converge (or even to achieve a non-negative total reward) in the long run, signifying the need of an adaptive UCB algorithm in a non-stationary setting.
\begin{figure}[!ht]
	\subfigure[Cumulative reward]{	\label{fig:2-arm_reward}\includegraphics[width=8cm,height=5.33cm]{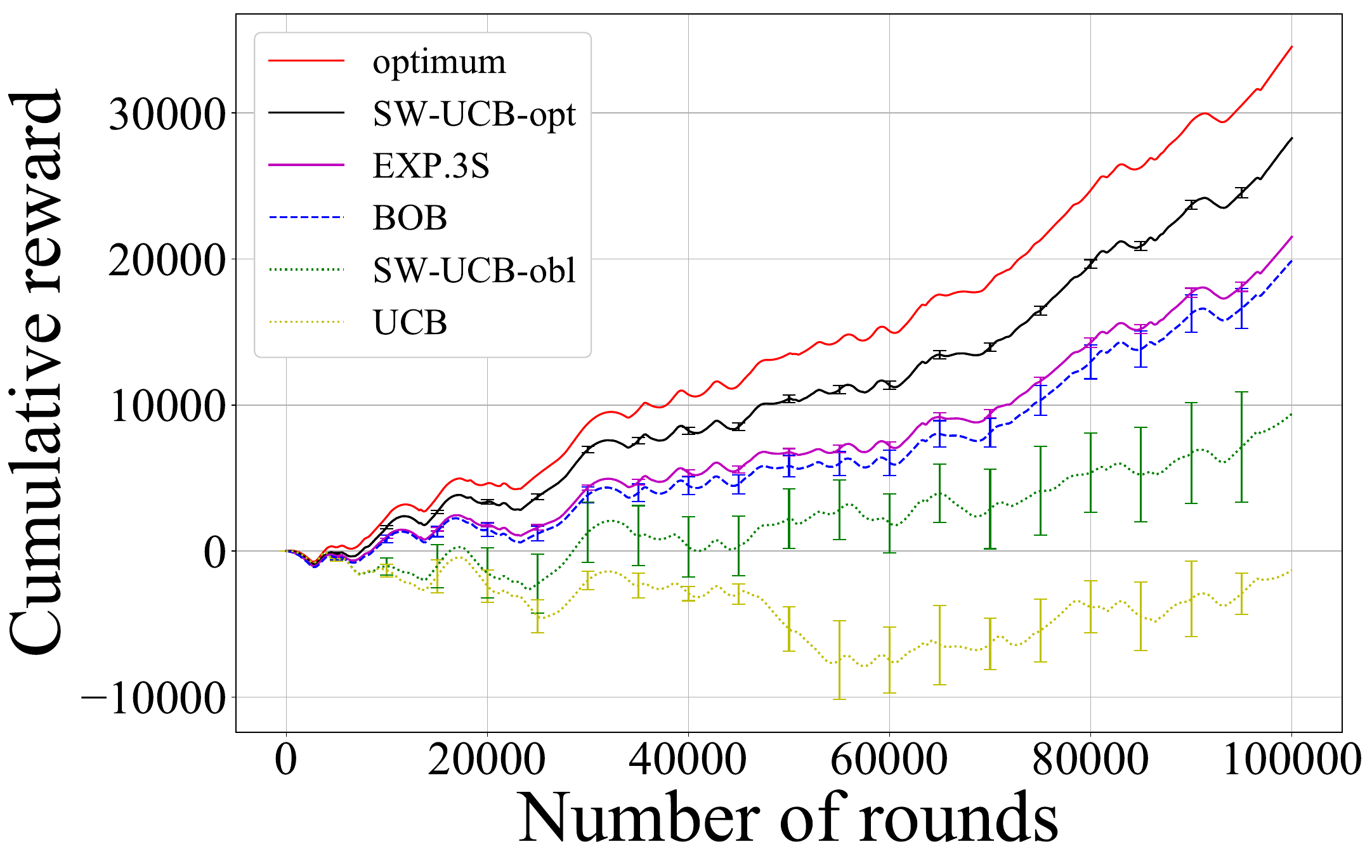}}
	\subfigure[Dynamic regret]{	\label{fig:2-arm_regret}\includegraphics[width=8cm,height=5.33cm]{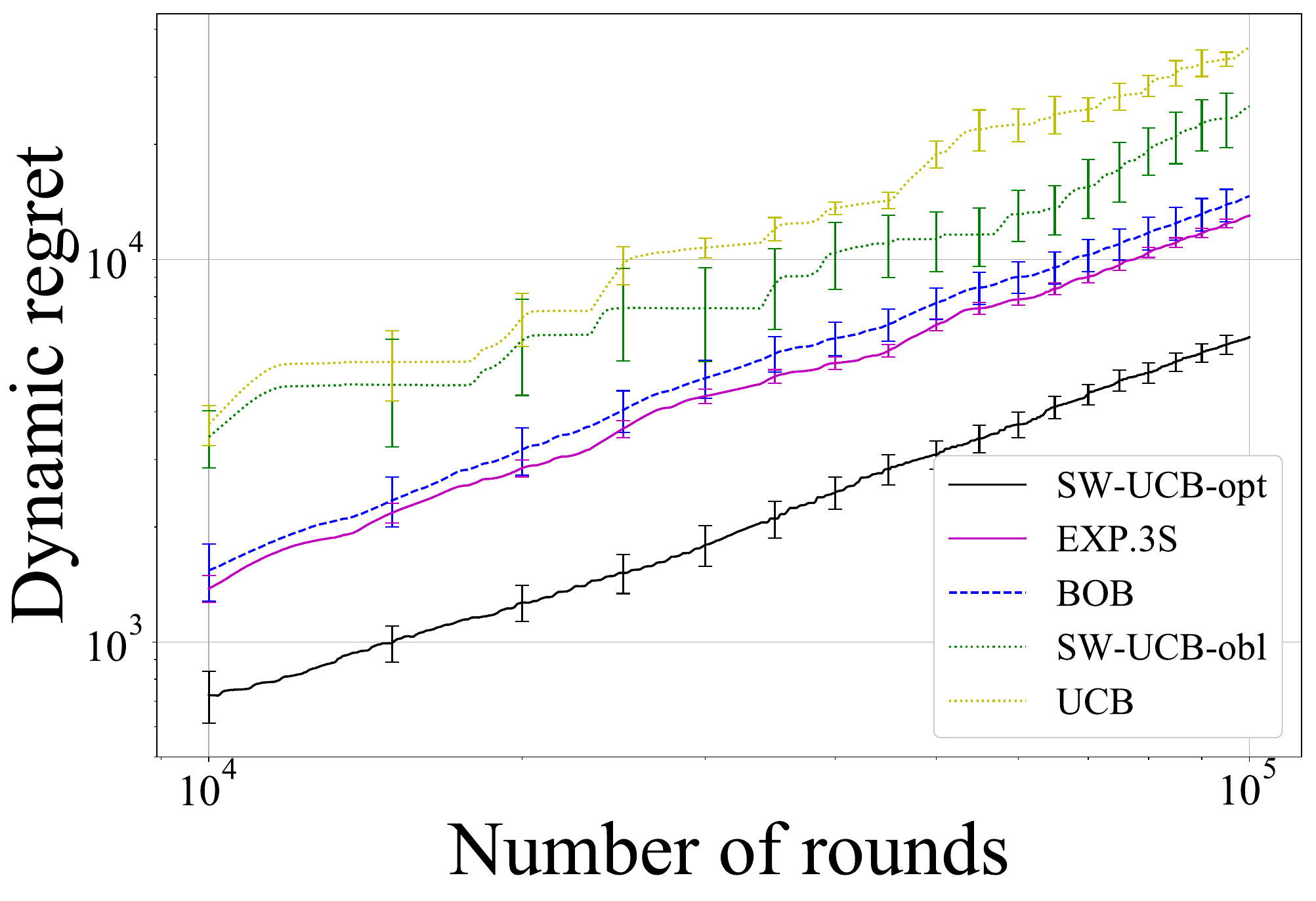}}
	\caption{Results for piecewise linear environment with 2 arms}
\end{figure}
\paragraph{Linear bandits. }Next, we move to the linear bandit case, and we consider the performance of  SW-UCB-opt, SW-UCB-obl, $\text{BOB}$ and $\text{UCB}$, as illustrated in Figs \ref{fig:lin-arm_reward}, \ref{fig:lin-arm_regret}. While the performance of the algorithms ranks similarly to the previous 2-armed bandit case, we witness that UCB, which is designed for the stationary setting, has a much better performance in the current case than the 2-armed case. We surmise that the relatively larger size of the action space $D_t$ here allows UCB to choose an action that performs well even when the reward vector is changing.
\begin{figure}[!ht]
	\subfigure[Cumulative reward]{	\label{fig:lin-arm_reward}\includegraphics[width=8cm,height=5.33cm]{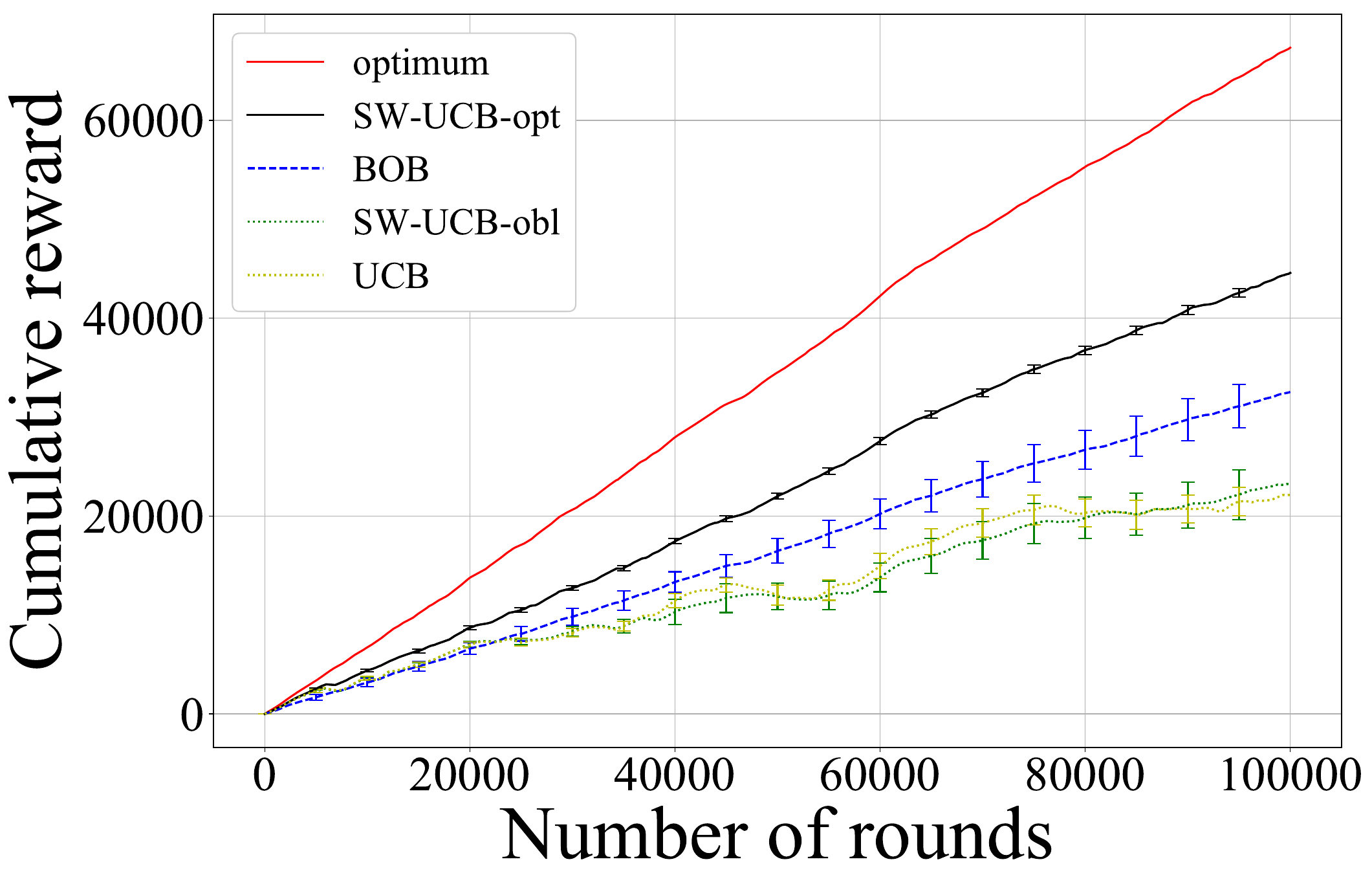}}
	\subfigure[Dynamic Regret]{	\label{fig:lin-arm_regret}\includegraphics[width=8cm,height=5.33cm]{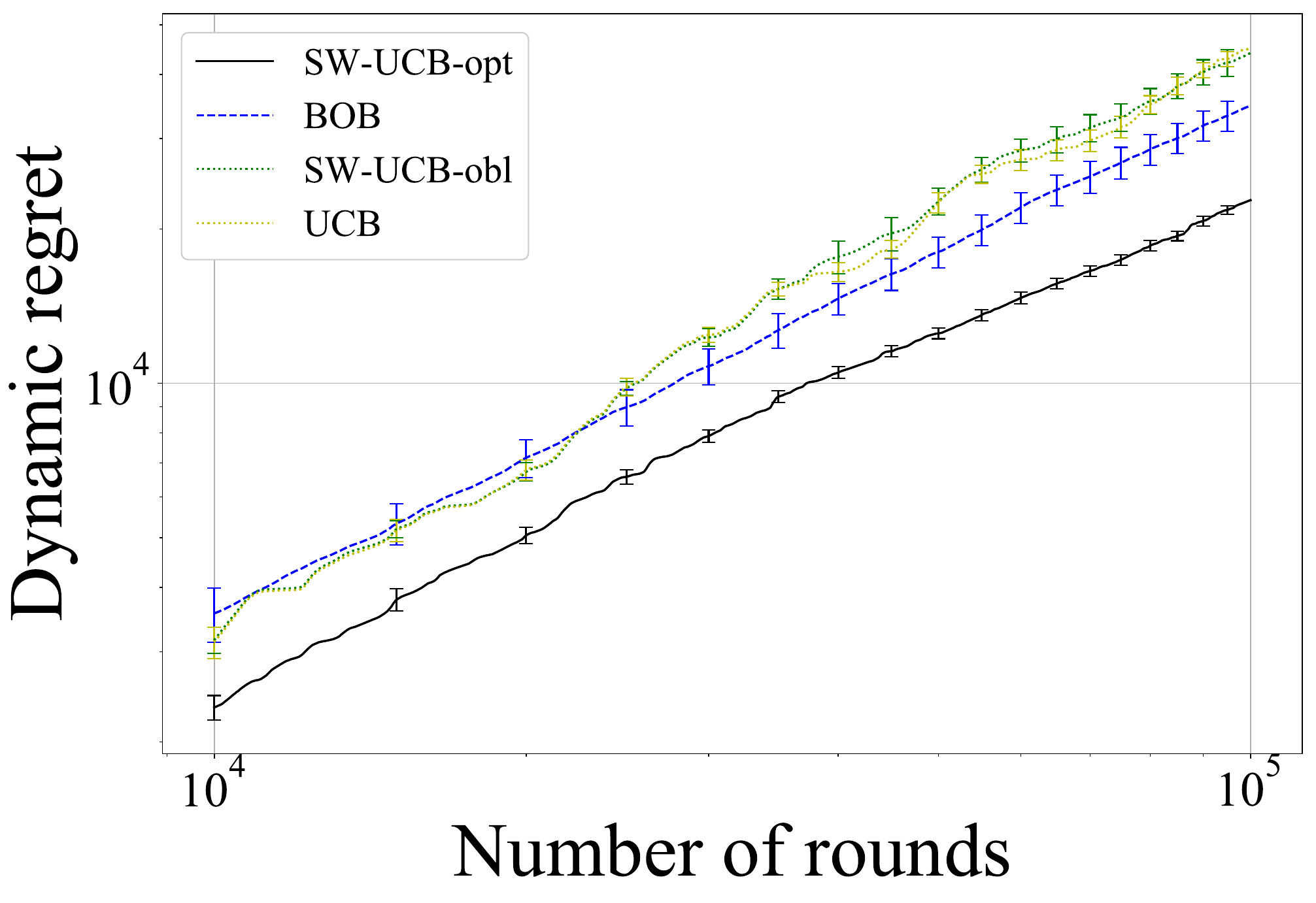}}
	\caption{Results for piecewise linear environment with linear action set.}
\end{figure}
\subsection{Experiments on Online Auto-Lending Dataset}\label{sec:auto}
We now conduct experiments on the on-line auto lending dataset, which was first studied by \citep{PSR15}, and subsequently used to evaluate dynamic pricing algorithms by \citep{BK18}. The dataset records all auto loan applications received by a major online lender in the United States from July 2002 through November 2004. Note that this was the time amid the severe acute respiratory syndrome (SARS) epidemic period \citep{WHO03}, and one could thus expect high volatility in demand similar to the COVID-19 pandemic period. Each datum consists of the borrower's feature (\eg, date of an application, the term and amount of loan requested, and some personal
information), the lender's decision (\eg, the monthly payment for the borrower), and whether or not this offer is accepted by the borrower. Please refer to Columbia University Center for Pricing and Revenue Management \citep{CRPM} for a detailed description of the dataset.

Similar to \cite{BK18}, we use the first $T=5\times10^4$ arrivals that span 276 days for this experiment. We adopt the commonly used \citep{LCLS10,BZ15} linear regression model to interpolate the response of each customer: for the $t^{\text{th}}$ customer with feature $x_t,$ if price $p_t$ is offered, she accepts the offer with ``probability" $\langle\theta_t,[x_t;p_tx_t]\rangle.$ Although the customers' responses are binary, \ie, whether or not she accepts the loan, \citep{BZ15} theoretically justified that the revenue loss caused by using this misspecified model is negligible. For the changing environment, we consider a piecewise stationary environment. In particular, we assume that the $\theta_t$'s remain stationary in a single day period, but can change across days. We also use the feature selection results in \citep{BK18} to pick the FICO score, the term of contract, the loan amount approved, prime rate, the type of car, and the competitor's rate as the feature vector for each customer.

Firstly, we recover the latent parameters $\theta_t$'s from the dataset with linear regression method. Since the lender's decisions, \ie, the price for each customer, is not presented in the dataset, we impute the price of a loan as the net present value of future payments (a function of the monthly payment, customer rate, and term approved, please refer to \citep{CRPM,BK18} for more details).  The resulted $B_T$ is $1.9\times 10^2~\left(\approx T^{0.48}\right),$ which means we are in the moderately non-stationary environment. Since the maximum of the imputed prices is $\approx400,$ the range of price in our experiment is thus set to $[0,500]$ with a step size of 10.

\begin{figure}[!ht]
	\centering
	\includegraphics[width=9cm,height=6.5cm]{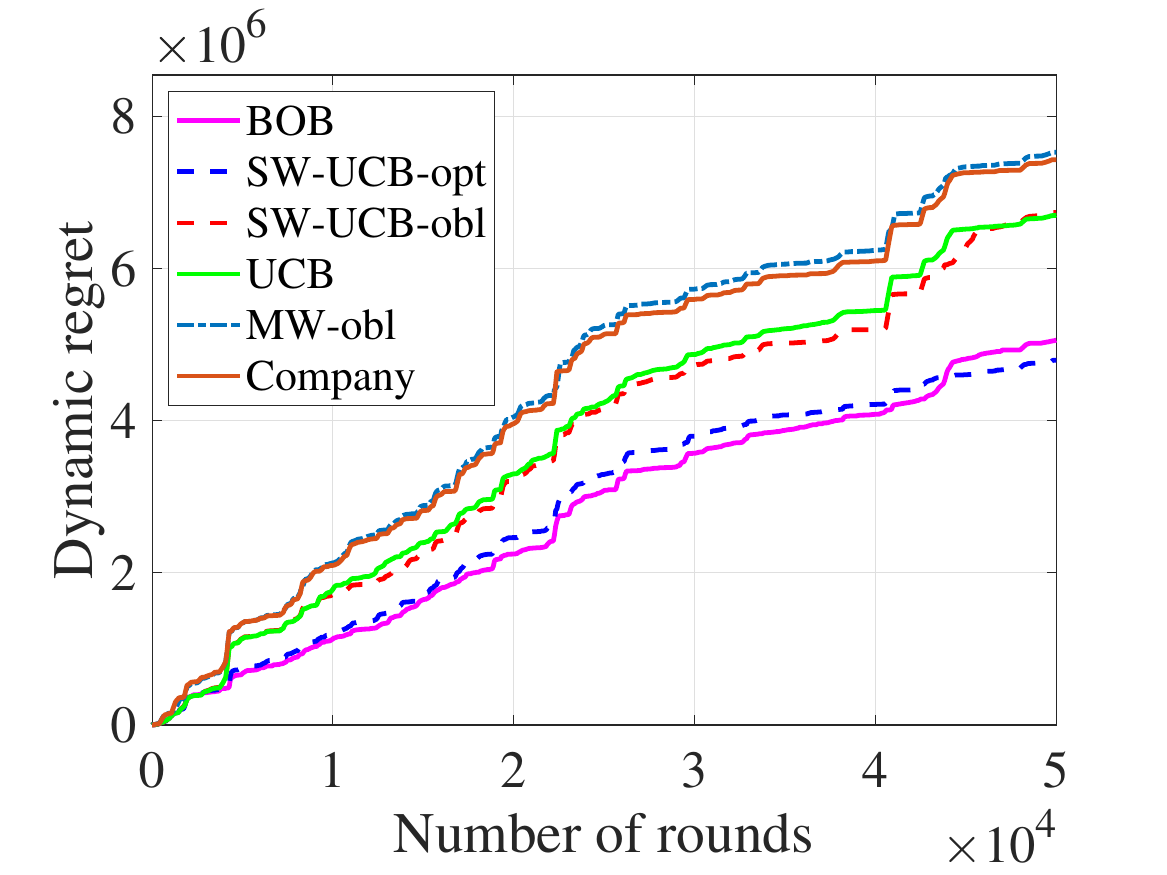}
	\caption{Results for the on-line auto lending dataset.}
	\label{fig:auto}
\end{figure}
We then run the experiment with the recovered parameters, and measure the dynamic regrets of the \swofu~(known $B_T$ and unknown $B_T$), the \bob, the UCB algorithm, the Moving Window (MW) algorithm \citep{KZ16} without knowing $B_T$, as well as the company's original decisions. Here, we note that the MW algorithm does not permit customer features, and hence its dynamic regret should scale linearly in $T$. The results are shown in Fig. \ref{fig:auto}. The plot shows that the \swofu~with known $B_T$ (SW-UCB-opt) and the \bob~have the lowest dynamic regrets. Besides, the dynamic regret of the parameter-free \bob~is $\geq24\%$ less than those of the obliviously tuned \swofu~(SW-UCB-obl) and the UCB algorithm. It also saves $\geq32\%$ dynamic regret when compared to the MW algorithm and the company's original decisions. The results clearly indicate that the \swofu~and the \bob~can deal with the drift while the UCB algorithm fails to keep track of the dynamic environment. More importantly, the results validate our theoretical findings regarding the parameter-free adaptation of the \bob.
\section{Conclusion}
\label{sec:conclusion}
In this paper, we develop general data-driven decision-making algorithms with state-of-the-art dynamic regret bounds in various non-stationary bandit settings. We characterize a minimax dynamic regret lower bound, and present a tuned Sliding Window Upper-Confidence-Bound algorithm with matching dynamic regret bounds. We further propose the parameter-free Bandit-over-Bandit framework that automatically adapts to the unknown non-stationarity. Finally, we conduct extensive numerical experiments on both synthetic and real-world datasets to validate our theoretical results.
\ACKNOWLEDGMENT{The authors thank the department editor J.George Shanthikumar,
	the anonymous associate editor, and three anonymous referees whose comments improved the
	manuscript. The previous version of the current paper contains an error in the proof of Theorem \ref{theorem:sw_deviation}. Fixing this requires Assumption \ref{ass:reg}, which was first introduced by \cite{FauryRAC21}. The authors would like to express sincere gratitude to Omar Besbes, Xi Chen, Dylan Foster, Yonatan Gur, Yujia Jin, Akshay Krishnamurthy, Haipeng Luo, Sasha Rakhlin, Vincent Tan, Kuang Xu, Assaf Zeevi, as well as various seminar attendees for helpful discussions and comments. The authors also gratefully acknowledge Columbia University Center for Pricing and Revenue Management for providing us the dataset on auto loans. This research is supported by the Ministry of Education, Singapore, under its 2019 Academic Research Fund Tier 3 grant call (Award ref: MOE-2019-T3-1-010). The research is also supported  by the MIT Data Science Lab, a lab focused on the development
	of analytic techniques and tools for improving decision making in environments that involve uncertainty and
	require statistical learning.} 
\bibliographystyle{ormsv080}
\bibliography{paperlist}

\begin{thebibliography}{58}
\expandafter\ifx\csname natexlab\endcsname\relax\def\natexlab#1{#1}\fi
\expandafter\ifx\csname url\endcsname\relax
  \def\url#1{{\tt #1}}\fi
\expandafter\ifx\csname urlprefix\endcsname\relax\def\urlprefix{URL }\fi
\expandafter\ifx\csname urlstyle\endcsname\relax
  \expandafter\ifx\csname doi\endcsname\relax
  \def\doi#1{doi:\discretionary{}{}{}#1}\fi \else
  \expandafter\ifx\csname doi\endcsname\relax
  \def\doi{doi:\discretionary{}{}{}\begingroup \urlstyle{rm}\Url}\fi \fi

\bibitem[{Abbasi-Yadkori et~al.(2011)Abbasi-Yadkori, P\'{a}l, and
  Szepesv\'{a}ri}]{AYPS11}
Abbasi-Yadkori, Yasin, David P\'{a}l, Csaba. Szepesv\'{a}ri. 2011.
\newblock Improved algorithms for linear stochastic bandits.
\newblock {\it NIPS\/}.

\bibitem[{Abeille and Lazaric(2017)}]{AL17}
Abeille, Marc, Alessandro Lazaric. 2017.
\newblock Linear thompson sampling revisited.
\newblock {\it Proceedings of International Conference on Artificial
  Intelligence and Statistics (AISTATS)\/}.

\bibitem[{Agarwal et~al.(2017)Agarwal, Luo, Neyshabur, and Schapire}]{ALNS17}
Agarwal, Alekh, Haipeng Luo, Behnam Neyshabur, Robert~E Schapire. 2017.
\newblock Corralling a band of bandit algorithms.
\newblock {\it Proceedings of Annual Conference on Learning Theory (COLT)\/}.

\bibitem[{Agrawal and Goyal(2013)}]{AG13}
Agrawal, Shipra, Navin Goyal. 2013.
\newblock Thompson sampling for contextual bandits with linear payoffs.
\newblock {\it Proceedings of the 30th International Conference on Machine
  Learning (ICML)\/}.

\bibitem[{Audibert and Bubeck(2009)}]{AB09}
Audibert, J.Y., S.~Bubeck. 2009.
\newblock Minimax policies for adversarial and stochastic bandits.
\newblock {\it Proceedings of Annual Conference on Learning Theory (COLT)\/}.

\bibitem[{Auer et~al.(2002{\natexlab{a}})Auer, Cesa-Bianchi, Freund, and
  Schapire}]{ABFS02}
Auer, P., N.~Cesa-Bianchi, Y.~Freund, R.~Schapire. 2002{\natexlab{a}}.
\newblock The nonstochastic multiarmed bandit problem.
\newblock {\it SIAM Journal on Computing, 2002, Vol. 32, No. 1 : pp. 48--77\/}.

\bibitem[{Auer(2002)}]{A02}
Auer, Peter. 2002.
\newblock Using confidence bounds for exploitation-exploration trade-offs.
\newblock {\it Journal of Machine Learning Research, 3:397--422, 2002.\/}.

\bibitem[{Auer et~al.(2002{\natexlab{b}})Auer, Cesa-Bianchi, and
  Fischer}]{ABF02}
Auer, Peter, Nicolo Cesa-Bianchi, Paul Fischer. 2002{\natexlab{b}}.
\newblock Finite-time analysis of the multiarmed bandit problem.
\newblock {\it Machine learning, 47, 235--256\/} .

\bibitem[{Auer et~al.(\ 2003)Auer, Cesa-Bianchi, Freund, and Schapire}]{ABFS03}
Auer, Peter, Nicolo Cesa-Bianchi, Yoav Freund, Robert Schapire. \ 2003.
\newblock The non-stochastic multi-armed bandit problem.
\newblock {\it SIAM Journal on Computing\/}.

\bibitem[{Auer et~al.(2019)Auer, Gajane, and Ortner}]{AuerGO19}
Auer, Peter, Pratik Gajane, Ronald Ortner. 2019.
\newblock Adaptively tracking the best bandit arm with an unknown number of
  distribution changes.
\newblock {\it Proceedings of the Thirty-Second Conference on Learning Theory
  (COLT)\/}.

\bibitem[{Ban and Keskin(2018)}]{BK18}
Ban, Gah-Yi, N.~Bora Keskin. 2018.
\newblock Personalized dynamic pricing with machine learning.
\newblock {\it Available at SSRN: https://ssrn.com/abstract=2972985 or
  http://dx.doi.org/10.2139/ssrn.2972985\/}.

\bibitem[{Bastani et~al.(2019)Bastani, Simchi-Levi, and Zhu}]{BastaniSLZ19}
Bastani, Hamsa, David Simchi-Levi, Ruihao Zhu. 2019.
\newblock Meta dynamic pricing: Learning across experiments.
\newblock {\it https://arxiv.org/abs/1902.10918\/}.

\bibitem[{Becdach et~al.(2020)Becdach, Brown, Halbardier, Henstorf, and
  Murphy}]{BecdachBHHM20}
Becdach, Camilo, Brandon Brown, Ford Halbardier, Brian Henstorf, Ryan Murphy.
  2020.
\newblock Rapidly forecasting demand and adapting commercial plans in a
  pandemic.
\newblock
  \urlprefix\url{https://www.mckinsey.com/industries/consumer-packaged-goods/our-insights/rapidly-forecasting-demand-and-adapting-commercial-plans-in-a-pandemic#}.

\bibitem[{Besbes et~al.(2014)Besbes, Gur, and Zeevi}]{BGZ14}
Besbes, Omar, Yonatan Gur, Assaf Zeevi. 2014.
\newblock Stochastic multi-armed bandit with non-stationary rewards.
\newblock {\it Proceedings of the 27th Annual Conference on Neural Information
  Processing Systems (NIPS)\/}.

\bibitem[{Besbes et~al.(2015)Besbes, Gur, and Zeevi}]{BGZ15}
Besbes, Omar, Yonatan Gur, Assaf Zeevi. 2015.
\newblock Non-stationary stochastic optimization.
\newblock {\it Operations Research, 2015, 63 (5), 1227--1244\/}.

\bibitem[{Besbes et~al.(2018)Besbes, Gur, and Zeevi}]{BGZ18}
Besbes, Omar, Yonatan Gur, Assaf Zeevi. 2018.
\newblock Optimal exploration-exploitation in a multi-armed-bandit problem with
  non-stationary rewards.
\newblock {\it Forthcomming in Stochastic Systems\/}.

\bibitem[{Besbes and Zeevi(2015)}]{BZ15}
Besbes, Omar, Assaf Zeevi. 2015.
\newblock On the (surprising) sufficiency of linear models for dynamic pricing
  with demand learning.
\newblock {\it Management Science 61(4):723--739\/}.

\bibitem[{Besson and Kaufmann(2019)}]{BessonK19}
Besson, Lilian, Emilie Kaufmann. 2019.
\newblock The generalized likelihood ratio test meets klucb: an improved
  algorithm for piece-wise non-stationary bandits.
\newblock {\it https://arxiv.org/abs/1902.01575\/}.

\bibitem[{Bubeck and Cesa{-}Bianchi(2012)}]{BC12}
Bubeck, S., N.~Cesa{-}Bianchi. 2012.
\newblock {\it Regret Analysis of Stochastic and Nonstochastic Multi-armed
  Bandit Problems\/}.
\newblock Foundations and Trends in Machine Learning, 2012, Vol. 5, No. 1: pp.
  1--122.

\bibitem[{Cao et~al.(2019)Cao, Wen, Kveton, and Xie}]{CaoWKX19}
Cao, Yang, Zheng Wen, Branislav Kveton, Yao Xie. 2019.
\newblock Nearly optimal adaptive procedure with change detection for
  piecewise-stationary bandit.
\newblock {\it Proceedings of the 22nd International Conference on Artificial
  Intelligence and Statistics (AISTATS)\/}.

\bibitem[{Cesa-Bianchi and Lugosi(2006)}]{CBL06}
Cesa-Bianchi, Nicol{\`o}, G\'abor Lugosi. 2006.
\newblock {\it Prediction, Learning, and Games\/}.
\newblock Cambridge University Press.

\bibitem[{Chen et~al.(2020)Chen, Wang, and Wang}]{ChenWW20}
Chen, Ningyuan, Chun Wang, Longlin Wang. 2020.
\newblock Learning and optimization with seasonal patterns.
\newblock {\it arXiv:2001.09390\/}.

\bibitem[{Chen et~al.(2019)Chen, Lee, Luo, and Wei}]{CLLW19}
Chen, Yifang, Chung-Wei Lee, Haipeng Luo, Chen-Yu Wei. 2019.
\newblock A new algorithm for non-stationary contextual bandits: Efficient,
  optimal, and parameter-free.
\newblock {\it Proceedings of Conference on Learning Theory (COLT)\/}.

\bibitem[{Cheung et~al.(2019)Cheung, Simchi-Levi, and Zhu}]{CSLZ19}
Cheung, Wang~Chi, David Simchi-Levi, Ruihao Zhu. 2019.
\newblock Learning to optimize under non-stationarity.
\newblock {\it Proceedings of International Conference on Artificial
  Intelligence and Statistics (AISTATS)\/}.

\bibitem[{Cheung et~al.(2020{\natexlab{a}})Cheung, Simchi-Levi, and
  Zhu}]{CheungSLZ19}
Cheung, Wang~Chi, David Simchi-Levi, Ruihao Zhu. 2020{\natexlab{a}}.
\newblock Non-stationary reinforcement learning: The blessing of (more)
  optimism.
\newblock {\it https://arxiv.org/abs/1906.02922\/}.

\bibitem[{Cheung et~al.(2020{\natexlab{b}})Cheung, Simchi-Levi, and
  Zhu}]{CheungSLZ20ICML}
Cheung, Wang~Chi, David Simchi-Levi, Ruihao Zhu. 2020{\natexlab{b}}.
\newblock Reinforcement learning for non-stationary markov decision processes:
  The blessing of (more) optimism.
\newblock {\it Proceedings of the 37th International Conference on Machine
  Learning (ICML)\/}.

\bibitem[{Chiang et~al.(2012)Chiang, Yang, Lee, Mahdavi, Lu, Jin, and
  Zhu}]{CYLMLJZ12}
Chiang, C., T.~Yang, C.~Lee, M.~Mahdavi, C.~Lu, R.~Jin, S.~Zhu. 2012.
\newblock Online optimization with gradual variations.
\newblock {\it Proceedings of Conference on Learning Theory (COLT)\/}.

\bibitem[{Chu et~al.(2011)Chu, Li, Reyzin, and Schapire}]{CLRS11}
Chu, Wei, Lihong Li, Lev Reyzin, Robert Schapire. 2011.
\newblock Contextual bandits with linear payoff functions.
\newblock {\it Proceedings of the the 14th International Conference on
  Artificial Intelligence and Statistics (AISTATS)\/}.

\bibitem[{Columbia(2015)}]{CRPM}
Columbia. 2015.
\newblock Center for pricing and revenue management datasets.
\newblock
  \urlprefix\url{https://www8.gsb.columbia.edu/cprm/sites/cprm/files/files/CPRM_AutoLoan_Data%20dictionary%283%29.pdf}.

\bibitem[{Cormen et~al.(2009)Cormen, Leiserson, Rivest, and Stein}]{CLRS09}
Cormen, Thomas~H., Charles~E. Leiserson, Ronald~L. Rivest, Clifford Stein.
  2009.
\newblock Introduction to algorithms.
\newblock {\it MIT Press\/}.

\bibitem[{Dani et~al.(2008)Dani, Hayes, and Kakade}]{DHK08}
Dani, Varsha, Thomas Hayes, Sham Kakade. 2008.
\newblock Stochastic linear optimization under bandit feedback.
\newblock {\it Proceedings of the 21st Conference on Learning Theory (COLT)\/}.

\bibitem[{Faury et~al.(2021)Faury, Russac, Abeille, and
  Calauzenes}]{FauryRAC21}
Faury, Louis, Yoan Russac, Marc Abeille, Clement Calauzenes. 2021.
\newblock Regret bounds for generalized linear bandits under parameter drift.
\newblock {\it https://arxiv.org/abs/2103.05750\/}.

\bibitem[{Filippi et~al.(2010)Filippi, Cappe, Garivier, and
  Szepesvari}]{FCAS10}
Filippi, Sarah, Olivier Cappe, Aurelien Garivier, Csaba Szepesvari. 2010.
\newblock Parametric bandits: The generalized linear case.
\newblock {\it Proceedings of Annual Conference on Neural Information
  Processing (NIPS)\/}.

\bibitem[{Foster et~al.(2019)Foster, Krishnamurthy, and Luo}]{FosterKL19}
Foster, Dylan~J., Akshay Krishnamurthy, Haipeng Luo. 2019.
\newblock Model selection for contextual bandits.
\newblock {\it Proceedings of the 33rd Conference on Neural Information
  Processing Systems (NeurIPS)\/}.

\bibitem[{Gai et~al.(2012)Gai, Krishnamachari, and Jain}]{GKJ12}
Gai, Yi, Bhaskar Krishnamachari, Rahul Jain. 2012.
\newblock Combinatorial network optimization with unknown variables:
  Multi-armed bandits with linear rewards and individual observations.
\newblock {\it IEEE/ACM Transactions on Networking\/}.

\bibitem[{Garivier and Moulines(2011)}]{GM11}
Garivier, A., E.~Moulines. 2011.
\newblock On upper-confidence bound policies for switching bandit problems.
\newblock {\it Proceedings of International Conferenc on Algorithmic Learning
  Theory (ALT)\/}.

\bibitem[{Golrezaei et~al.(2020)Golrezaei, Manshadi, Schneider, and
  Sekar}]{GolrezaeiMSS20}
Golrezaei, Negin, Vahideh Manshadi, Jon Schneider, Shreyas Sekar. 2020.
\newblock Learning product rankings robust to fake users.
\newblock {\it ArXiv:2009.05138 [cs.LG]\/}.

\bibitem[{Jadbabaie et~al.(2015)Jadbabaie, Rakhlin, Shahrampour, and
  Sridharan}]{JRSS15}
Jadbabaie, A., A.~Rakhlin, S.~Shahrampour, K.~Sridharan. 2015.
\newblock Online optimization : Competing with dynamic comparators.
\newblock {\it Proceedings of International Conference on Artificial
  Intelligence and Statistics (AISTATS)\/}.

\bibitem[{Karnin and Anava(2016)}]{KA16}
Karnin, Z., O.~Anava. 2016.
\newblock Multi-armed bandits: Competing with optimal sequences.
\newblock {\it Procedding of Annual Conference on Neural Information Processing
  Systems (NIPS)\/}.

\bibitem[{Keskin and Zeevi(2016)}]{KZ16}
Keskin, N., A.~Zeevi. 2016.
\newblock Chasing demand: Learning and earning in a changing environments.
\newblock {\it Mathematics of Operations Research, 2016, 42(2), 277--307\/}.

\bibitem[{Keskin and Zeevi(2014)}]{KZ14}
Keskin, N.~Bora, Assaf Zeevi. 2014.
\newblock Dynamic pricing with an unknown demand model: Asymptotically optimal
  semi-myopic policies.
\newblock {\it Operations Research 62(5):1142--1167\/}.

\bibitem[{Kveton et~al.(2015)Kveton, Wen, Ashkan, and Szepesv\'{a}ri}]{KWAS15}
Kveton, Branislav, Zheng Wen, Azin Ashkan, Csaba Szepesv\'{a}ri. 2015.
\newblock Tight regret bounds for stochastic combinatorial semi-bandits.
\newblock {\it AISTATS\/}.

\bibitem[{Lattimore and Szepesv\'{a}ri(2018)}]{LS18}
Lattimore, T., C.~Szepesv\'{a}ri. 2018.
\newblock {\it Bandit Algorithms\/}.
\newblock Cambridge University Press.

\bibitem[{Li et~al.(2010)Li, Chu, Langford, and Schapire}]{LCLS10}
Li, Lihong, Wei Chu, John Langford, Robert Schapire. 2010.
\newblock A contextual-bandit approach to personalized news article
  recommendation.
\newblock {\it Proceedings of International conference on World wide web
  (WWW)\/}.

\bibitem[{Li et~al.(2017)Li, Lu, and Zhou}]{LLZ17}
Li, Lihong, Yu~Lu, Dengyong Zhou. 2017.
\newblock Provably optimal algorithms for generalized linear contextual
  bandits.
\newblock {\it Proceedings of International Conference on Machine Learning
  (ICML)\/}.

\bibitem[{Liu et~al.(2018)Liu, Lee, and Shroff}]{LiuLS18}
Liu, Fang, Joohyun Lee, Ness Shroff. 2018.
\newblock A change-detection based framework for piecewise-stationary
  multi-armed bandit problem.
\newblock {\it Proceedings of the Thirty-Second AAAI Conference on Artificial
  Intelligence (AAAI)\/}.

\bibitem[{Luo et~al.(2018)Luo, Wei, Agarwal, and Langford}]{LWAL18}
Luo, H., C.~Wei, A.~Agarwal, J.~Langford. 2018.
\newblock Efficient contextual bandits in non-stationary worlds.
\newblock {\it Proceedings of Conference on Learning Theory (COLT)\/}.

\bibitem[{Lykouris et~al.(2018)Lykouris, Mirrokni, and Leme}]{LykourisML18}
Lykouris, Thodoris, Vahab Mirrokni, Renato~Paes Leme. 2018.
\newblock Stochastic bandits robust to adversarial corruptions.
\newblock {\it Proceedings of the 50th Annual ACM SIGACT Symposium on Theory of
  Computing (STOC)\/}.

\bibitem[{Phillips et~al.(2015)Phillips, Simsek, and van Ryzin}]{PSR15}
Phillips, Robert, A.~Serdar Simsek, Garrett van Ryzin. 2015.
\newblock The effectiveness of field price discretion: Empirical evidence from
  auto lending.
\newblock {\it Management Science 61(8):1741--1759\/}.

\bibitem[{Rigollet and H\"utter(2018)}]{RH18}
Rigollet, R., J.~H\"utter. 2018.
\newblock {\it High Dimensional Statistics\/}.
\newblock Lecture Notes.

\bibitem[{Rusmevichientong and Tsitsiklis(2010)}]{RT10}
Rusmevichientong, Paat, John~N. Tsitsiklis. 2010.
\newblock Linearly parameterized bandits.
\newblock {\it Mathematics of Operations Research 35(2):395--411.\/}.

\bibitem[{Russo and Van~Roy(2014)}]{RVR14}
Russo, Daniel, Benjamin Van~Roy. 2014.
\newblock Learning to optimize via posterior sampling.
\newblock {\it Mathematics of Operations Research 39(4):1221--1243.
  https://doi.org/10.1287/moor.2014.0650\/}.

\bibitem[{Wei et~al.(2016)Wei, Hong, and Lu}]{WHL16}
Wei, Chen-Yu, Yi-Te Hong, Chi-Jen Lu. 2016.
\newblock Tracking the best expert in non-stationary stochastic environments.
\newblock {\it Proceedings of Annual Conference on Neural Information
  Processing (NIPS)\/}.

\bibitem[{Wei and Srivastava(2018)}]{WeiS18}
Wei, Lai, Vaibhav Srivastava. 2018.
\newblock On abruptly-changing and slowly-varying multiarmed bandit problems.
\newblock {\it Proceedings of Annual American Control Conference (ACC)\/}.

\bibitem[{{World Health Organization (WHO)}(2003)}]{WHO03}
{World Health Organization (WHO)}. 2003.
\newblock Severe acute respiratory syndrome (sars).
\newblock \urlprefix\url{https://www.who.int/csr/sars/en/}.

\bibitem[{{World Health Organization (WHO)}(2020)}]{WHO20}
{World Health Organization (WHO)}. 2020.
\newblock Coronavirus disease (covid-19) pandemic.
\newblock
  \urlprefix\url{https://www.who.int/emergencies/diseases/novel-coronavirus-2019}.

\bibitem[{Zhao et~al.(2019)Zhao, Wang, Zhang, and Zhou}]{ZhaoWZZ19}
Zhao, Peng, Guanghui Wang, Lijun Zhang, Zhi-Hua Zhou. 2019.
\newblock Bandit convex optimization in non-stationary environments.
\newblock {\it https://arxiv.org/abs/1907.12340\/}.

\bibitem[{Zhou et~al.(2020)Zhou, Chen, Gao, and Xiong}]{ZhouCGX20}
Zhou, Xiang, Ningyuan Chen, Xuefeng Gao, Yi~Xiong. 2020.
\newblock Regime switching bandits.
\newblock {\it arXiv:2001.09390\/}.

\end{thebibliography}
\newpage
\begin{APPENDIX}{Proofs}
	\section{Proof of Theorem \ref{theorem:lower_bound}}
\label{sec:theorem:lower_bound}
First, let's review the lower bound of the linear bandit setting, which is related to ours except that the $\theta_t$'s do not vary across rounds, and are equal to the same (unknown) $\theta,$ \ie, $\forall t\in[T]~\theta_t=\theta.$
\begin{lemma}[\citep{LS18}]
	\label{lemma:lower_bound}
	For any $T_0\geq\sqrt{d}/2$ and let $D=\left\{x\in\Re^d:\|x\|\leq1\right\},$ then there exists a $\theta\in\left\{\pm\sqrt{d/4T_0}\right\}^d,$ such that the worst case regret of any algorithm for linear bandits with unknown parameter $\theta$ is $\Omega(d\sqrt{T_0}).$
\end{lemma}
Going back to the non-stationary environment, suppose nature divides the whole time horizon into $\lceil T/H\rceil$ blocks of equal length $H$ rounds (the last block can possibly have less than $H$ rounds), and each block is a decoupled linear bandit instance so that the knowledge of previous blocks cannot help the decision within the current block. Following Lemma \ref{lemma:lower_bound}, we restrict the sequence of $\theta_t$'s are drawn from the set $\left\{\pm\sqrt{d/4H}\right\}^d.$ Moreover, $\theta_t$'s remain fixed within a block, and can vary across different blocks, \ie,
\begin{align}
\forall i\in\left[\left\lceil \frac{T}{H}\right\rceil\right]\forall t_1,t_2\in[(i-1)H+1,i\cdot H\wedge T]\quad \theta_{t_1}=\theta_{t_2}.
\end{align}
We argue that even if the DM knows this additional information, it still incur a regret $\Omega(d^{2/3}B_T^{1/3}T^{2/3}).$ Note that different blocks are completely decoupled, and information is thus not passed across blocks. Therefore, the regret of each block is $\Omega\left(d\sqrt{H}\right),$ and the total regret is at least
\begin{align}
\left(\left\lceil \frac{T}{H}\right\rceil-1\right)\Omega\left(d\sqrt{H}\right)=\Omega\left(dTH^{-\frac{1}{2}}\right).
\end{align} 
Intuitively, if $H,$ the number of length of each block, is smaller, the worst case regret lower bound becomes larger. But too small a block length can result in a violation of the variation budget. So we work on the total variation of $\theta_t$'s to see how small can $H$ be. 
The total variation of the $\theta_t$'s can be seen as the total variation across consecutive blocks as $\theta_t$ remains unchanged within a single block. Observe that for any pair of $\theta,\theta'\in\left\{\pm\sqrt{d/4H}\right\}^d,$ the $\ell_2$ difference between $\theta$ and $\theta'$ is upper bounded as
\begin{align}
\sqrt{\sum_{i=1}^d\frac{4d}{4H}}=\frac{d}{\sqrt{H}}
\end{align}
and there are at most $\lfloor T/H\rfloor$ changes across the whole time horizon, the total variation is at most 
\begin{align}
B=\frac{T}{H}\cdot\frac{d}{\sqrt{H}}={dTH^{-\frac{3}{2}}}.
\end{align}
By definition, we require that $B\leq B_T,$ and this indicates that
\begin{align}
H\geq(dT)^{\frac{2}{3}}B_T^{-\frac{2}{3}}.
\end{align} 
Taking $H=\left\lceil{(dT)^{\frac{2}{3}}B_T^{-\frac{2}{3}}}\right\rceil,$ the worst case regret is 
\begin{align}
\Omega\left(dT\left((dT)^{\frac{2}{3}}B_T^{-\frac{2}{3}}\right)^{-\frac{1}{2}}\right)=\Omega\left(d^{\frac{2}{3}}B_T^{\frac{1}{3}}T^{\frac{2}{3}}\right).
\end{align}
Note that in order for $H\leq T,$ we require $B_T\geq d T^{-1/2}.$ Also, to make $|\langle x,\theta_t\rangle|\leq1$ for all $t\in[T]$ and $x\in D_t,$ we need $\|\theta_t\|\leq1,$ which means $\sqrt{d^2/4H}\leq1$ or $B_T\leq8d^{-2}T.$
\section{Proof of Theorem \ref{theorem:sw_deviation}}
\label{sec:theorem:sw_deviation}
The difference $\hat{\theta}_t-\theta_t$ has the following expression:
\begin{align}
\nonumber&V_{t-1}^{-1}\left(\sum_{s=1\vee(t-w)}^{t-1}X_sX_s^{\top}\theta_s+\sum_{s=1\vee(t-w)}^{t-1}\eta_sX_s\right)-\theta_t\\
\label{eq:sw16}=&V_{t-1}^{-1}\sum_{s=1\vee(t-w)}^{t-1}X_sX_s^{\top}\left(\theta_s-\theta_t\right)+V_{t-1}^{-1} \left(\sum_{s=1\vee(t-w)}^{t-1}\eta_sX_s-\lambda\theta_t\right),
\end{align}
The first term on the right hand side of eq. (\ref{eq:sw16}) is the estimation inaccuracy due to the non-stationarity; while the second term is the estimation error due to random noise. We now upper bound the two terms separately. We upper bound the first term under the Euclidean norm.
\begin{lemma}
	\label{lemma:sw}
	For any $t\in[T],$ we have $$\left\|V_{t-1}^{-1}\sum_{s=1\vee(t-w)}^{t-1}X_sX_s^{\top}\left(\theta_s-\theta_t\right)\right\|_2\leq\sum^{t-1}_{s = 1\vee (t-w)}\left\|\theta_s-\theta_{s+1}\right\|_2.$$
\end{lemma}
\begin{proof}{Poof.}
	In the proof, we denote $B(1)$ as the unit Euclidean ball, and $\lambda_\text{max}(M)$ as the maximum eigenvalue of a square matrix $M$. 
	In addition, recall the definition that $V_{t-1} = \lambda I + \sum^{t-1}_{s = 1\vee (t-w)} X_s X_s^\top$ We prove the Lemma as follows:
	\begin{align}
	\nonumber&\left\|V_{t-1}^{-1}\sum_{s=1\vee(t-w)}^{t-1}X_sX_s^{\top}\left(\theta_s-\theta_t\right)\right\|_2 \\
	\nonumber= &\left\|V_{t-1}^{-1}\sum_{s=1\vee(t-w)}^{t-1}X_sX_s^{\top}\left[\sum^{t-1}_{p=s}\left(\theta_p-\theta_{p+1}\right)\right]\right\|_2\\
	\label{eq:sw} = & \left\|V_{t-1}^{-1} \sum^{t-1}_{p=1\vee(t-w)}  \sum_{s=1\vee(t-w)}^{p}X_sX_s^{\top}\left(\theta_p-\theta_{p+1}\right)\right\|_2\\
	\label{eq:sw1} \leq &  \sum^{t-1}_{p=1\vee(t-w)} \left\| V_{t-1}^{-1} \left(\sum_{s=1\vee(t-w)}^{p}X_sX_s^{\top}\right)\left(\theta_p-\theta_{p+1}\right)\right\|_2\\
	\label{eq:sw2} \leq &  \sum^{t-1}_{p=1\vee(t-w)}\sqrt{ \lambda_\text{max}\left( \left(\sum_{s=1\vee(t-w)}^{p}X_sX_s^{\top}\right)V_{t-1}^{-2} \left(\sum_{s=1\vee(t-w)}^{p}X_sX_s^{\top}\right)\right)}\left\|\theta_p-\theta_{p+1}\right\|_2\\
	\label{eq:sw3} \leq &  \sum^{t-1}_{p=1\vee(t-w)} \left\|\theta_p-\theta_{p+1}\right\|_2.
	\end{align}	
	Equality (\ref{eq:sw}) is by the observation that both sides of the equation is summing over the terms $X_s X^\top_s (\theta_p - \theta_{p+1})$ with indexes $(s, p)$ ranging over $\{(s, p) : 1\vee (t-w) \leq s\leq p \leq t-1\}$. Inequality (\ref{eq:sw1}) is by the triangle inequality. 

	 Inequality (\ref{eq:sw2}) is by the fact that, for any matrix $M\in \mathbb{R}^{d\times d}$ with $\lambda_\text{max}(M)\geq 0$ and any vector $y\in \mathbb{R}^d$, we have $\left\|M y\right\|_2 \leq \sqrt{\lambda_\text{max}(M^2)}\left\|y\right\|_2$.  Applying the above claim with $M = V_{t-1}^{-1} \left(\sum_{s=1\vee(t-w)}^{p}X_sX_s^{\top}\right)$ and $y = \theta_p - \theta_{p+1}$ demonstrates inequality (\ref{eq:sw2}). 

	Finally, for inequality (\ref{eq:sw3}), we denote the corresponding basis for each $X_s$ as $\psi_{i(s)},$ \ie, $X_s=z_s\psi_{i(s)}=z_s\Psi e_{i(s)},$ where $e_{i}$ is the $i^{\text{th}}$ standard orthonormal basis. Let $A_1=\sum_{s=1\vee(t-w)}^{t-1}e_{i(s)}e_{i(s)}^{\top}+\lambda I$ and $A_2=\sum_{s=1\vee(t-w)}^{p}e_{i(s)}e_{i(s)}^{\top},$ it is evident that $V_{t-1}=\Psi A_1\Psi^{\top}$ and $\sum_{s=1\vee(t-w)}^{p}X_sX_s^{\top}=\Psi A_2\Psi^{\top}.$ Therefore, we have 
	\begin{align}
		\nonumber\lambda_\text{max}&\left( \left(\sum_{s=1\vee(t-w)}^{p}X_sX_s^{\top}\right)V_{t-1}^{-2} \left(\sum_{s=1\vee(t-w)}^{p}X_sX_s^{\top}\right)\right)=\lambda_\text{max}\left( \Psi A_2 \Psi^{\top} (\Psi A_1 \Psi^{\top})^{-2}\Psi A_2\Psi^{\top} \right)\\
		&=\lambda_\text{max}\left( \Psi A_2A^{-2}_1 A_2\Psi^{\top} \right)=\lambda_\text{max}\left( A_2A^{-2}_1 A_2 \right)\leq 1,
\end{align} 
where we have used the fact that both $A_1$ and $A_2$ are diagonal matrix in the last step. Altogether, the Lemma is proved.\halmos
\end{proof}
Applying Theorem 2 of \citep{AYPS11}, we have the following upper bound for the second term in eq. \eqref{eq:sw20}.
\begin{lemma}[\citep{AYPS11}]
	\label{lemma:sw1}
	For any $t\in[T]$ and any $\delta\in[0,1],$ we have $$\left\|\sum_{s=1\vee(t-w)}^{t-1}\eta_sX_s-\lambda\theta_t\right\|_{V_{t-1}^{-1}}\leq R\sqrt{d\ln\left(\frac{1+wL^2/\lambda}{\delta}\right)}+\sqrt{\lambda}S$$ holds with probability at least $1-\delta.$ 
\end{lemma}
Combining the above two lemmas: fixed any $\delta\in[0,1],$ we have that for any $t\in[T]$ and any $x\in D_t,$
\begin{align}
\nonumber\left|x^\top ( \hat{\theta}_t - \theta_t)\right|=&\left| x^\top \left( V_{t-1}^{-1}\sum_{s=1\vee(t-w)}^{t-1}X_sX_s^{\top}\left(\theta_s-\theta_t\right)\right)+ x^\top V_{t-1}^{-1}\left(\sum_{s=1\vee(t-w)}^{t-1}\eta_sX_s-\lambda\theta_t\right) \right|  \nonumber\\
\label{eq:sw13}\leq&\left| x^\top \left( V_{t-1}^{-1}\sum_{s=1\vee(t-w)}^{t-1}X_sX_s^{\top}\left(\theta_s-\theta_t\right)\right)\right|+\left| x^\top V_{t-1}^{-1}\left(\sum_{s=1\vee(t-w)}^{t-1}\eta_sX_s-\lambda\theta_t\right) \right| \\
\leq& \left\|x\right\|_2 \cdot \left\|V_{t-1}^{-1}\sum_{s=1\vee(t-w)}^{t-1}X_sX_s^{\top}\left(\theta_s-\theta_t\right)\right\|_2 + \left\|x\right\|_{V^{-1}_{t-1}}\left\|\sum_{s=1\vee(t-w)}^{t-1}\eta_sX_s-\lambda\theta_t\right\|_{V_{t-1}^{-1}}\label{eq:sw_by_CS}\\
\leq&L \sum^{t-1}_{s = 1\vee (t-w)}\left\|\theta_s-\theta_{s+1}\right\|_2 + \beta \left\|x\right\|_{V^{-1}_{t-1}} \label{eq:sw-ineq},
\end{align}
where inequality (\ref{eq:sw13}) uses triangle inequality, inequality (\ref{eq:sw_by_CS}) follows from Cauchy-Schwarz inequality, and inequality (\ref{eq:sw-ineq}) are consequences of Lemmas \ref{lemma:sw}, \ref{lemma:sw1}.

\section{Proof of Theorem \ref{theorem:sw_main}}
\label{sec:theorem:sw_main}
In the proof, we choose $\lambda$ so that $\beta \geq 1$, for example by choosing $\lambda\geq 1/S^2$. 	
By virtue of UCB, the regret in any round $t\in[T]$ is
\begin{align}
\label{eq:sw15} \langle x^*_t - X_t, \theta_t\rangle & \leq  L \sum^{t-1}_{s = 1\vee (t-w)}\left\|\theta_s-\theta_{s+1}\right\|_2 +  \langle X_t, \hat{\theta}_t\rangle + \beta\left\|X_t\right\|_{V^{-1}_{t-1}} - \langle X_t, \theta_t\rangle\\
\label{eq:sw12} & \leq 2 L \sum^{t-1}_{s = 1\vee (t-w)}\left\|\theta_s-\theta_{s+1}\right\|_2 + 2\beta \left\|X_t\right\|_{V^{-1}_{t-1}}.
\end{align}	
Inequality (\ref{eq:sw15}) is by an application of our \swofu~established in equation (\ref{eq:sw-ucb}). Inequality (\ref{eq:sw12}) is by an application of inequality (\ref{eq:sw-ineq}), which bounds the difference $| \langle X_t, \hat{\theta}_t - \theta_t\rangle |$ from above. By the assumption $|\langle X,\theta_t\rangle|\leq1$ in Section \ref{sec:formulation}, it is evident that $\langle X_t, \hat{\theta}_t - \theta_t\rangle\leq |\langle X_t, \hat{\theta}_t\rangle|+|\langle X_t, - \theta_t\rangle| \leq 2$, and we have
\begin{equation}\label{eq:sw17}
\langle x^*_t - X_t, \theta_t\rangle\leq 2 L \sum^{t-1}_{s = 1\vee (t-w)}\left\|\theta_s-\theta_{s+1}\right\|_2 + 2\beta \left(\left\|X_t\right\|_{V^{-1}_{t-1}}\wedge 1\right).
\end{equation}
Summing equation (\ref{eq:sw17}) over $1\leq t\leq T$, the regret of the \swofu~is upper bounded as
\begin{align}
\nonumber\Ex\left[\nonumber\text{Regret}_T\left(\swofu\right)\right]=&\Ex\left[\sum_{t\in[T]}\langle x^*_t- X_t,\theta_t\rangle\right]\\
\nonumber\leq& 2 L \left[ \sum^T_{t = 1}\sum^{t-1}_{s = 1\vee (t-w)}\left\|\theta_s-\theta_{s+1}\right\|_2\right] + 2\beta\cdot \Ex\left[\sum^T_{t=1}\left(\left\|X_t\right\|_{V^{-1}_{t-1}}\wedge 1\right)\right]\\
\nonumber = & 2 L \left[ \sum^T_{s = 1}\sum^{(s+w)\wedge T}_{t = s+1}\left\|\theta_s-\theta_{s+1}\right\|_2\right] + 2\beta\cdot \Ex\left[\sum^T_{t=1}\left(\left\|X_t\right\|_{V^{-1}_{t-1}}\wedge 1\right)\right]\\
\label{eq:sw11} \leq & 2 L w B_T + 2\beta\cdot\Ex\left[\sum^T_{t=1}\left(\left\|X_t\right\|_{V^{-1}_{t-1}}\wedge 1\right)\right].
\end{align}

What's left is to upper bound the quantity $2\beta\cdot\Ex\left[\sum_{t\in[T]}\left(1\wedge\left\|X_t\right\|_{V^{-1}_{t-1}}\right)\right]$. Following the trick introduced by the authors of \citep{AYPS11}, we apply Cauchy-Schwarz inequality to the term $\sum_{t\in[T]}\left(1\wedge\left\|X_t\right\|_{V^{-1}_{t-1}}\right).$
\begin{align}
\sum_{t\in[T]}\left(1\wedge\left\|X_t\right\|_{V^{-1}_{t-1}}\right)\leq\sqrt{T}\sqrt{\sum_{t\in[T]}1\wedge\left\|X_t\right\|^2_{V^{-1}_{t-1}}}.
\end{align}
By dividing the whole time horizon into consecutive pieces of length $w,$ we have 
\begin{align}
\label{eq:sw9}
\sqrt{\sum_{t\in[T]}1\wedge\left\|X_t\right\|^2_{V^{-1}_{t-1}}}\leq\sqrt{\sum_{i=0}^{\lceil T/w\rceil-1}\sum_{t=i\cdot w+1}^{(i+1) w}1\wedge\left\|X_t\right\|^2_{V^{-1}_{t-1}}}.
\end{align}
While a similar quantity has been analyzed by Lemma 11 of \citep{AYPS11}, we note that due to the fact that $V_{t}$'s are accumulated according to the sliding window principle, the key eq. (6) in Lemma 11's proof breaks, and thus the analysis of \citep{AYPS11} cannot be applied here. To this end, we state a technical lemma based on the Sherman-Morrison formula.
\begin{lemma}
	\label{lemma:sw2}
	For any $i\leq \lceil T/w\rceil-1,$ 
	\begin{align*}
	\sum_{t=i\cdot w+1}^{(i+1) w}1\wedge\left\|X_t\right\|^2_{V^{-1}_{t-1}}\leq\sum_{t=i\cdot w+1}^{(i+1) w}1\wedge\left\|X_t\right\|^2_{\overline{V}_{t-1}^{-1}},
	\end{align*}
	where 
	\begin{align}
	\overline{V}_{t-1}=\sum_{s=i\cdot w+1}^{t-1}X_sX_s^{\top}+\lambda I.
	\end{align}
\end{lemma}
\begin{proof}{Proof of Lemma \ref{lemma:sw2}.}
	For a fixed $i\leq \lceil T/w\rceil-1,$ 
	\begin{align}
	\nonumber\sum_{t=i\cdot w+1}^{(i+1) w}1\wedge\left\|X_t\right\|^2_{V^{-1}_{t-1}}=&\sum_{t=i\cdot w+1}^{(i+1) w}1\wedge X_t^{\top}V_{t-1}^{-1}X_t\\
	\label{eq:sw7}=&\sum_{t=i\cdot w+1}^{(i+1) w}1\wedge X_t^{\top}\left(\sum_{s=1\vee(t-w)}^{t-1}X_sX_s^{\top}+\lambda I\right)^{-1}X_t.
	\end{align}
	Note that $i\cdot w+1\geq1$ and $i\cdot w+1\geq t-w~\forall t\leq(i+1)w,$ we have
	\begin{align}
	i\cdot w+1\geq 1\vee (t-w).
	\end{align}
	Consider any $d$-by-$d$ positive definite matrix $A$ and $d$-dimensional vector $y,$ then by the Sherman-Morrison formula, the matrix
	\begin{align}
	B=A^{-1}-\left(A+yy^{\top}\right)^{-1}=A^{-1}-A^{-1}+\frac{A^{-1}yy^{\top}A^{-1}}{1+y^{\top}A^{-1}y}=\frac{A^{-1}yy^{\top}A^{-1}}{1+y^{\top}A^{-1}y}
	\end{align}
	is positive semi-definite. Therefore, for a given $t,$ we can iteratively apply this fact to obtain
	\begin{align}
	\nonumber&X_t^{\top}\left(\sum_{s=i\cdot w+1}^{t-1}X_sX_s^{\top}+\lambda I\right)^{-1}X_t\\
	\nonumber=&X_t^{\top}\left(\sum_{s=i\cdot w}^{t-1}X_sX_s^{\top}+\lambda I\right)^{-1}X_t+X_t^{\top}\left(\left(\sum_{s=i\cdot w+1}^{t-1}X_sX_s^{\top}+\lambda I\right)^{-1}-\left(\sum_{s=i\cdot w}^{t-1}X_sX_s^{\top}+\lambda I\right)^{-1}\right)X_t\\
	\nonumber=&X_t^{\top}\left(\sum_{s=i\cdot w}^{t-1}X_sX_s^{\top}+\lambda I\right)^{-1}X_t+X_t^{\top}\left(\left(\sum_{s=i\cdot w+1}^{t-1}X_sX_s^{\top}+\lambda I\right)^{-1}-\left(X_{i\cdot w}X_{i\cdot w}^{\top}+\sum_{s=i\cdot w+1}^{t-1}X_sX_s^{\top}+\lambda I\right)^{-1}\right)X_t\\
	\nonumber\geq&X_t^{\top}\left(\sum_{s=i\cdot w}^{t-1}X_sX_s^{\top}+\lambda I\right)^{-1}X_t\\
	\nonumber&\vdots\\
	\label{eq:sw8}\geq&X_t^{\top}\left(\sum_{s=1\vee(t-w)}^{t-1}X_sX_s^{\top}+\lambda I\right)^{-1}X_t.
	\end{align}
	Plugging inequality (\ref{eq:sw8}) to (\ref{eq:sw7}), we have
	\begin{align}
	\nonumber\sum_{t=i\cdot w+1}^{(i+1) w}1\wedge\left\|X_t\right\|^2_{V^{-1}_{t-1}}\leq&\sum_{t=i\cdot w+1}^{(i+1) w}1\wedge X_t^{\top}\left(\sum_{s=i\cdot w+1}^{t-1}X_sX_s^{\top}+\lambda I\right)^{-1}X_t\\
	\leq&\sum_{t=i\cdot w+1}^{(i+1) w}1\wedge\left\|X_t\right\|^2_{\overline{V}_{t-1}^{-1}},
	\end{align}
	which concludes the proof.\halmos
\end{proof}
From Lemma \ref{lemma:sw2} and eq. (\ref{eq:sw9}), we know that
\begin{align}
\nonumber2\beta\sum_{t\in[T]}\left(1\wedge\left\|X_t\right\|_{V^{-1}_{t-1}}\right)\leq&2\beta\sqrt{T}\cdot\sqrt{\sum_{i=0}^{\lceil T/w\rceil-1}\sum_{t=i\cdot w+1}^{(i+1) w}1\wedge\left\|X_t\right\|^2_{\overline{V}^{-1}_{t-1}}}\\
\label{eq:sw10}\leq&2\beta\sqrt{T}\cdot\sqrt{\sum_{i=0}^{\lceil T/w\rceil-1}2d\ln\left(\frac{d\lambda+wL^2}{d\lambda}\right)}\\
\nonumber\leq&2\beta T\sqrt{\frac{2d}{w}\ln\left(\frac{d\lambda+wL^2}{d\lambda}\right)}.
\end{align}
Here, eq. (\ref{eq:sw10}) follows from Lemma 11 of \citep{AYPS11}.

Now putting these two parts to eq. (\ref{eq:sw11}), we have
\begin{align}
\nonumber&\Ex\left[\text{Regret}_T\left(\swofu\right)\right]\\
\nonumber\leq&2LwB_T+2\beta T\sqrt{\frac{2d}{w}\ln\left(\frac{d\lambda+wL^2}{d\lambda}\right)}+2T\delta\\
\label{eq:explicit_swucb_bd}=&2LwB_T+\frac{2T}{\sqrt{w}}\left(R\sqrt{d\ln\left(\frac{1+wL^2/\lambda}{\delta}\right)}+\sqrt{\lambda}S\right)\sqrt{2d\ln\left(\frac{d\lambda+wL^2}{d\lambda}\right)}+2T\delta.
\end{align}
Now if $B_T$ is known, we can take $w=\Theta\left((dT)^{2/3}B_t^{-2/3}\right)$ and $\delta=1/T,$ we have
\begin{align}
\Ex\left[\nonumber\text{Regret}_T\left(\swofu\right)\right]=\widetilde{O}\left(d^{\frac{2}{3}}B_T^{\frac{1}{3}}T^{\frac{2}{3}}\right);
\end{align}
while if $B_T$ is not unknown, taking $w=\Theta\left((dT)^{2/3}\right)$ and $\delta=1/T,$ we have
\begin{align}
\Ex\left[\nonumber\text{Regret}_T\left(\swofu\right)\right]=\widetilde{O}\left(d^{\frac{2}{3}}B_TT^{\frac{2}{3}}\right).
\end{align}
\section{Proof of Lemma \ref{lemma:bob}}\label{sec:lemma:bob}
For any block $i,$ the absolute sum of rewards can be written as
\begin{align*}
	\left|\sum_{t=(i-1)H+1}^{i\cdot H\wedge T}\langle X_t,\theta_t\rangle+\eta_t\right|\leq\sum_{t=(i-1)H+1}^{i\cdot H\wedge T}\left|\langle X_t,\theta_t\rangle\right|+\left|\sum_{t=(i-1)H+1}^{i\cdot H\wedge T}\eta_t\right|\leq H\nu+\left|\sum_{t=(i-1)H+1}^{i\cdot H\wedge T}\eta_t\right|,
\end{align*}
where we have iteratively applied the triangle inequality as well as the fact that $\left|\langle X_t,\theta_t\rangle\right|\leq \nu$ for all $t.$
	
Now by property of the $R$-sub-Gaussian \citep{RH18}, we have the absolute value of the noise term $\eta_t$ exceeds $2R\sqrt{\ln T}$ for a fixed $t$ with probability at most $1/T^2$ \ie, 
\begin{align}
	\Pr\left(\left|\sum_{t=(i-1)H+1}^{i\cdot H\wedge T}\eta_t\right|\geq 2R\sqrt{H\ln\frac{T}{\sqrt{H}}}\right)\leq\frac{2H}{T^2}.
\end{align} 
Applying a simple union bound, we have
\begin{align}
	\Pr\left(\exists i\in\left\lceil\frac{T}{H}\right\rceil:\left|\sum_{t=(i-1)H+1}^{i\cdot H\wedge T}\eta_t\right|\geq 2R\sqrt{H\ln\frac{T}{\sqrt{H}}}\right)\leq\sum_{i=1}^{\lceil T/H\rceil}\Pr\left(\left|\sum_{t=(i-1)H+1}^{i\cdot H\wedge T}\eta_t\right|\geq 2R\sqrt{H\ln\frac{T}{\sqrt{H}}}\right)\leq\frac{2}{T}.
\end{align}
	Therefore, we have
\begin{align}
	\Pr\left(Q\geq H\nu+2R\sqrt{H\ln\frac{T}{\sqrt{H}}}\right)\leq\Pr\left(\exists i\in\left\lceil\frac{T}{H}\right\rceil:\left|\sum_{t=(i-1)H+1}^{i\cdot H\wedge T}\eta_t\right|\geq 2R\sqrt{H\ln\frac{T}{\sqrt{H}}}\right)\leq\frac{2}{T}.
\end{align}
	The statement then follows. 
\section{Proof of Proposition \ref{prop:bob}}
\label{sec:prop:bob}
By design of the \bob, its dynamic regret can be decomposed as the regret of the \swofu~with the optimally tuned window size $w_i=w^{\dag}$ for each block $i$ plus the loss due to learning the value $w^{\dag}$ with the EXP3 algorithm, \ie,
\begin{align}
	\label{eq:bob_decompose}
	\nonumber\Ex\left[\text{Regret}_T(\bob)\right]=&\Ex\left[\sum_{t=1}^T\langle x_t^*,\theta_t\rangle-\sum_{t=1}^T\langle X_t,\theta_t\rangle\right]\\
	\nonumber=&\Ex\left[\sum_{t=1}^T\langle x_t^*,\theta_t\rangle-\sum_{i=1}^{\lceil T/H\rceil}\sum_{t=(i-1)H+1}^{i\cdot H\wedge T}\left\langle X_t^{w^{\dag}},\theta_t\right\rangle\right]\\
	&+\Ex\left[\sum_{i=1}^{\lceil T/H\rceil}\sum_{t=(i-1)H+1}^{i\cdot H\wedge T}\left\langle X_t^{w^{\dag}},\theta_t\right\rangle\right.\left.-\sum_{i=1}^{\lceil T/H\rceil}\sum_{t=(i-1)H+1}^{i\cdot H\wedge T}\left\langle X_t^{w_i},\theta_t\right\rangle\right].
\end{align}
Here, eq. (\ref{eq:bob_decompose}) holds as the \bob~restarts the \swofu~in each block, and for a round $t$ in block $i,$ $X_t^{w}$ refers to the action selected in round $t$ by the \swofu~with window size $w\wedge (t-(i-1)H-1)$ initiated at the beginning of block $i.$ 

By Theorem \ref{theorem:sw_main}, the first expectation in eq. (\ref{eq:bob_decompose}) can be upper bounded as
\begin{align}
	\label{eq:bob_decompose1}
	\nonumber\Ex\left[\sum_{t=1}^T\langle x_t^*,\theta_t\rangle-\sum_{i=1}^{\lceil T/H\rceil}\sum_{t=(i-1)H+1}^{i\cdot H\wedge T}\left\langle X_t^{w^{\dag}},\theta_t\right\rangle\right]
	\nonumber=&\Ex\left[\sum_{i=1}^{\lceil T/H\rceil}\sum_{t=(i-1)H+1}^{i\cdot H\wedge T}\left\langle x^*_t-X_t^{w^{\dag}},\theta_t\right\rangle\right]\\
	\nonumber=&\sum_{i=1}^{\lceil T/H\rceil}\widetilde{O}\left(w^{\dag}B_T(i)+\frac{dH}{\sqrt{w^{\dag}}}\right)\\
	=&\widetilde{O}\left(w^{\dag}B_T+\frac{dT}{\sqrt{w^{\dag}}}\right),
\end{align}
where $$B_{T}(i)=\sum_{t=(i-1)H+1}^{(i\cdot H\wedge t)-1}\|\theta_{t}-\theta_{t+1}\|_2$$ is the total variation in block $i.$ 

We then turn to the second expectation in eq. (\ref{eq:bob_decompose}). We can easily see that the number of rounds for the EXP3 algorithm is $\lceil T/H\rceil$ and the number of possible values of $w_i$'s is $|J|.$ If the maximum absolute sum of reward of any block does not exceed $Q,$ the authors of \citep{ABFS02} gives the following regret bound.
\begin{align}
	\nonumber&\Ex\left[\sum_{i=1}^{\lceil T/H\rceil}\sum_{t=(i-1)H+1}^{i\cdot H\wedge T}\left\langle X_t^{w^{\dag}},\theta_t\right\rangle.-\sum_{i=1}^{\lceil T/H\rceil}\sum_{t=(i-1)H+1}^{i\cdot H\wedge T}\left\langle X_t^{w_i},\theta_t\right\rangle\middle|\forall i\in[\lceil T/H\rceil]\sum_{t=(i-1) H+1}^{i\cdot H\wedge T}Y_t\leq Q/2\right]\\
	\label{eq:bob_decompose2}=&\widetilde{O}\left(Q\sqrt{\frac{|J|T}{H}}\right).
\end{align} 
Note that the regret of our problem is at most $T,$ eq. (\ref{eq:bob_decompose2}) can be further upper bounded as
\begin{align}
	\label{eq:bob_decompose3}
\nonumber&\Ex\left[\sum_{i=1}^{\lceil T/H\rceil}\sum_{t=(i-1)H+1}^{i\cdot H\wedge T}\left\langle X_t^{w^{\dag}},\theta_t\right\rangle-\sum_{i=1}^{\lceil T/H\rceil}\sum_{t=(i-1)H+1}^{i\cdot H\wedge T}\left\langle X_t^{w_i},\theta_t\right\rangle\right]\\
\nonumber\leq&\widetilde{O}\left(Q\sqrt{\frac{|J|T}{H}}\right)\times\Pr\left(\forall i\in[\lceil T/H\rceil]\sum_{t=(i-1) H+1}^{i\cdot H\wedge T}Y_t\leq Q/2\right)\\
\nonumber&+\Ex\left[\sum_{i=1}^{\lceil T/H\rceil}\sum_{t=(i-1)H+1}^{i\cdot H\wedge T}\left\langle X_t^{w^{\dag}},\theta_t\right\rangle-\sum_{i=1}^{\lceil T/H\rceil}\sum_{t=(i-1)H+1}^{i\cdot H\wedge T}\left\langle X_t^{w_i},\theta_t\right\rangle\middle| \exists i\in[\lceil T/H\rceil]\sum_{t=(i-1) H+1}^{i\cdot H\wedge T}Y_t\geq Q/2\right]\\
\nonumber&\quad\times\Pr\left(\exists i\in[\lceil T/H\rceil]\sum_{t=(i-1) H+1}^{i\cdot H\wedge T}Y_t\geq Q/2\right)\\
\nonumber\leq&\widetilde{O}\left(\sqrt{H|J|T}\right)+T\cdot\frac{2}{T}\\
=&\widetilde{O}\left(\sqrt{H|J|T}\right).
\end{align} 
Combining eq. (\ref{eq:bob_decompose}), (\ref{eq:bob_decompose1}), and (\ref{eq:bob_decompose3}), the statement follows.

\section{Proof of Theorem \ref{theorem:bob}}
\label{sec:theorem:bob}
With Proposition \ref{prop:bob} as well as the choices of $H$ and $J$ in eq. (\ref{eq:bob_parameters}), the regret of the \bob~is 
\begin{align}
&\R_T(\bob)=\widetilde{O}\left(w^{\dag}B_T+\frac{dT}{\sqrt{w^{\dag}}}+\sqrt{H|J|T}\right)=\widetilde{O}\left(w^{\dag}B_T+\frac{dT}{\sqrt{w^{\dag}}}+d^{\frac{1}{2}}T^{\frac{3}{4}}\right).
\end{align}
Therefore, we have that when $B_T\geq d^{-1/2}T^{1/4},$ the \bob~is able to converge to the optimal window size, \ie, $w^{\dag}=w^*~(\leq H),$ and the dynamic regret of the \bob~is upper bounded as
\begin{align}
\R_T(\bob)=&\widetilde{O}\left(d^{\frac{2}{3}}B_T^{\frac{1}{3}}T^{\frac{2}{3}}+d^{\frac{1}{2}}T^{\frac{3}{4}}\right);
\end{align}
while if $B_T< d^{-1/2}T^{1/4},$ the \bob~converges to the window size $w^{\dag}=H,$ and the dynamic regret is 
\begin{align}
\R_T(\bob)=&\widetilde{O}\left(dB_TT^{\frac{1}{2}}+d^{\frac{1}{2}}T^{\frac{3}{4}}\right)=\widetilde{O}\left(d^{\frac{1}{2}}T^{\frac{3}{4}}\right).
\end{align}
Combining the above two cases, we conclude the desired dynamic regret bound.

\section{Proof of Theorem \ref{theorem:mab_deviation}}\label{sec:theorem:mab_deviation}
Similar to eq. (\ref{eq:sw16}), we can rewrite the difference $\hat{\theta}_t-\theta_t$ as
\begin{align}
\nonumber&V_{t-1}^{*}\left(\sum_{s=1\vee(t-w)}^{t-1}X_sX_s^{\top}\theta_s+\sum_{s=1\vee(t-w)}^{t-1}\eta_sX_s\right)-\theta_t\\
\label{eq:sw22}=&V_{t-1}^{*}\sum_{s=1\vee(t-w)}^{t-1}X_sX_s^{\top}\left(\theta_s-\theta_t\right)+V_{t-1}^{*}\left(\sum_{s=1\vee(t-w)}^{t-1}\eta_sX_s\right).
\end{align}
We then analyze the two terms in eq. (\ref{eq:sw22}) separately. For the first term,
\begin{align}
\nonumber\left\|V_{t-1}^{*}\sum_{s=1\vee(t-w)}^{t-1}X_sX_s^{\top}\left(\theta_s-\theta_t\right)\right\|_{\infty}
=&\left\|V_{t-1}^{*}\sum_{s=1\vee(t-w)}^{t-1}X_sX_s^{\top}\left[\sum_{p=s}^{t-1}\left(\theta_p-\theta_{p+1}\right)\right]\right\|_{\infty}\\
\nonumber=&\left\|\sum_{p=1	\vee(t-w)}^{t-1}\left[V_{t-1}^{*}\sum_{s=1\vee(t-w)}^{p}X_sX_s^{\top}\left(\theta_p-\theta_{p+1}\right)\right]\right\|_{\infty}\\
\nonumber\leq&\sum_{p=1	\vee(t-w)}^{t-1}\left\|V_{t-1}^{*}\sum_{s=1\vee(t-w)}^{p}X_sX_s^{\top}\left(\theta_p-\theta_{p+1}\right)\right\|_{\infty}\\
\label{eq:sw23}\leq&\sum^{t-1}_{p = 1\vee (t-w)}\left\|\theta_s-\theta_{s+1}\right\|_{\infty}.
\end{align}	
Here, almost all the steps follow exactly the same arguments as those of eq. (\ref{eq:sw})-(\ref{eq:sw3}), except that in inequality (\ref{eq:sw23}), we make the direct observation that 
\begin{align}
\label{eq:sw25}
V_{t-1}^*=\begin{pmatrix}
\frac{\bm{1}[N_{t-1}(1)>0]}{N_{t-1}(1)}&0&\ldots&\ldots&\ldots&0\\
0&	\frac{\bm{1}[N_{t-1}(2)>0]}{N_{t-1}(2)}&0&\ldots&\ldots&0\\
0&0&\ddots&0&\ldots&0\\
\vdots&\vdots&\vdots&\ddots&\ddots&\vdots\\
0&0&0&\ldots&\frac{\bm{1}[N_{t-1}(d-1)>0]}{N_{t-1}(d-1)}&0\\
0&0&0&\ldots&0&\frac{\bm{1}[N_{t-1}(d)>0]}{N_{t-1}(d)}\\
\end{pmatrix}
\end{align}
and 
\begin{align}
\sum_{s=1\vee(t-w)}^pX_sX_s^{\top}=\begin{pmatrix}
N'_{p}(1)&0&\ldots&\ldots&\ldots&0\\
0&	N'_{p}(2)&0&\ldots&\ldots&0\\
0&0&\ddots&0&\ldots&0\\
\vdots&\vdots&\vdots&\ddots&\ddots&\vdots\\
0&0&0&\ldots&N'_{p}(d-1)&0\\
0&0&0&\ldots&0&N'_{p}(d)\\
\end{pmatrix},
\end{align}
where $N'_{p}(i)$ is the number of times that action $e_i$ is selected during rounds $1\vee(t-w),\ldots,p$ for all $i\in[d].$ As $p\leq t-1,$ we have $N'_{p}(i)\leq N_{t-1}(i)$ for all $i\in[d].$ Now, $V_{t-1}^{*}\sum_{s=1\vee(t-w)}^{p}X_sX_s^{\top}$ is a diagonal matrix with all diagonal entries less than 1, and hence the argument.

For the second term of eq. (\ref{eq:sw22}), we consider for any fixed $i\in[d],$
\begin{align}
\nonumber\left|e^{\top}_iV_{t-1}^*\left(\sum_{s=1\vee(t-w)}^{t-1}\eta_sX_s\right)\right|=&\frac{\bm{1}[N_{t-1}(i)>0]}{N_{t-1}(i)}\left|e^{\top}_i\left(\sum_{s=1\vee(t-w)}^{t-1}\eta_sX_s\right)\right|\\
\label{eq:sw26}=&\frac{\bm{1}[N_{t-1}(i)>0]\left(\sum_{s=1\vee(t-w)}^{t-1}\bm{1}[I_s=i]\eta_s\right)}{N_{t-1}(i)},
\end{align}
where the first step again use the definition of $V_{t-1}^*$ in eq. (\ref{eq:sw25}). Now if $N_{t-1}(i)=0,$ eq. (\ref{eq:sw26}) equals to 0; while if $N_{t-1}(i)>0,$ we can apply the Corollary 1.7 of \citep{RH18} to obtain that
\begin{align}
\label{eq:sw27}
\Pr\left(\left|\frac{\bm{1}[N_{t-1}(i)>0]\left(\sum_{s=1\vee(t-w)}^{t-1}\bm{1}[I_s=i]\eta_s\right)}{N_{t-1}(i)}\right|\leq R\sqrt{\frac{2\ln\left({2dT^2}\right)}{N_{t-1}(i)}}\right)\geq1-\frac{1}{dT^2}.
\end{align}
Hence, with probability at least $1-1/dT^2,$ for any fixed $t\in[T]$ and any fixed $i\in[d],$ 
\begin{align}
\nonumber\left|e_i^\top ( \hat{\theta}_t - \theta_t)\right|=&\left| e_i^\top \left( V_{t-1}^{*}\sum_{s=1\vee(t-w)}^{t-1}X_sX_s^{\top}\left(\theta_s-\theta_t\right)\right)+ e_i^\top V_{t-1}^{*}\left(\sum_{s=1\vee(t-w)}^{t-1}\eta_sX_s-\lambda\theta_t\right) \right|\\
\label{eq:sw28}\leq&\left| e_i^\top \left( V_{t-1}^{*}\sum_{s=1\vee(t-w)}^{t-1}X_sX_s^{\top}\left(\theta_s-\theta_t\right)\right)\right|+\left| e_i^\top V_{t-1}^{*}\left(\sum_{s=1\vee(t-w)}^{t-1}\eta_sX_s-\lambda\theta_t\right) \right| \\
\label{eq:sw29}\leq& \left\|e_i\right\|_1\cdot \left\|V_{t-1}^{*}\sum_{s=1\vee(t-w)}^{t-1}X_sX_s^{\top}\left(\theta_s-\theta_t\right)\right\|_{\infty}+R\sqrt{\frac{2\ln\left({2dT^2}\right)}{N_{t-1}(i)}}\\
\label{eq:sw30}\leq&\sum^{t-1}_{s = 1\vee (t-w)}\left\|\theta_s-\theta_{s+1}\right\|_{\infty} +  R\sqrt{\frac{2\ln\left({2dT^2}\right)}{N_{t-1}(i)}},
\end{align}
where inequality (\ref{eq:sw28}) applies the triangle inequality, inequality (\ref{eq:sw29}) follows from the Holder's inequality as well as inequality (\ref{eq:sw26}) and (\ref{eq:sw27}), and inequality (\ref{eq:sw30}) follows from inequality (\ref{eq:sw23}). 

The statement of the theorem now follows immediately by applying union bound over the decision set and the time horizon as well as the simple observation $\|e_i\|_{V_{t-1}^*}=\sqrt{1/N_{t-1}(i)}.$
\section{Proof of Theorem \ref{theorem:glm_sw_deviation}}
\label{sec:theorem:glm_sw_deviation}
From the proof of Proposition 1 in \cite{FCAS10}, we know that for all $x\in D$
\begin{align}
\label{eq:glm1}
\left|\mu\left(\left\langle x,\theta_t\right\rangle\right)-\mu\left(\left\langle x,\hat\theta_t\right\rangle\right)\right|\leq {k_{\mu}}\left|x^{\top}G_{t-1}^{-1}\left[\sum_{s=1\vee(t-w)}^{t-1}\left(\mu\left(\left\langle X_s,\theta_t\right\rangle\right)-\mu\left(\left\langle X_s,\hat\theta_t\right\rangle\right)\right)X_s\right]\right|,
\end{align}
where 
$$G_{t-1}=\int_{0}^1\left[\sum_{s=1\vee(t-w)}^{t-1}X_sX_s^{\top}\mu\left(\left\langle X_s,s_0\theta_t+(1-s_0)\hat{\theta}_t\right\rangle\right)\right]ds_0$$
By virtue of the maximum quasi-likelihood estimation, \ie, eq. (\ref{eq:glm}) we have
\begin{align}
\sum_{s=1\vee(t-w)}^{t-1}\mu\left(\left\langle X_s,\hat\theta_t\right\rangle\right)X_s=	\sum_{s=1\vee(t-w)}^{t-1}Y_sX_s=\sum_{s=1\vee(t-w)}^{t-1}\left(\mu\left(\left\langle X_s,\theta_s\right\rangle\right)+\eta_s\right)X_s,
\end{align}
and (\ref{eq:glm1}) is
\begin{align}
\nonumber&k_{\mu}\left|x^{\top}G_{t-1}^{-1}\sum_{s=1\vee(t-w)}^{t-1}\left(\mu\left(\left\langle X_s,\theta_t\right\rangle\right)-\mu\left(\left\langle X_s,\theta_s\right\rangle\right)-\eta_s\right)X_s\right|\\
\nonumber=&k_{\mu}\left|x^{\top}G_{t-1}^{-1}\sum_{s=1\vee(t-w)}^{t-1}\left(\mu\left(\left\langle X_s,\theta_t\right\rangle\right)-\mu\left(\left\langle X_s,\theta_s\right\rangle\right)\right)X_s-x^{\top}G_{t-1}^{-1}\sum_{s=1\vee(t-w)}^{t-1}\eta_sX_s\right|\\
\label{eq:glm2}\leq&k_{\mu}\left|x^{\top}G_{t-1}^{-1}\sum_{s=1\vee(t-w)}^{t-1}\left(\mu\left(\left\langle X_s,\theta_t\right\rangle\right)-\mu\left(\left\langle X_s,\theta_s\right\rangle\right)\right)X_s\right|+k_{\mu}\left|x^{\top}G_{t-1}^{-1}\sum_{s=1\vee(t-w)}^{t-1}\eta_sX_s\right|\\
\label{eq:glm3}\leq&k_{\mu}\left|x^{\top}G_{t-1}^{-1}\sum_{s=1\vee(t-w)}^{t-1}\left(\mu\left(\left\langle X_s,\theta_t\right\rangle\right)-\mu\left(\left\langle X_s,\theta_s\right\rangle\right)\right)X_s\right|+\beta\|x\|_{V^{-1}_{t-1}}\\
\label{eq:glm4}\leq&k_{\mu}\left\|x\right\|_2 \left\|G_{t-1}^{-1}\sum_{s=1\vee(t-w)}^{t-1}\left(\mu\left(\left\langle X_s,\theta_t\right\rangle\right)-\mu\left(\left\langle X_s,\theta_s\right\rangle\right)\right)X_s\right\|_2+\beta\|x\|_{V^{-1}_{t-1}}\\
\nonumber\leq&\frac{k_{\mu}L}{c_{\mu}}\left\|V_{t-1}^{-1}\sum_{s=1\vee(t-w)}^{t-1}\left(\mu\left(\left\langle X_s,\theta_t\right\rangle\right)-\mu\left(\left\langle X_s,\theta_s\right\rangle\right)\right)X_s\right\|_2+\beta\|x\|_{V^{-1}_{t-1}}.
\end{align}
Here, inequality (\ref{eq:glm2}) is a consequence of the triangle inequality, inequality (\ref{eq:glm3}) again follows from Proposition 1 of \cite{FCAS10}, inequality (\ref{eq:glm4}) is the Cauchy-Schwarz inequality, and the last step uses the fact that $G_{t-1}\succeq c_{\mu}V_{t-1}.$ For the firs quantity, we have
\begin{align}
\nonumber&\left\|V_{t-1}^{-1}\sum_{s=1\vee(t-w)}^{t-1}\left(\mu\left(\left\langle X_s,\theta_t\right\rangle\right)-\mu\left(\left\langle X_s,\theta_s\right\rangle\right)\right)X_s\right\|_2\\
\nonumber=&\left\|V_{t-1}^{-1}\sum_{s=1\vee(t-w)}^{t-1}X_s\sum_{p=s}^{t-1}\left(\mu\left(\left\langle X_s,\theta_{p+1}\right\rangle\right)-\mu\left(\left\langle X_s,\theta_p\right\rangle\right)\right)\right\|_2\\
\nonumber=&\left\|V_{t-1}^{-1}\sum_{p=1\vee(t-w)}^{t-1}\sum_{s=1\vee(t-w)}^{p}X_s\left(\mu\left(\left\langle X_s,\theta_{p+1}\right\rangle\right)-\mu\left(\left\langle X_s,\theta_p\right\rangle\right)\right)\right\|_2\\
\label{eq:glm5}\leq&\sum_{p=1\vee(t-w)}^{t-1}\left\|V_{t-1}^{-1}\sum_{s=1\vee(t-w)}^{p}X_s\left(\mu\left(\left\langle X_s,\theta_{p+1}\right\rangle\right)-\mu\left(\left\langle X_s,\theta_p\right\rangle\right)\right)\right\|_2\\
\label{eq:glm6}=&\sum_{p=1\vee(t-w)}^{t-1}\left\|V_{t-1}^{-1}\sum_{s=1\vee(t-w)}^{p}X_s\dot{\mu}\left(\left\langle X_s,\tilde{\theta}_p\right\rangle\right)X_s^{\top}\left(\theta_{p+1}-\theta_p\right)\right\|_2\\
\nonumber=&\sum_{p=1\vee(t-w)}^{t-1}\left\|V_{t-1}^{-1}\sum_{s=1\vee(t-w)}^{p}\dot{\mu}\left(\left\langle X_s,\tilde{\theta}_p\right\rangle\right)X_sX_s^{\top}\left(\theta_{p+1}-\theta_p\right)\right\|_2\\
\label{eq:glm7}=&\sum_{p=1\vee(t-w)}^{t-1}\lambda_{\max}\left(V_{t-1}^{-1}\sum_{s=1\vee(t-w)}^{p}\dot{\mu}\left(\left\langle X_s,\tilde{\theta}_p\right\rangle\right)X_sX_s^{\top}\right)\left\|\left(\theta_{p+1}-\theta_p\right)\right\|_2\\
\nonumber\leq&k_{\mu}\sum_{p=1\vee(t-w)}^{t-1}\lambda_{\max}\left(\left(\sum_{s=1\vee(t-w)}^{p}X_sX_s^{\top}\right)V_{t-1}^{-2}\left(\sum_{s=1\vee(t-w)}^{p}X_sX_s^{\top}\right)\right)\left\|\left(\theta_{p+1}-\theta_p\right)\right\|_2\\
\label{eq:glm8}\leq&k_{\mu}\sum_{p=1\vee(t-w)}^{t-1}\left\|\left(\theta_{p+1}-\theta_p\right)\right\|_2,
\end{align}
where inequality (\ref{eq:glm5}) is an immediate consequence of the triangle inequality, eq. (\ref{eq:glm6}) utilizes the mean value theorem (with $\tilde{\theta}_p$ being some certain linear combination of $\theta_p$ and $\theta_{p+1}$ for all $p$), and inequalities (\ref{eq:glm7}) and (\ref{eq:glm8}) follow from the same steps as the proof of Lemma \ref{lemma:sw} in Section \ref{sec:theorem:sw_deviation}.
\section{Proof of Theorem \ref{theorem:semi_lower_bound}}\label{sec:theorem:semi_lower_bound}
We start with a regret lower bound result from \citep{BGZ18} on drifting $K$-armed bandits:
\begin{theorem}[\cite{BGZ18}]\label{thm:bgz_lowerbd}
Consider the drifting $K$-armed bandit problem, where $K\geq 2$, with $T\geq 1$ rounds. For any $B_T\in [1 / K, T/K]$, there exists a finite class of reward distributions $\tilde{\mathcal{P}} = \{\tilde{P}^{(\ell)}\}^L_{\ell=1}$, where $ \tilde{P}^{(\ell)} = \{\tilde{P}^{(\ell)}_{t, k}\}_{t\in [T], k\in [K]}$, that satisfy the following:
\begin{itemize}
\item Each $\tilde{P}^{(\ell)}_{t, k}$ represents the reward distribution of arm $k$ in round $t$ under distribution $\tilde{P}^{(\ell)}$. For each $ell, t, k$, the distribution $\tilde{P}^{(\ell)}_{t, k}$  is a Bernoulli distribution, with the mean denoted $\tilde{\theta}^{(\ell)}_{t, k}$.
\item For every $\ell\in [L]$, the following variational budget inequality holds:
$$
\sum^{T-1}_{t=1} \max_{k\in [K]}\left\{ \left|\tilde{\theta}^{(\ell)}_{t + 1}(k) - \tilde{\theta}^{(\ell)}_{t}(k)  \right| \right\}\leq B_T.
$$
\item For any non-anticipatory policy $\tilde{\pi}$ , there exists $\ell\in [L]$ under which the dynamic regret is lower bounded:
$$
\sum^T_{t=1}\left\{ \max_{k\in [K]}\tilde{\theta}^{(\ell)}_{t}(k) - \mathbb{E}[\tilde{\theta}^{(\ell)}_{t}(I_t)]\right\} \geq \frac{1}{4\sqrt{2}} (K B_T)^{1/3} T^{2/3}.
$$
We denote the choice of arm under policy $\tilde{\pi}$ in round $t$ as $I_t$, and the expectation is taken over the randomness in the choice of $I_t$, which is caused by the previous outcomes and the policy's internal randomness.
\end{itemize}
\end{theorem}
We prove the Theorem by modifying the class of instances ${\cal P}$ to suit the setting of drifting combinatorial semi-bandits. The modification follows the style of Kveton et al. \citep{KWAS15}. Let $d, m$ be two integers, where $d$ is divisible by $m$ W.L.O.G.. We define the ground set $E = [d]$. In addition, we define the action set ${\cal E}_t = \{a_1,\ldots,a_{d/m}\} \subset \{0, 1\}^d$, which contains $d/m$ combinatorial arms and does not vary with $t$. Each combinatorial arm $a_i$ belongs to $\{0, 1\}^d$. For each $1\leq i\leq d/m$, we define $a_i(j)=1$ if $(i-1)m+1\leq j\leq i\cdot m$, and $a_i(j)=0$ for other $j$. 

Consider Theorem \ref{thm:bgz_lowerbd} when  $K = d / m \geq 2$, and let $\tilde{\cal P} = \{\tilde{P}^{(\ell)}\}^L_{\ell=1}$ be the class of reward distributions for the regret lower bound. For each $\tilde{P}_\ell = \{\tilde{P}^{(\ell)}_{t, k}\}_{t\in [T], k\in [K]}$ (which is on the $K = d/m$-armed bandit instance),  we construct another reward distribution $P_\ell = \{P^{(\ell)}_{t, j}\}_{t\in [T], j\in [d]} $ that is defined on the combinatorial semi-bandit instance. For each $j\in [d]$, we identify the index $i\in [d/m]$ such that $(i - 1)m + 1 \leq j\leq i\cdot m$, and define $P^{(\ell)}_{t, j}$ to be the same distribution as $\tilde{P}^{(\ell)}_{t, i}$. That is,  $P^{(\ell)}_{t, j}$ is a Bernoulli distribution with mean $\theta_{t}(j) = \tilde{\theta}_{t}(i)$, where $i = \lceil j / m \rceil$. By the second property in Theorem \ref{thm:bgz_lowerbd}, it is straightforward to check that $B_T$ is also a variation budget for $P^{(\ell)}$ for each $\ell$, that is, $$
\sum^{T-1}_{t=1} \max_{j\in [d]}\left\{ \left|\theta^{(\ell)}_{t + 1}(j) - \theta^{(\ell)}_{t}(j)  \right| \right\}\leq B_T.
$$
For each $1\leq i\leq d/m$, the random rewards $W_t((i - 1)m + 1), \ldots, W_t(i \cdot m)$ for the items in combinatorial arm $i$ are identical Bernoulli random variables. That is, they simultaneously realize as all ones or all zeros. 

Finally, to complete the proof, we relate the dynamic regret of any non-anticipatory policy $\pi$ on the drifting combinatorial semi-bandit instance to that of some non-anticipatory policy $\tilde{\pi}$ on the drifting $K$-armed instance. For the combinatorial bandit instance, a non-anticipatory policy $\pi$ is in fact a sequence of mappings $\{\pi_t\}^\infty_{t=1}$, where $\pi_t$ maps the historical information $H_{t-1} = \{X_s, \{ W_s(i) \}_{i\in X_s}\}^{t-1}_{s=1}$ from time $1$ to $t-1$ and a random seed $U$ to the combinatorial arm $X_t$ to pull in time $t$, or more mathematically $\pi_t(H_{t-1}, U) = X_t$. Likewise is true for any non-anticipatory policy $\tilde{\pi}$ for a $K$-armed instance. 

Given a non-anticipatory policy $\pi$ for the combinatorial semi-bandit instance, we construct another non-anticipatory policy $\tilde{\pi}$ for the $K$-armed bandit instance that mimics the behaviour of $\pi$. Suppose that $\pi_t(H, U) = X_j$ for a realization of the history $H = \{X_s, \{ W_s(i) \}_{i\in X_s}\}^{t-1}_{s=1}$ and random seed $U$. To construct $\tilde{\pi}$, we map the $H$ to the historical information $\tilde{H}$ for the $K$-armed bandit instance, where $\tilde{H} = \{\tilde{X}_s, \tilde{W}_s\}^{t-1}_{s=1}$ is defined as follows: $\tilde{X}_s = i$ iff $X_s = a_i$, and $\tilde{W}_s = \frac{1}{m}\sum_{i\in [d]} X_s(i) W_s(i)$. It is clear that $\tilde{W}_s\in \{0, 1\}$ for each $s$, by our assumption on the correlations among $\{W_t(i)\}_{i\in [d]}$. Finally, we define $\tilde{\pi}_t(\tilde{H}, U) = i$ if and only if $\pi_t(H, U) = a_i$. It is evident from our construction that $\pi_t$ is well-defined, in the sense that it maps to a unique arm for every possible realization of $\tilde{H}, U$. Importantly, for any $1\leq \ell\leq L$, we know that 
\begin{align*}
\text{Expected reward of $\pi$ under $P^{(\ell)}$} &= m \times \text{Expected reward of $\tilde{\pi}$ under $\tilde{P}^{(\ell)}$}\nonumber, \\  
\text{Optimal expected reward under $P^{(\ell)}$} &= m \times \text{Optimal expected reward under $\tilde{P}^{(\ell)},$}\nonumber
\end{align*}
or more mathematically we have $\sum^T_{t=1} \max_{a_i\in {\cal E}_t} \sum_{j : a_i(j) = 1} \theta^{(\ell)}_{t}(j) = m \times \sum^T_{t=1} \max_{k\in [K]} \tilde{\theta}^{(\ell)}_{t}(k).$
Consequently, by the third property of Theorem \ref{thm:bgz_lowerbd}, we know that for any non-anticipatory policy $\pi$, there is an index $\ell$ such that the dynamic regret of $\pi$ under $P^{(\ell)}$ is at least $m\times (\frac{1}{4\sqrt{2}}(\frac{d}{m}B_T)^{1/3} T^{2/3})$, which proves the theorem. 



\section{Proof of Theorem \ref{theorem:semi_deviation}}\label{sec:theorem:semi_deviation}
Define
$$\bar{\theta}_{t, i} = \frac{\sum^{t-1}_{s = 1\vee (t-w)} \theta_s(i) \cdot \mathbf{1}[X_s(i) = 1]}{\max\{N_{t-1}(i), 1\}}. $$
First, we claim that, with probability at least $1 - \delta$, for all  $i\in [d], t\in T$ it holds that
\begin{equation}\label{eq:d_arm_stoc}
\left| \bar{\theta}_{t, i} - \hat{\theta}_{t, i}\right| \leq 2R \sqrt{\frac{\log(2 d T / \delta)}{\max\{N_{t-1}(i),1\}}} \leq  4R \sqrt{\frac{\log(2 d T / \delta)}{N_{t-1}(i) + 1}}. 
\end{equation}
The Claim is proved by applying the following inequality for each item $i\in [d]$. Let $\Upsilon_1, \ldots, \Upsilon_T$ be i.i.d $R$-sub-Gaussian random variables with mean zero. For any $\delta\in (0, 1)$, we have
\begin{equation}\label{eq:interval_Chernoff}
\Pr\left( \left|\frac{1}{t - q + 1}\sum^t_{s = q} \Upsilon_s\right| \leq 2R \sqrt{\frac{\log(2 d T / \delta)}{t - q + 1}}\text{ for all $1\leq q\leq t\leq T$} \right)\geq 1 - \frac{\delta}{d},
\end{equation}
by Corollary 1.7 of Rigollet and H\"{u}tter \citep{RH18} and a union bound over all $(q, t)$ with $1\leq q\leq t\leq T$ (We can alternatively use Lemma 6 in Abbasi-Yadkori et al. \citep{AYPS11} for a slightly worse bound, but holds for more general $\eta_t$ ).

Next, observe that for each $i, t$, for certain we have 
\begin{align}
 \left|\bar{\theta}_{t, i} - \theta_{t, i} \right|  & \leq  \frac{1}{\max\{N_{t-1}(i), 1\}}\sum^{t-1}_{s = 1\vee (t-w)}  \mathbf{1}[X_s(i) = 1] \cdot \left| \theta_s(i) - \theta_t(i)  \right| \nonumber\\
& \leq   \frac{1}{\max\{N_{t-1}(i), 1\}}\sum^{t-1}_{s = 1\vee (t-w)}\mathbf{1}[X_s(i) = 1] \cdot \left(\sum^{t-1}_{q = s} \left| \theta_q(i) - \theta_{q+1}(i)  \right|\right) \nonumber\\
& \leq  \sum^{t-1}_{s = 1\vee (t-w)} \left| \theta_s(i) - \theta_{s+1}(i)  \right| \leq  \sum^{t-1}_{s = 1\vee (t-w)} \left\| \theta_s - \theta_{s+1}\right\|_\infty . \label{eq:d_arm_adv}
\end{align}

\section{Proof of Theorem \ref{theorem:semi_sw_main}}\label{sec:theorem:semi_sw_main}
Recall our notations on $N_{t-1}(i)$ and $\hat{\theta}_{t, i}$ (Note that $\mathbf{1}[X_s(i) = 1] = X_s(i)$):
\begin{align}
N_{t-1}(i) &= \sum^{t-1}_{s = 1\vee (t-w)} \mathbf{1}[X_s(i) = 1] \nonumber,\\
\hat{\theta}_{t, i} &= \frac{\sum^{t-1}_{s = 1\vee (t-w)} W_s(i) \cdot \mathbf{1}[X_s(i) = 1]}{\max\{N_{t-1}(i), 1\}}. 
\end{align}
First, we claim that, with probability at least $1-\delta$, it holds that
$$
\left|\hat{\theta}_{t, i} - \theta_{t, i}\right| \leq 4R \sqrt{\frac{\log(2 d T / \delta)}{N_{t-1}(i) + 1}} + \sum^{t-1}_{s = 1\vee (t-w)} \left\| \theta_s - \theta_{s+1}\right\|_\infty.
$$
Consequently, the following UCB holds for each $t$ with probability at least $1-\delta$:
\begin{align}
\theta^\top_t X_t &\leq \max_{x\in {\cal E}_t} \left\{ \theta^\top_t x \right\} \nonumber\\
& \leq \max_{x\in {\cal E}_t} \left\{ \sum_{i\in E} \left[\hat{\theta}_{t, i} + 4R \sqrt{\frac{\log(2 d T / \delta)}{N_{t-1}(i) + 1}} +\sum^{t-1}_{s = 1\vee (t-w)} \left\| \theta_s - \theta_{s+1}\right\|_\infty \right] x(i) \right\} \nonumber\\
& = \sum_{i\in E}\left[\hat{\theta}_{t, i} + 4R \sqrt{\frac{\log(2 d T / \delta)}{N_{t-1}(i) + 1 }} +\sum^{t-1}_{s = 1\vee (t-w)} \left\| \theta_s - \theta_{s+1}\right\|_\infty \right] X_t(i) .\label{eq:d_arm_tgt}
\end{align}
By summing (\ref{eq:d_arm_tgt}) across $t$, we can bound the dynamic regret with probability at least $1-\delta$ as
\begin{align}
&\R_T(\swofu\text{ for combinatorial semi-bandits})\nonumber\\
\leq & \underbrace{\sum^T_{t=1} \sum_{i\in E} 4R \sqrt{\frac{\log(2 d T / \delta)}{ N_{t-1}(i) + 1 }} \cdot  \mathbf{1}[X_t(i) = 1] }_{(\dagger_{\text{SCB}})}+ \underbrace{m\sum^T_{t=1}\sum^{t-1}_{s = 1\vee (t-w)} \left\| \theta_s - \theta_{s+1}\right\|_\infty}_{(\ddagger_{\text{SCB}})}.
\end{align}
To complete the proof on the regret bound, we bound each $(\dagger_{\text{SCB}},\ddagger_{\text{SCB}})$ from above. 

\textbf{Analysing $(\dagger_{\text{SCB}})$.} Let's first define the notation $\bar{N}_{i, t} = \sum^{t-1}_{s =1 + \lfloor t / w \rfloor \cdot w} \mathbf{1}[X_s(i) = 1]$. We can understand $\bar{N}_{i, t}$ as follows, similarly to the derivation in the proof of Lemma \ref{lemma:sw1}. On one hand, the parameter $N_{i, t}$ counts the occurrences of $X_s(i) = 1$ in the $w$ previous rounds (or $t-1$ previous rounds if $t\leq w$). On the other hand, for the parameter  $\bar{N}_{i, t}$, we first divide the horizon into consecutive blocks of $w$ rounds (with the last block having $T - \lfloor T / w\rfloor \cdot w$ rounds). Then, for a round $t$, we look at the block that $t$ belongs to, and the parameter $\bar{N}_{i, t}$ counts the occurrences of $X_s(i) = 1$ for $s < t$ in that block. Certainly, we have $\bar{N}_{i, t} \leq N_{i, t}$.

We next use $\bar{N}_{i, t}$ to proceed with the bound:
\begin{align}
\sum^T_{t=1} \sum_{i\in E} \sqrt{\frac{\mathbf{1}[X_t(i) = 1]}{N_{i, t} + 1}} & \leq \sum^T_{t=1} \sum_{i\in E} \sqrt{\frac{\mathbf{1}[X_t(i) = 1]}{\bar{N}_{i, t} + 1 }} \nonumber\\
& = \sum^{\lceil T / w \rceil}_{j = 1}\sum_{i\in E} \sum^{j\cdot w \wedge T}_{t=(j-1)w + 1}  \sqrt{\frac{\mathbf{1}[X_t(i) = 1]}{\bar{N}_{i, t} + 1}} \nonumber\\
& \leq \sum^{\lceil T / w \rceil}_{j = 1}\sum_{i\in E} \sum^{j\cdot w \wedge T}_{t=(j-1)w + 1}  \sqrt{\frac{\mathbf{1}[X_t(i) = 1]}{\max\{\bar{N}_{i, t}, 1\}}} \nonumber\\
& \leq \sum^{\lceil T / w \rceil}_{j = 1}\sum_{i\in E} \left\{1 + 2\sqrt{\bar{N}_{i, j\cdot w \wedge T}}\right\}\label{eq:com_d-arm_step_1}\\ 
&\leq \sum^{\lceil T / w \rceil}_{j = 1} \left\{ d + 2\sqrt{d m w} \right\} \label{eq:com_d-arm_step_2}\\
&\leq \sum^{\lceil T / w \rceil}_{j = 1}  3\sqrt{dmw} \leq \frac{6\sqrt{dm}T}{\sqrt{w}}. \label{eq:com_d-arm_step_3}
\end{align}
Step (\ref{eq:com_d-arm_step_1}) is by the observation that, when we enumerate the non-zero summands $\sqrt{\frac{\mathbf{1}[X_t(i) = 1]}{\max\{\bar{N}_{i, t}, 1\}}}$ from $t=(i-1)w + 1 $ to $t = i\cdot w \wedge T$, the enumerated terms are $1/\sqrt{1}, 1/\sqrt{1}, 1  /\sqrt{2}, 1/\sqrt{3}, \ldots, 1 / \sqrt{ \max\{ \bar{N}_{i, j\cdot w \wedge T} ,1 \}}$. The sum of these terms is upper bounded as $1 + 2\sqrt{\bar{N}_{i, j\cdot w \wedge T}}$. Step (\ref{eq:com_d-arm_step_2}) is by the following calculation:
$$
\sum_{i\in E} \sqrt{\bar{N}_{i, j\cdot w \wedge T}} \leq \sqrt{d\cdot \sum_{i\in E}\bar{N}_{i, j\cdot w \wedge T}} = \sqrt{d\cdot \sum_{i\in E} \sum^{j\cdot w \wedge T}_{t=(j-1)w + 1} \mathbf{1}[X_t(i) = 1] }\leq \sqrt{d m w}.
$$
Finally, step (\ref{eq:com_d-arm_step_3}) is by the Theorem's assumption that $(d/m)\leq w\leq T$.

\textbf{Analysing $(\ddagger_{\text{SCB}})$.} We note that 
\begin{align}
m\sum^T_{t=1}\sum^{t-1}_{s = 1\vee (t-w)} \left\| \theta_s - \theta_{s+1}\right\|_\infty=m\sum_{s=1}^{T-1}\sum_{t=s+1}^{T\wedge (s+w)}\left\| \theta_s - \theta_{s+1}\right\|_\infty\leq mwB_T.
\end{align}
\section{Proof of Theorem \ref{theorem:semi_bob}}\label{sec:theorem:semi_bob}
Similar to the proof of Proposition \ref{prop:bob}, the dynamic regret of the \bob~can be decomposed as the regret of the \swofu~with the optimally tuned window size $w_i=w^{\dag}~(\geq d/m)$ for each block $i$ plus the loss due to learning the value $w^{\dag}$ with the EXP3 algorithm, \ie,
\begin{align}
\label{eq:semi_decompose}
\nonumber\Ex\left[\text{Regret}_T(\bob)\right]
\nonumber=&\Ex\left[\sum_{t=1}^T\langle x_t^*,\theta_t\rangle-\sum_{i=1}^{\lceil T/H\rceil}\sum_{t=(i-1)H+1}^{i\cdot H\wedge T}\left\langle X_t^{w^{\dag}},\theta_t\right\rangle\right]\\
&+\Ex\left[\sum_{i=1}^{\lceil T/H\rceil}\sum_{t=(i-1)H+1}^{i\cdot H\wedge T}\left\langle X_t^{w^{\dag}},\theta_t\right\rangle\right.\left.-\sum_{i=1}^{\lceil T/H\rceil}\sum_{t=(i-1)H+1}^{i\cdot H\wedge T}\left\langle X_t^{w_i},\theta_t\right\rangle\right].
\end{align}
Here, eq. (\ref{eq:semi_decompose}) holds as the \bob~restarts the \swofu~in each block, and for a round $t$ in block $i,$ $X_t^{w}$ refers to the action selected in round $t$ by the \swofu~with window size $w\wedge (t-(i-1)H-1)$ initiated at the beginning of block $i.$ 

By Theorem \ref{theorem:semi_sw_main}, the first expectation in eq. (\ref{eq:semi_decompose}) can be upper bounded as
\begin{align}
\label{eq:semi_decompose1}
\nonumber\Ex\left[\sum_{t=1}^T\langle x_t^*,\theta_t\rangle-\sum_{i=1}^{\lceil T/H\rceil}\sum_{t=(i-1)H+1}^{i\cdot H\wedge T}\left\langle X_t^{w^{\dag}},\theta_t\right\rangle\right]
\nonumber=&\Ex\left[\sum_{i=1}^{\lceil T/H\rceil}\sum_{t=(i-1)H+1}^{i\cdot H\wedge T}\left\langle x^*_t-X_t^{w^{\dag}},\theta_t\right\rangle\right]\\
\nonumber=&\sum_{i=1}^{\lceil T/H\rceil}\widetilde{O}\left(w^{\dag}mB_T(i)+\frac{\sqrt{dm}H}{\sqrt{w^{\dag}}}\right)\\
=&\widetilde{O}\left(w^{\dag}B_T+\frac{\sqrt{dm}T}{\sqrt{w^{\dag}}}\right),
\end{align}
where $$B_{T}(i)=\sum_{t=(i-1)H+1}^{(i\cdot H\wedge t)-1}\|\theta_{t}-\theta_{t+1}\|_{\infty}$$ is the total variation in block $i.$ 

We then turn to the second expectation in eq. (\ref{eq:semi_decompose}). We can easily see that the number of rounds for the EXP3 algorithm is $\lceil T/H\rceil$ and the number of possible values of $w_i$'s is $|J|.$ If the maximum absolute sum of reward of any block does not exceed $Q,$ the authors of \citep{ABFS02} gives the following regret bound.
\begin{align}
\nonumber&\Ex\left[\sum_{i=1}^{\lceil T/H\rceil}\sum_{t=(i-1)H+1}^{i\cdot H\wedge T}\left\langle X_t^{w^{\dag}},\theta_t\right\rangle-\sum_{i=1}^{\lceil T/H\rceil}\sum_{t=(i-1)H+1}^{i\cdot H\wedge T}\left\langle X_t^{w_i},\theta_t\right\rangle\middle|\forall i\in[\lceil T/H\rceil]\sum_{t=(i-1) H+1}^{i\cdot H\wedge T}Y_t\leq Q/2\right]\\
\label{eq:semi_decompose2}=&\widetilde{O}\left(Q\sqrt{\frac{|J|T}{H}}\right).
\end{align} 
Note that the regret of our problem is at most $T,$ eq. (\ref{eq:semi_decompose2}) can be further upper bounded as
\begin{align}
\label{eq:semi_decompose3}
\nonumber&\Ex\left[\sum_{i=1}^{\lceil T/H\rceil}\sum_{t=(i-1)H+1}^{i\cdot H\wedge T}\left\langle X_t^{w^{\dag}},\theta_t\right\rangle-\sum_{i=1}^{\lceil T/H\rceil}\sum_{t=(i-1)H+1}^{i\cdot H\wedge T}\left\langle X_t^{w_i},\theta_t\right\rangle\right]\\
\nonumber\leq&\widetilde{O}\left(Q\sqrt{\frac{|J|T}{H}}\right)\times\Pr\left(\forall i\in[\lceil T/H\rceil]\sum_{t=(i-1) H+1}^{i\cdot H\wedge T}Y_t\leq Q/2\right)\\
\nonumber&+\Ex\left[\sum_{i=1}^{\lceil T/H\rceil}\sum_{t=(i-1)H+1}^{i\cdot H\wedge T}\left\langle X_t^{w^{\dag}},\theta_t\right\rangle-\sum_{i=1}^{\lceil T/H\rceil}\sum_{t=(i-1)H+1}^{i\cdot H\wedge T}\left\langle X_t\left(w_i\right),\theta_t\right\rangle\middle| \exists i\in[\lceil T/H\rceil]\sum_{t=(i-1) H+1}^{i\cdot H\wedge T}Y_t\geq Q/2\right]\\
\nonumber&\quad\times\Pr\left(\exists i\in[\lceil T/H\rceil]\sum_{t=(i-1) H+1}^{i\cdot H\wedge T}Y_t\geq Q/2\right)\\
\nonumber\leq&\widetilde{O}\left(m\sqrt{H|J|T}\right)+T\cdot\frac{2}{T}\\
=&\widetilde{O}\left(m\sqrt{H|J|T}\right).
\end{align} 
Combining eq. (\ref{eq:semi_decompose}), (\ref{eq:semi_decompose1}), and (\ref{eq:semi_decompose3}), we have for any $w^{\dag}\in J$ and $w^{\dag}\geq d/m,$
\begin{align*}
\Ex\left[\text{Regret}_T(\bob)\right]=\widetilde{O}\left(w^{\dag}mB_T(i)+\frac{\sqrt{dm}H}{\sqrt{w^{\dag}}}+m\sqrt{H|J|T}\right)=\widetilde{O}\left(w^{\dag}mB_T+\frac{\sqrt{dm}T}{\sqrt{w^{\dag}}}+d^{\frac{1}{4}}m^{\frac{3}{4}}T^{\frac{3}{4}}\right).
\end{align*}
where we have plugged in the choices of $H$ and $J$ in eq. (\ref{eq:semi_bob_parameters}).
Therefore, we have that when $B_T\geq d^{-1/4}m^{1/4}T^{1/4},$ the \bob~is able to converge to the optimal window size \ie, $w^{\dag}=w^*~(\leq H),$ and the dynamic regret of the \bob~is upper bounded as
\begin{align}
	\R_T(\bob)=&\widetilde{O}\left(d^{\frac{1}{3}}m^{\frac{2}{3}}B_T^{\frac{1}{3}}T^{\frac{2}{3}}+d^{\frac{1}{4}}m^{\frac{3}{4}}T^{\frac{3}{4}}\right)=\widetilde{O}\left(d^{\frac{1}{3}}m^{\frac{2}{3}}B_T^{\frac{1}{3}}T^{\frac{2}{3}}\right);
\end{align}
while if $B_T< d^{-1/4}m^{1/4}T^{1/4},$ the \bob~converges to the window size $w^{\dag}=H,$ and the dynamic regret is 
\begin{align}
	\R_T(\bob)=&\widetilde{O}\left(d^{\frac{1}{2}}m^{\frac{1}{2}}B_TT^{\frac{1}{2}}+d^{\frac{1}{2}}T^{\frac{3}{4}}\right)=\widetilde{O}\left(d^{\frac{1}{4}}m^{\frac{3}{4}}T^{\frac{3}{4}}\right).
\end{align}
Combining the above two cases, we conclude the desired dynamic regret bound.
\section{Supplementary Details for Section \ref{sec:numerical}}\label{sec:numerical_details}
When $B_T$ is known , we select $w^\text{opt}$ that minimizes the explicit regret bound in (\ref{eq:explicit_swucb_bd}), resulting in
\begin{equation}\label{eq:w_opt_deri}
w^\text{opt} = \left\lceil \frac{\bar{w}}{B^{2/3}_T} \right\rceil, \text{ where } \bar{w} =  \frac{d^{1/3}T^{2/3}}{2^{1/3}L^{2/3}} \left(R\sqrt{d\ln\left(T+T^2 L^2/\lambda\right)}+\sqrt{\lambda}S\right)^{2/3} \log^{1/3}\left(1 + \frac{T L^2}{d\lambda^2}\right).
\end{equation}	
When $B_T$ is not known, we select $w^\text{obl} = \lceil \bar{w} \rceil $, which is independent of $B_T$.  
\end{APPENDIX}
\end{document}